%% file: uai2024-template.tex
\documentclass[accepted]{uai2024} 
                        
\usepackage[american]{babel}

\usepackage{natbib} 
\let\cite\citep
\usepackage{mathtools} 
\usepackage{booktabs} 
\usepackage{tikz} 

\usepackage{graphicx}
\usepackage{caption}
\usepackage{subcaption}
\usepackage{float}
\usepackage{multirow}
\usepackage{multicol}
\usepackage{sidecap}

\usepackage{amsmath}
\usepackage{amssymb}
\usepackage{amsthm}
\usepackage{makecell}
\usepackage{apptools}

\newcommand{\probust}{\texorpdfstring{$p^\mathrm{robust}_{\sigma}$}{probust}}
\newcommand{\pmc}{\texorpdfstring{$p^\mathrm{mc}_{\sigma}$}{pmc}}
\newcommand{\ptaylor}{\texorpdfstring{$p^\mathrm{taylor}_{\sigma}$}{ptaylor}}
\newcommand{\ptaylormvs}{\texorpdfstring{$p^\mathrm{taylor\_mvs}_{\sigma}$}{ptaylormvs}}
\newcommand{\pmmse}{\texorpdfstring{$p^\mathrm{mmse}_{\sigma}$}{pmmse}}
\newcommand{\pmmsemvs}{\texorpdfstring{$p^\mathrm{mmse\_mvs}_{\sigma}$}{pmmsemvs}}
\newcommand{\psoftmax}{\texorpdfstring{$p^\mathrm{softmax}_{T}$}{psoftmax}}
\newcommand{\probustwsigma}[1]{\texorpdfstring{$p^\mathrm{robust}_{\sigma= {#1}}$}{probust}}
\newcommand{\pmmsewsigma}[1]{\texorpdfstring{$p^\mathrm{mmse}_{\sigma = {#1}}$}{pmmse}}

\title{Characterizing Data Point Vulnerability via Average-Case Robustness}

\author[1]{\hspace{2em}Tessa Han*}
\author[2]{\hspace{3em}Suraj Srinivas*}
\author[3]{Himabindu Lakkaraju}

\affil[1,2,3]{Harvard University, Cambridge, MA}
\affil[1]{\texttt{than@g.harvard.edu, $^2$ssrinivas@seas.harvard.edu, $^3$hlakkaraju@hbs.edu}}
    
\begin{document}
\maketitle
\footnotetext{Accepted at Conference on Uncertainty in AI (UAI) 2024}
\footnotetext{*Equal contribution}

\begin{abstract}
Studying the robustness of machine learning models is important to ensure consistent model behaviour across real-world settings. To this end, adversarial robustness is a standard framework, which views robustness of predictions through a binary lens: either a worst-case adversarial misclassification exists in the local region around an input, or it does not. However, this binary perspective does not account for the degrees of vulnerability, as data points with a larger number of misclassified examples in their neighborhoods are more vulnerable. In this work, we consider a complementary framework for robustness, called average-case robustness, which measures the fraction of points in a local region that provides consistent predictions. However, computing this quantity is hard, as standard Monte Carlo approaches are inefficient especially for high-dimensional inputs. In this work, we propose the first analytical estimators for average-case robustness for multi-class classifiers. We show empirically that our estimators are accurate and efficient for standard deep learning models and demonstrate their usefulness for identifying vulnerable data points, as well as quantifying robustness bias of models. Overall, our tools provide a complementary view to robustness, improving our ability to characterize model behaviour. 
\end{abstract}

\input{1-intro}

\input{2-related-work}
\input{3-method}

\input{4-experiments}

\input{5-conclusion}

\bibliography{references.bib}

\appendix
\onecolumn
\input{6-appendix}

\end{document}

%% file: 1-intro.tex
\section{Introduction}

A desirable attribute of machine learning models is robustness to perturbations of input data. A popular notion of robustness is adversarial robustness, the ability of a model to maintain its prediction when presented with adversarial perturbations, i.e., perturbations designed to cause the model to change its prediction. Although adversarial robustness identifies whether a misclassified example exists in a local region around an input, it fails to capture the degree of vulnerability of that example, indicated by the difficulty in finding an adversary. For example, if the model geometry is such that 99\% of the local region around an example (say, point $A$) contains correctly classified examples, this makes it harder to find an adversarial example as compared to the case where only 1\% of the local region (say, for point $B$) contains correctly classified examples, where even random perturbations may be misclassified. However, from the adversarial robustness perspective, the prediction at a point is declared either robust or not, and thus both points $A$ and $B$ are considered equally non-robust (see Figure \ref{fig:main-fig} for an illustrative example). The ease of obtaining a misclassification, or \textit{data point vulnerability}, is captured by another kind of robustness: \emph{average-case robustness}, i.e., the fraction of points in a local region around an input for which the model provides consistent predictions. \footnote{In addition to the size of the misclassified region, another factor that affects the ease of finding misclassified examples is the specific optimization method used. In this study, we aim to study model robustness in a manner agnostic to the specific optimization used, and thus, we only focus on the size of the misclassified region. We believe this study can form the basis for future studies looking into the properties (e.g., “ease of identifying misclassified examples”) of specific optimization methods.} If this fraction is less than one, then an adversarial perturbation exists. The smaller this fraction, the easier it is to find a misclassified example. 
While adversarial robustness is motivated by model security, average-case robustness is better suited for model and dataset understanding, and debugging.

Standard approaches to computing average-case robustness involve Monte-Carlo sampling, which is computationally inefficient especially for high-dimensional data. For example, \citet{cohen2019certified} use $n=100,000$ Monte Carlo samples per data point to compute this quantity. 
In this paper, we propose to compute average-case robustness via analytical estimators, reducing the computational burden, while simultaneously providing insight into model decision boundaries. Our estimators are exact for linear models and well-approximated for non-linear models, especially those having a small local curvature \cite{moosavi2019robustness, srinivas2022efficient}. Overall, our work makes the following contributions:

\begin{enumerate}
    \item We derive novel analytical estimators to efficiently compute the average-case robustness of multi-class classifiers. We also provide estimation error bounds for these estimators that characterizes approximation errors for non-linear models. 

    \item We empirically validate our analytical estimators on standard deep learning models and datasets, demonstrating that these estimators accurately and efficiently estimate average-case robustness.

    \item We demonstrate the usefulness of our estimators in two case studies: identifying vulnerable samples in a dataset and measuring class-level robustness bias \cite{nanda2021fairness}, where we find that standard models exhibit significant robustness bias among classes. 
\end{enumerate}

To our knowledge, this work is the first to investigate analytical estimation of average-case robustness for the multi-class setting. In addition, the efficiency of these estimators makes the computation of average-case robustness practical, especially for large deep neural networks.

\begin{figure}
    \centering
    \includegraphics[width=0.9\linewidth]{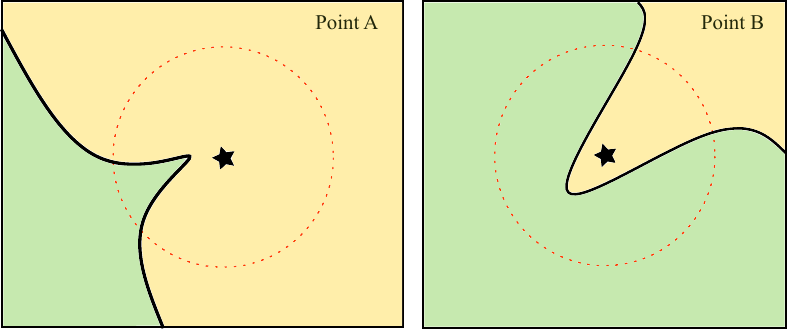}
    \caption{Consider a binary classifier (green vs. yellow) and points $A$ (left) and $B$ (right), both correctly classified to the yellow class. The dotted red circles represent $\epsilon$-balls around the data points. Although adversarial robustness rightly considers the model non-robust at both points (due to the existence of adversarial examples within the $\epsilon$-ball), it fails to discern that point $B$ has a larger fraction of misclassified points in its neighborhood, making it more vulnerable than point $A$, an aspect exactly captured by average-case robustness.}
    \label{fig:main-fig}
\end{figure}

%% file: 2-related-work.tex
\section{Related Work}

\textbf{Adversarial robustness.} Prior works have proposed methods to generate adversarial attacks \cite{carlini2017towards, goodfellow2014explaining, moosavi2016deepfool}, which find adversarial perturbations in a local region around a point. In contrast, this work investigates average-case robustness, which calculates the probability that a model’s prediction remains consistent in a local region around a point. Prior works have also proposed methods to certify model robustness \cite{cohen2019certified, carlini2022certified}, guaranteeing the lack of adversarial examples for a given $\epsilon$-ball under certain settings. Specifically, \citet{cohen2019certified} propose randomized smoothing, which involves computing class-wise average-case robustness, which is taken as the output probabilities of the randomized smoothed model. However, they estimate these probabilities via Monte Carlo sampling with $n=100,000$ samples, which is computationally expensive. Viewing average-case robustness from the lens of randomized smoothing, our estimators can also be seen as providing an analytical estimate of randomized smoothed models. However, in this work, we focus on their applications for model understanding and debugging as opposed to improving robustness.

\textbf{Probabilistic robustness.} Prior works have explored notions of probabilistic and average-case robustness.
For instance, \citet{fazlyab2019probabilistic, kumar2020certifying, mangal2019robustness} focus on certifying robustness of real-valued outputs to input perturbations. In contrast, this work focuses on only those output changes that cause misclassification. Like our work, \citet{franceschi2018robustness} also considers misclassifications. However, \citet{franceschi2018robustness} aims to find the smallest neighborhood with no adversarial example, while we compute the probability of misclassification in a given neighborhood. \citet{robey2022probabilistically, rice2021robustness} also aim to compute average-case robustness. However, they do so by computing the average loss over the neighborhood, while we use the misclassification rate. Closest to our work is the work by \citet{weng2019proven} which aims to certify the binary misclassification rate (with respect to a specific class to misclassify to) using lower and upper linear bounds. In contrast, our work estimates the multi-class misclassification rate, as opposed to bounding the quantity in a binary setting. A crucial contribution of our work is its applicability to multi-class classification and the emphasis on estimating, rather than bounding, robustness.

\textbf{Robustness to distributional shifts.} 
Prior works have explored the performance of models under various distributions shifts~\cite{taori2020measuring, ovadia2019can}. From the perspective of distribution shift, average-case robustness can be seen as a measure of model performance under Gaussian noise, a type of natural distribution shift. In addition, in contrast to works in distributional robustness which seek to build models that are robust to distributions shifts~\cite{thulasidasan2021effective, moayeri2022explicit}, this work focuses on measuring the vulnerability of existing models to Gaussian distribution shifts.

%% file: 3-method.tex
\section{Average-Case Robustness Estimation}
\label{sec:methods}

\newcommand{\E}{\mathop{\mathbb{E}}}
\newcommand{\R}{\mathbb{R}}
\newcommand{\X}{\mathbf{x}}
\newcommand{\W}{\mathbf{w}}
\newcommand{\U}{\mathbf{u}}
\newcommand{\matU}{\mathbf{U}}

\newcommand{\grad}{\nabla_{\X}}
\newcommand{\cdf}{\Phi_{\matU \matU^\top}}

\newtheorem{defn}{Definition}
\newtheorem{thm}{Proposition}
\newtheorem{lemma}{Lemma}
\newtheorem{remark}{Remark}

\AtAppendix{\counterwithin{lemma}{section}}
\AtAppendix{\counterwithin{thm}{section}}

\newenvironment{hproof}{%
  \renewcommand{\proofname}{Proof Idea}\proof}{\endproof}

In this section, we first describe the mathematical problem of average-case robustness estimation. Then, we present the naïve estimator based on Monte Carlo sampling and derive more efficient analytical estimators. 

\subsection{Notation and Preliminaries}

Assume that we have a neural network $f: \R^d \rightarrow \R^C$ with $C$ output classes and that the classifier predicts class $t \in [1,..., C]$ for a given input $\X \in \R^d$, i.e., $t~=~\arg \max_{i=1}^{C} f_i(\X)$, where $f_i$ denotes the logits for the $i^{th}$ class. Given this classifier, the average-case robustness estimation problem is to compute the probability of consistent classification (to class $t$) under noise perturbation of the input. 

\begin{defn} We define the \textbf{average-case robustness} of a classifier $f$ to noise $\mathcal{R}$ at a point $\X$ as
\begin{align*}
    p^\mathrm{robust}(\X, t) = P_{\epsilon \sim R} \left[ \arg\max_i f_i(x+ \epsilon) = t \right]
\end{align*}
\end{defn}

The more robust the model is in the local neighborhood around $\X$, the larger the average-case robustness measure $p^\text{robust}$($\X, t$). In this paper, given that robustness is always measured with respect to the predicted class~$t$ at $\X$, we henceforth suppress the dependence on $t$ in the notation. We also explicitly show the dependence of $p^\text{robust}$ on the noise scale $\sigma$ by denoting it as $p^\text{robust}_{\sigma}$. 

In this work, we shall consider $\mathcal{R}$ as an isotropic Normal distribution, i.e., $\mathcal{R} = \mathcal{N}(0, \sigma^2)$. However, as we shall discuss in the next section, it is possible to accommodate both non-isotropic and non-Gaussian distributions in our method. Note that for high-dimensional data ($d \rightarrow \infty$), the isotropic Gaussian distribution converges to the uniform distribution on the surface of the sphere with radius $r = \sigma \sqrt{d}$ \footnote{Alternately, if $\epsilon \sim \mathcal{N}(0, \sigma^2 / d)$, then $r = \sigma$} due to the concentration of measure phenomenon \cite{vershynin2018high}. 

Observe that when the domain of the input noise is restricted to an $\ell_p$ ball, $p^\text{robust}_{\sigma}$ generalizes the corresponding $\ell_p$ adversarial robustness. In other words, adversarial robustness is concerned with the quantity $\mathbf{1}(p^\text{robust}_{\sigma} < 1)$, i.e., the indicator function that average-case robustness is less than one (which indicates the presence of an adversarial perturbation), while this work focuses on computing the quantity \probust{} itself. In the rest of this section, we derive estimators for \probust{}.

\paragraph{The Monte-Carlo estimator.}
A naïve estimator of average-case robustness is the Monte-Carlo estimator \pmc{}. It computes the robustness of a classifier $f$ at input $\X$ by generating $M$ noisy samples of $\X$ and then calculating the fraction of these noisy samples that are classified to the same class as $\X$. In other words,

\begin{align*}
    p_{\sigma}^\text{robust}(\X) &=P_{\epsilon \sim \mathcal{N}(0,\sigma^2)} \left[ \arg\max_i f_i(\X+ \epsilon) = t \right] \\
    &= \E_{\epsilon \sim \mathcal{N}(0,\sigma^2)} \left[ \mathbf{1}_{\arg\max_i f_i(\X+ \epsilon) = t} \right] \\
    &\approx \frac{1}{M} \sum_{j=1}^{M} \left[ \mathbf{1}_{\arg\max_i f_i(\X+ \epsilon_j) = t} \right]
    = p_{\sigma}^\text{mc}(\X)
\end{align*}

\pmc{} replaces the expectation with the sample average of the $M$ noisy samples of $\X$ and has been used in prior work \citep{nanda2021fairness}. Technically, the error for the Monte-Carlo estimator is independent of dimensionality and is given by $\mathcal{O}(1/ \sqrt{M})$ \citep{vershynin2018high}. However, in practice, for neural networks, \pmc{} requires a large number of random samples to converge to the underlying expectation. For example, for MNIST and CIFAR10 CNNs, it takes around $M = 10,000$ samples per point for \pmc{} to converge, which is computationally expensive, and further, provides little information regarding the decision boundaries of the underlying model. Thus, we set out to address this problem by developing more efficient and informative analytical estimators of average-case robustness.

\subsection{Robustness Estimation via Linearization}

Before deriving analytical robustness estimators for non-linear models, we first consider the simpler problem of deriving this quantity for linear models. This is challenging, especially for multi-class classifiers. For example, given a linear model for a three-class classification problem with weights $w_1, w_2, w_3$ and biases $b_1, b_2, b_3$, such that $y = \arg \max_i \{w_i^\top\X + b_i \mid i \in [1,2,3] \}$, the decision boundary function between classes $i$ and $j$ is given by $y_{ij} = (w_i - w_j)^\top \X + (b_i - b_j)$. If the predicted label at $\X$ is $y = 1$, the relevant decision boundary functions are $y_{12}, y_{13}$ which characterize the decision boundaries of misclassifications from class $1$ to classes $2, 3$ respectively. To compute the total probability of misclassification, we must compute the probability of decision boundaries $y_{12}, y_{13}$ being crossed separately. Crucially, it is important not to ``double count'' the probability of both $y_{12}$ and $y_{13}$ being simultaneously crossed. Computing the probability of falling into this problematic region is non-trivial, as it depends on the relative orientations of $y_{12}$ and $y_{13}$. If they are orthogonal, then this problem is avoided, as the probability of crossing $y_{12}$ and $y_{13}$ are independent random variables. However, this is not true in general for non-orthogonal decision boundaries.
Further, this ``double counting'' problem increases in complexity with an increasing number of classes, stemming from a corresponding increase in the number of such pairwise decision boundaries. Lemma \ref{estimator-linear-models} provides an elegant solution to this combinatorial problem via the multivariate Gaussian CDF.

\textbf{Notation}: For clarity, we represent tensors by collapsing along the "class" dimension, i.e., $a_i ~ \big\vert_{i=1}^C := (a_1, a_2, ... a_i, ... a_c)$, where for an order-$t$ tensor $a_i$, the expansion $a_i~ \big\vert_{i=1}^C$ is an order-$(t+1)$ tensor. 

\newcommand{\tensor}{~\bigg\vert_{\substack{i = 1\\i \neq t}}^{C}}

\begin{lemma}
The local robustness of a multi-class linear model $f(\X) = \mathbf{w}^\top \X + b$ (with $\mathbf{w} \in \R^{d \times C}$ and $b \in \R^C$) at point $\X$ with respect to a target class $t$ is given by the following. Define weights $\U_i = \W_t - \W_i \in \R^d, \forall i \neq t$, where $\W_t, \W_i$ are rows of $\mathbf{w}$ and biases $c_i = {\U_i}^\top\X + (b_t - b_i) \in \R$. Then, 
\begin{align*}
    p^\mathrm{robust}_\sigma(\X) = \cdf \left( \frac{c_i}{\sigma \| \U_i \|_2} \tensor \right)\\
    \mathrm{where}~~\matU = \frac{\U_i}{\| \U_i \|_2} \tensor \in \R^{(C-1) \times d}
\end{align*}
and $\cdf$ is the ($C-1$)-dimensional Normal CDF with zero mean and covariance $\matU \matU^\top$.
\label{estimator-linear-models}
\end{lemma}

\begin{hproof}
    The proof involves constructing decision boundary functions $g_i(\X) = f_t(\X) - f_i(\X)$ and computing the probability $p^\text{robust}_{\sigma}(\X) = P_{\epsilon}(\bigcup_{\substack{i=1\\i \neq t}}^{C} g_i(\X + \epsilon) > 0)$. For Gaussian $\epsilon$, we observe that $\frac{\U}{\sigma \| \U \|_2}^\top \epsilon \sim \mathcal{N}(0, 1)$ is also a Gaussian, which applied vectorially results in our usage of $\Phi$. As convention, we represent $\matU$ in a normalized form to ensure that its rows are unit norm.
\end{hproof}

The proof is in Appendix~\ref{app:proofs}. Thus, the multivariate Gaussian CDF provides an elegant solution to the previously mentioned ``double counting'' problem. Here, the matrix $\matU$ exactly captures the linear decision boundaries, and the covariance matrix $\matU \matU^\top$ encodes the alignment between pairs of decision boundaries of different classes. 

\textbf{Remark.} For the binary classification case, we get $\matU \matU^\top = 1$ (a scalar), and $p^\text{robust}_{\sigma}(\X) = \phi(\frac{c}{\sigma \| \U \|_2} )$, where $\phi$ is the CDF of the scalar standard normal, which was previously also shown by \citet{weng2019proven, pawelczyk2022probabilistically}. Hence Lemma \ref{estimator-linear-models} is a multi-class generalization of these works.

If the decision boundary vectors $\U_i$ are all orthogonal to each other, then the covariance matrix $\matU \matU^\top$ is the identity matrix. For diagonal covariance matrices, the multivariate Normal CDF (\emph{mvn-cdf}) can be written as the product of univariate Normal CDFs, which is easy to compute. However, in practice, we find that the covariance matrix is strongly non-diagonal, indicating that the decision boundaries are not orthogonal to each other. This non-diagonal nature of covariance matrices in practice leads to the resulting \emph{mvn-cdf} not having a closed form solution, and thus needing to be approximated via sampling \cite{botev2017normal, SciPy}. However, this sampling is performed in the $(C-1)$-dimensional space as opposed to the $d$-dimensional space that \pmc{} samples from. In practice, for classification problems, we often have $C << d$, making sampling in $(C-1)$-dimensions more efficient. We would like to stress here that the expression in Lemma \ref{estimator-linear-models} represents the simplest expression to compute the average-case robustness: the usage of the multi-variate Gaussian CDF cannot be avoided due to the computational nature of this problem.
We now discuss the applicability of Lemma \ref{estimator-linear-models} to non-Gaussian noise. 

\begin{lemma} \label{lemma:universality}(\textbf{Application to non-Gaussian noise})
    For high-dimensional data ($d \rightarrow \infty$), Lemma \ref{estimator-linear-models} generalizes to any coordinate-wise independent noise distribution that satisfies Lyapunov's condition. 
\end{lemma} 

\begin{hproof}
    Applying Lyupanov's central limit theorem \cite{patrick1995probability}, given $\epsilon \sim \mathcal{R}$ is sampled from some distribution $\mathcal{R}$, we have $\frac{\U}{\sigma \| \U \|_2}^\top \epsilon = \sum_{j=1}^{d} \frac{\U_j}{\sigma\| \U \|_2} \epsilon_j ~~\substack{d\\\longrightarrow} ~~\mathcal{N}(0, 1)$, which holds as long as the sequence $\{\frac{\U_j}{\| \U \|_2} \epsilon_j\}$ are independent random variables and satisfy the Lyapunov condition, which encodes the fact that higher-order moments of such distributions progressively shrink. 
\end{hproof}

Thus, as long as the input noise distribution is ``well-behaved'', the central limit theorem ensures that the distribution of high-dimensional dot products is Gaussian, thus motivating our use of the \emph{mvn-cdf} more generally beyond Gaussian input perturbations. We note that it is also possible to easily generalize Lemma \ref{estimator-linear-models} to \textbf{non-isotropic} Gaussian perturbations with a covariance matrix $\mathcal{C}$, which only changes the form of the covariance matrix of the \emph{mvn-cdf} from $\matU\matU^\top \rightarrow \matU \mathcal{C} \matU^\top$, which we elaborate in Appendix \ref{app:proofs}. In the rest of this paper, we focus on the isotropic case. 

\subsubsection{Estimator 1: The Taylor Estimator} 

Using the estimator derived for multi-class linear models in Lemma \ref{estimator-linear-models}, we now derive the Taylor estimator, a local robustness estimator for non-linear models.

\begin{defn}
    The \textbf{Taylor estimator} for the local robustness of a classifier $f$ at point $\X$ with respect to target class $t$ is given by linearizing $f$ around $\X$ using a first-order Taylor expansion, with decision boundaries $g_i(\X) = f_t(\X) - f_i(\X)$, $\forall i \neq t$, leading to
    \begin{align*}
        p^\mathrm{taylor}_{\sigma}(\X) = \cdf \left( \frac{g_i(\X)}{\sigma \|\grad g_i(\X)\|_2} \tensor \right) 
    \end{align*}
    with $\matU$ and $\Phi$ defined as in the linear case.
\label{eqn:taylor-estimator}
\end{defn}

The proof is in Appendix~\ref{app:proofs}. It involves locally linearizing non-linear decision boundary functions $g_i(\X)$ using a Taylor series expansion. We expect this estimator to have a small error when the underlying model is well-approximated by a locally linear function in the local neighborhood. We formalize this intuition by computing the estimation error for a quadratic classifier. 

\begin{thm} The \textbf{estimation error} of the Taylor estimator for a classifier with a quadratic decision boundary $g_i(\X) = \X^\top A_i \X + \U_i^\top \X + c_i$ and positive-semidefinite $A_i$ is upper bounded by
    \begin{align*}
        | p^\mathrm{robust}_{\sigma}(\X) - p^\mathrm{taylor}_{\sigma}(\X) | \leq k \sigma^{C-1} \prod_{\substack{i=1\\i\neq t}}^{C} \frac{\lambda_{\max}^{A_i}}{\| \U_i \|_2} 
    \end{align*}
    for noise $\epsilon \sim \mathcal{N}(0, \sigma^2 / d)$, in the limit of $d \rightarrow \infty$. Here, $\lambda_{\max}^{A_i}$ is the max eigenvalue of $A_i$, and $k$ is a small problem dependent constant.
\end{thm} 

The proof is in Appendix~\ref{app:proofs}. This statement formalizes two key intuitions with regards to the Taylor estimator: (1) the estimation error depends on the size of the local neighborhood $\sigma$ (the smaller the local neighborhood, the more locally linear the model, and the smaller the estimator error), and (2) the estimation error depends on the extent of non-linearity of the underlying function, which is given by the ratio of the max eigenvalue of $A$ to the Frobenius norm of the linear term. This measure of non-linearity of a function, called normalized curvature, has also been independently proposed by previous work \cite{srinivas2022efficient}. Notably, if the max eigenvalue is zero, then the function $g_i(\X)$ is exactly linear, and the estimation error is zero, reverting back to the linear case in Lemma \ref{estimator-linear-models}.

\subsubsection{Estimator 2: The MMSE Estimator} 

While the Taylor estimator is more efficient than the Monte Carlo estimator, it has a drawback: its linearization is only faithful at perturbations close to the data point and not necessarily for larger perturbations. To mitigate this issue, we use a form of linearization that is faithful over larger noise perturbations. Linearization has been studied in feature attribution research, which concerns itself with approximating non-linear models with linear ones to produce model explanations \cite{han2022explanation}. In particular, the SmoothGrad \cite{smilkov2017smoothgrad} technique has been described as the MMSE (minimum mean-squared error) optimal linearization of the model \cite{han2022explanation, agarwal2021towards} in a Gaussian neighborhood around the data point. Using a similar idea, we propose the MMSE estimator \pmmse{} as follows.

\begin{defn}
    The \textbf{MMSE estimator} for the local robustness of a classifier $f$ at point $\X$ with respect to target class $t$ is given by an MMSE linearization $f$ around $\X$, for decision boundaries $g_i(\X) = f_t(\X) - f_i(\X)$, $\forall i \neq t$, leading to
    \begin{align*}
        &p^\mathrm{mmse}_{\sigma}(\X) = \cdf \left( \frac{ \Tilde{g}_i(\X)}{\sigma \| \grad \Tilde{g}_i(\X)\|_2} \tensor \right) \\
        &\mathrm{where}~~\Tilde{g}_i(\X) = \frac{1}{N}\sum_{j=1}^{N} g_i(\X + \epsilon) ~,~ \epsilon \sim \mathcal{N}(0, \sigma^2)
    \end{align*}
    with $\matU$ and $\Phi$ defined as in the linear case, and $N$ is the number of perturbations. 
\end{defn}

The proof is in Appendix~\ref{app:proofs}. It involves creating a randomized smooth model \cite{cohen2019certified} from the base model and computing the decision boundaries of this smooth model. Note that this estimator also involves drawing noise samples like the Monte Carlo estimator. However, unlike the Monte Carlo estimator, we find that the MMSE estimator converges fast (around $N = 5$), leading to an empirical advantage. We now compute the estimation error of the MMSE estimator.

\newcommand{\mn}{\text{mean}}

\begin{thm} The \textbf{estimation error} of the MMSE estimator for a classifier with a quadratic decision boundary $g_i(\X) = \X^\top A_i \X + \U_i^\top \X + c_i$ and positive-semidefinite $A_i$ is upper bounded by
    \begin{align*}
        | p^\mathrm{robust}_{\sigma}(\X) - p^\mathrm{mmse}_{\sigma}(\X) | \leq k \sigma^{C-1} \prod_{\substack{i=1\\i\neq t}}^C \frac{\lambda_{\max}^{A_i} - \lambda_{\mn}^{A_i}}{\| \U_i \|_2}  
    \end{align*}
    for noise $\epsilon \sim \mathcal{N}(0, \sigma^2 / d)$, in the limit of $d \rightarrow \infty$ and $N \rightarrow \infty$. Here, $\lambda_{\max}^{A_i}, \lambda_{\mn}^{A_i}$ are the maximum and mean eigenvalue of $A_i$ respectively, and $k$ is a small problem dependent constant. 
\end{thm}

The proof is in Appendix~\ref{app:proofs}. The result above highlights two aspects of the MMSE estimator: (1) it incurs a smaller estimation error than the Taylor estimator, and (2) even in the limit of large number of samples $N \rightarrow \infty$, the error of the MMSE estimator is non-zero, except when $\lambda_{\mn}^{A_i} = \lambda_{\max}^{A_i}$. For PSD matrices, this becomes zero when $A_i$ is a multiple of the identity matrix \footnote{When $d \rightarrow \infty$, $\epsilon^\top A \epsilon = \lambda \| \epsilon \|^2 = \lambda \sigma^2$ is a constant, and thus an isotropic quadratic function resembles a linear one in this neighborhood. }, reverting back to the linear case in Lemma \ref{estimator-linear-models}.

\subsubsection{(Optionally) Approximating \emph{mvn-cdf}: Connecting Robustness Estimation with Softmax}

\paragraph{Approximation with Multivariate Sigmoid.} One drawback of the Taylor and MMSE estimators is their use of the \emph{mvn-cdf}, which does not have a closed form solution and can cause the estimators to be slow for settings with a large number of classes $C$. In addition, the \emph{mvn-cdf} makes these estimators non-differentiable, which is inconvenient for applications which require differentiating \probust{}. To alleviate these issues, we approximate the \emph{mvn-cdf} with an analytical closed-form expression. As CDFs are monotonically increasing functions, the approximation should also be monotonically increasing.

To this end, it has been previously shown that the \emph{univariate} Normal CDF $\phi$ is well-approximated by the sigmoid function \cite{hendrycks2016gaussian}. It is also known that when $\matU \matU^\top = I$, \emph{mvn-cdf} is given by $\Phi(\X) = \prod_i\phi(\X_i)$, i.e., it is given by the product of the univariate normal CDFs. Thus, we may choose to approximate $\Phi(\X) = \prod_i \text{sigmoid}(\X)$. However, when the inputs are small, this can be simplified as follows:

\begin{align*}
    &\Phi_{I}(\X) = \prod_i \phi(\X_i) \approx \prod_i \frac{1}{1 + \exp(-\X_i)}\\
    &= \frac{1}{1 + \sum_i \exp(-\X_i) + \sum_{j,k} \exp(-\X_j - \X_k) + ...} \\
    &\approx \frac{1}{1 + \sum_i \exp(-\X_i)} ~~~(\text{for} ~~\X_i \rightarrow \infty~~ \forall i)
\end{align*}


We call the final expression the ``multivariate sigmoid'' (\emph{mv-sigmoid}) which serves as our approximation of \emph{mvn-cdf}, especially at the tails  of the distribution. While we expect estimators using \emph{mv-sigmoid} to approximate ones using \emph{mvn-cdf} only when $\matU \matU^\top = \mathbf{I}$, we find experimentally that the approximation works well even for practical values of the covariance matrix $\matU\matU^\top$. Using this approximation to substitute \emph{mv-sigmoid} for \emph{mvn-cdf} in the \ptaylor{} and \pmmse{} estimators yields the \ptaylormvs{} and \pmmsemvs{} estimators, respectively. We present further analysis on the multivariate sigmoid in Appendix \ref{app:experiments}.

\paragraph{Approximation with Softmax.} A common method to estimate the confidence of model predictions is to use the softmax function applied to the logits $f_i(\X)$ of a model. We note that softmax is identical to \emph{mv-sigmoid} when directly applied to the logits of neural networks: 

\begin{align*}
    &\text{softmax}_t\left( f_i(\X) ~\Big\vert_{\substack{i = 1}}^{C} \right) = \frac{\exp(f_t(\X))}{\sum_{i=1}^C \exp(f_i(\X))} = \\& \frac{1}{1 + \sum\limits_{\substack{i=1\\i \neq t}}^{C} \exp(f_i(\X) - f_t(\X))} = \text{mv-sigmoid}\left( g_i(\X) ~\Big\vert_{\substack{i = 1\\i\neq t}}^{C} \right)
\end{align*}

Recall that $g_i(\X) = f_t(\X) - f_i(\X)$ is the decision boundary function. Note that this equivalence only holds for the specific case of logits, and cannot be applied to approximate the Taylor estimator, for instance. Nonetheless, given this similarity, it is reasonable to ask whether softmax applied to logits (henceforth $p^\text{softmax}_{T}$ for softmax with temperature $T$) itself can be a ``good enough'' estimator of $p^\text{robust}_{\sigma}$ in practice. In other words, does $p^\text{softmax}_T$ well-approximate $p^\text{robust}_{\sigma}$ in certain settings?
In Appendix \ref{app:proofs}, we provide a theoretical result for a restricted linear setting where softmax can indeed match the behavior of \ptaylormvs{}, which happens precisely when $\matU \matU^\top = \mathbf{I}$ and all the class-wise gradients are equal. In the next section, we demonstrate empirically that the softmax estimator $p^{\text{softmax}}_T$ is a poor estimator of average-case robustness in practice.

%% file: 4-experiments.tex
\section{Empirical Evaluation}
\label{sec:exp}

\begin{figure*}[h!]
    \centering
    \begin{subfigure}{0.6\textwidth}
        \centering
        \includegraphics[width=\textwidth]{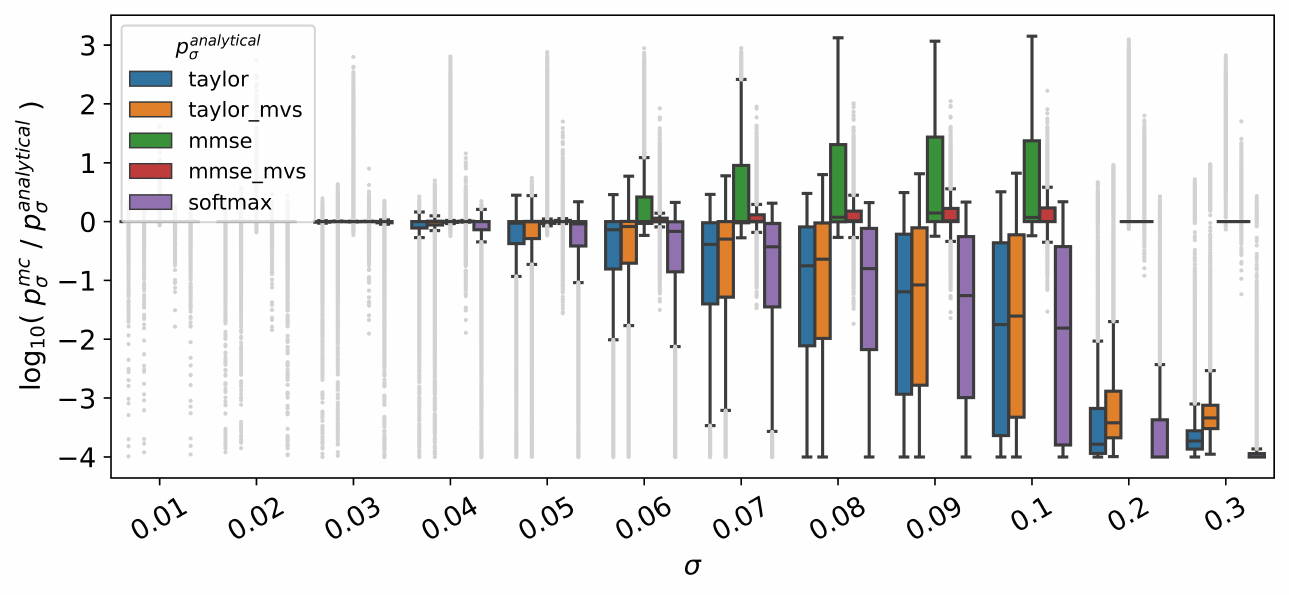}
        \captionsetup{justification=centering}
        \caption{CIFAR10, ResNet18 \\ Comparing estimator errors}
        \label{fig1a:method-works-over-sigma}
    \end{subfigure}
    \hspace{1em}
    \begin{subfigure}{0.35\textwidth}
        \centering
        \includegraphics[width=\textwidth]{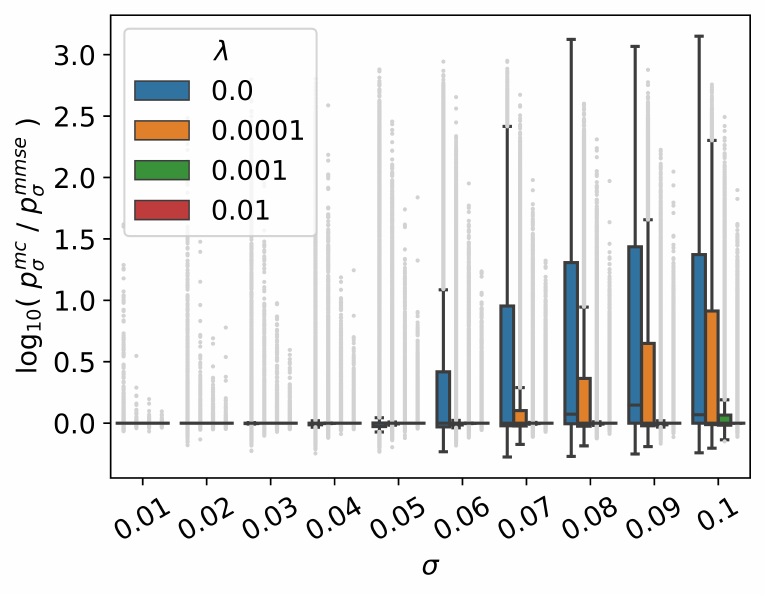}
        \captionsetup{justification=centering}
        \caption{CIFAR10, ResNet18 \\ Varying model robustness}
        \label{fig1b:method-works-robust}
    \end{subfigure}
    \caption{Empirical evaluation of analytical estimators. (a) The smaller the noise neighborhood $\sigma$, the more accurately the estimators compute \probust{}. \pmmse{} and \pmmsemvs{} are the best estimators of \probust{}, followed closely by \ptaylormvs{} and \ptaylor{}, trailed by \psoftmax{}. (b) For more robust models, the estimators compute \probust{} more accurately over a larger $\sigma$. Together, these results indicate that the analytical estimators accurately compute \probust{}.}
    \label{fig1:method-works}
\end{figure*}

In this section, we first evaluate the estimation errors and computational efficiency of the analytical estimators, and then evaluate the impact of robustness training within models on these estimation errors. Then, we analyze the relationship between average-case robustness and softmax probability. Lastly, we demonstrate the usefulness of local robustness for model and dataset understanding with two case studies. Key results are discussed in this section and full results are in Appendix~\ref{app:experiments}.

\textbf{Datasets and models.}
We evaluate the estimators on four datasets: MNIST \citep{deng2012mnist}, FashionMNIST \citep{xiao2017fashion}, CIFAR10 \citep{krizhevsky2009learning}, and CIFAR100 \citep{krizhevsky2009learning}. For MNIST and FashionMNIST, we train linear models and CNNs. For CIFAR10 and CIFAR100, we train Transformer models. We also train ResNet18 models \citep{he2016deep} using varying levels of gradient norm regularization~\cite{srinivas2018knowledge, srinivas2024models} to obtain models with varying levels of robustness. 
For gradient norm regularization, the objective function is $\ell(f(x), y) + \lambda \|\nabla_x f(x)\|_2^2$, where $\lambda$ is the regularization constant. The larger $\lambda$ is, the more robust the model.
Note that gradient norm regularization is equivalent to Gaussian data augmentation with an infinite number of augmented samples~\cite{srinivas2018knowledge} and is different from adversarial training.
Unless otherwise noted, the experiments below use each dataset's test set which consists of 10,000 points. Additional details about the datasets and models are described in Appendix~\ref{app:datasets} and \ref{app:models}.

\subsection{Evaluation of the estimation errors of analytical estimators}
\label{sec:exp_correctness}

\textbf{The analytical estimators accurately compute local robustness.}
To empirically evaluate the estimation error of our estimators, we calculate \probust{} for each model using \pmc{}, \ptaylor{}, \pmmse{}, \ptaylormvs{}, \pmmsemvs{}, and \psoftmax{} for different $\sigma$ values. For \pmc{}, \pmmse{}, and \pmmsemvs{}, we use a sample size at which these estimators have converged ($n=10000, 500, \text{and } 500$, respectively). (Convergence analyses are in Appendix~\ref{app:experiments}.) We take the Monte-Carlo estimator as the gold standard estimate of $p^{robust}_{\sigma}$), and compute the absolute and relative difference between \pmc{} and the other estimators to evaluate their estimation errors. 

The performance of the estimators for the CIFAR10 ResNet18 model is shown in Figure~\ref{fig1a:method-works-over-sigma}. The results indicate that \pmmsemvs{} and \pmmse{} are the best estimators of \probust{}, followed closely by \ptaylormvs{} and \ptaylor{}, trailed by \psoftmax{}. This is consistent with the theory in Section~\ref{sec:methods}, where the analytical estimation errors of $p^{mmse}_{\sigma}$ are lower than $p^{taylor}_{\sigma}$.

The results also confirm that the smaller the noise neighborhood $\sigma$, the more accurately the estimators compute \probust{}. For the MMSE and Taylor estimators, this is because their linear approximation of the model around the input is more faithful for smaller $\sigma$. As expected, when the model is linear, \ptaylor{} and \pmmse{} accurately compute \probust{} for all $\sigma$'s (Appendix~\ref{app:experiments}). For the softmax estimator, \psoftmax{} values are constant over $\sigma$ and this particular model has high \psoftmax{} values for most points. Thus, for small $\sigma$'s where \probust{} is near one, \psoftmax{} happens to approximate \probust{} for this model. Examples of images with varying levels of noise ($\sigma$) are in Appendix~\ref{app:experiments}.

\textbf{Impact of robust training on estimation errors.} 
The performance of \pmmse{} for CIFAR10 ResNet18 models of varying levels of robustness is shown in Figure~\ref{fig1b:method-works-robust}. The results indicate that the estimator is more accurate for more robust models (larger $\lambda$) over a larger $\sigma$. This is because robust training leads to models that are more locally linear \cite{moosavi2019robustness}, making the estimator's linear approximation of the model around the input more accurate over a larger $\sigma$, making its \probust{} values more accurate.

\textbf{Evaluating estimation error of mv-sigmoid.} To examine \emph{mv-sigmoid}'s approximation of \emph{mvn-cdf}, we compute both functions using the same inputs ($z~=~ \frac{g_i(\X)}{\sigma \|\grad g_i(\X)\|_2} \vert_{\substack{i=1\\i\neq t}}^C$, as described in Proposition~\ref{eqn:taylor-estimator}) for the CIFAR10 ResNet18 model for different $\sigma$. The plot of \emph{mv-sigmoid(z)} against \emph{mvn-cdf(z)} for $\sigma=0.05$ is shown in Appendix~\ref{app:experiments} (Figure~\ref{fig2:mvsig-mvncdf}). The results indicate that the two functions are strongly positively correlated with low approximation error, suggesting that \emph{mv-sigmoid} approximates the \emph{mvn-cdf} well in practice.

\subsection{Evaluation of computational efficiency of analytical estimators}

\textbf{The analytical estimators are more efficient than the naïve estimator.}
We examine the efficiency of the estimators by measuring their runtimes when calculating \probustwsigma{0.1} for the CIFAR10 ResNet18 model for 50 points. Runtimes are displayed in Table~\ref{table:runtimes}. They indicate that \ptaylor{} and \pmmse{} perform 35x and 17x faster than \pmc{}, respectively. Additional runtimes are in Appendix~\ref{app:experiments}.

\begin{table}[ht!]
\centering
\begin{tabular}{l|l|l|l }
    Estimator   & \thead{Number of \\Samples ($n$)}   & \thead{CPU\\Runtime\\(h:m:s)}  & \thead{GPU\\Runtime\\(h:m:s)} \\
    \toprule
    \pmc{}   & \begin{tabular}[c]{@{}l@{}}  $n=10,000$\end{tabular}               
             & \begin{tabular}[c]{@{}l@{}}  1:41:11\end{tabular}                 
             & \begin{tabular}[c]{@{}l@{}}  0:19:56\end{tabular}  \\
    \ptaylor{}   & N/A
                 & 0:00:08                                                                   
                 & 0:00:02  \\
    \pmmse{}   & \begin{tabular}[c]{@{}l@{}} $n=5$\end{tabular} 
               & \begin{tabular}[c]{@{}l@{}} 0:00:41\end{tabular} 
               & \begin{tabular}[c]{@{}l@{}} 0:00:06\end{tabular} \\              
\end{tabular}
\vspace{0.2cm}
\caption{Runtimes of \probust{} estimators. Each estimator computes \probustwsigma{0.1} for the CIFAR10 ResNet18 model for 50 data points. Estimators that use sampling use the minimum number of samples necessary for convergence. Runtimes are in the format of hour:minute:second. The GPU used was a Tesla V100. The analytical estimators (\ptaylor{} and \pmmse{}) are more efficient than the naïve estimator (\pmc{}).} 
\vspace{-0.5cm}
\label{table:runtimes}
\end{table}

We also examine the efficiency of the analytical estimators in terms of memory usage. The backward pass is observed to take about twice the amount of floating-point operations (FLOPs) as a forward pass~\cite{flops}. In addition, we performed an experiment and found that a forward and backward pass uses about twice the peak memory of a single forward pass. Thus, each iteration of \pmmse{} (which consists of a forward and backward pass) is roughly 3x the number of FLOPs and twice the peak memory of a single iteration of \pmc{} (which consists of one forward pass). However, \pmmse{} requires 5 iterations for convergence while \pmc{} requires about 10,000. Thus, overall, \pmmse{} is more memory-efficient than \pmc{}.

\subsection{Case Studies}
\label{subsec:case-studies}

\textbf{Identifying non-robust data points.} While robustness is typically viewed as the property of a model, the average-case robustness perspective compels us to view robustness as a joint property of both the model and the data point. In light of this, we can ask, given the same model, which samples are robustly and non-robustly classified? We evaluate whether \probust{} can distinguish such images better than \psoftmax{}. To this end, we train a simple CNN to distinguish between images with high and low \pmmse{} and the same CNN to also distinguish between images with high and low \psoftmax{} (additional setup details described in Appendix~\ref{app:experiments}). Then, we compare the performance of the two models. For CIFAR10, the test set accuracy for the \pmmse{} CNN is $\mathbf{92\%}$ while that for the \psoftmax{} CNN is $\textbf{58\%}$. These results indicate that \probust{} better identifies images that are robust to and vulnerable to random noise than \psoftmax{}.

We also present visualizations of images with the highest and lowest \pmmse{} in each class for each model. For comparison, we do the same with \psoftmax{}. Example CIFAR10 images are shown in Figure~\ref{fig4:topk-vs-bottomk-main}. We observe that images with low \probust{} tend to have neutral colors, with the object being a similar color as the background (making the prediction likely to change when the image is slightly perturbed), while images with high \probust{} tend to be brightly-colored, with the object strongly contrasting with the background (making the prediction likely to stay constant when the image is slightly perturbed). Recall that points with small \probust{} are close to the decision boundary, while those farther away have a high \probust{}. Thus, high \probust{} points may be thought of as ``canonical'' examples of the underlying class, while low \probust{} examples are analogous to ``support vectors'', that are critical to model learning. These results showcase the utility of average-case robustness for dataset exploration and analysis, particularly in identifying canonical and outlier examples.

\begin{figure}[htbp!]
    \centering
    \begin{flushleft}
        \hspace{-0.1cm}\rotatebox{90}{\hspace{-9.4cm} \hspace{3cm}Car \hspace{2.9cm}Boat}
        \hspace{1.1cm}Lowest \pmmsewsigma{0.1}
        \hspace{1.5cm} Highest \pmmsewsigma{0.1}
    \end{flushleft}
         
    \begin{subfigure}{0.22\textwidth}
        \includegraphics[width=\linewidth, trim={0.2cm, 0.2cm, 0.2cm, 0.2cm}]{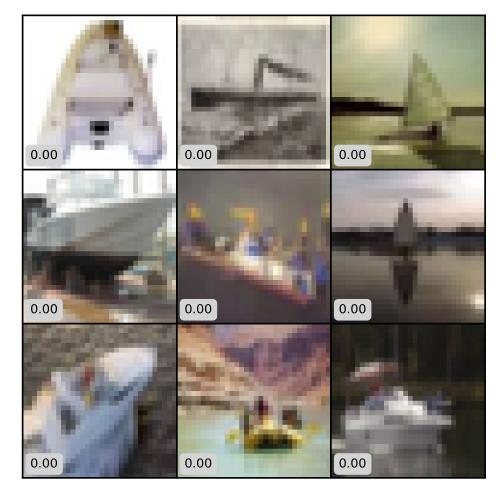}
    \end{subfigure}
    \begin{subfigure}{0.22\textwidth}
        \includegraphics[width=\linewidth, trim={0.2cm, 0.2cm, 0.2cm, 0.2cm}]{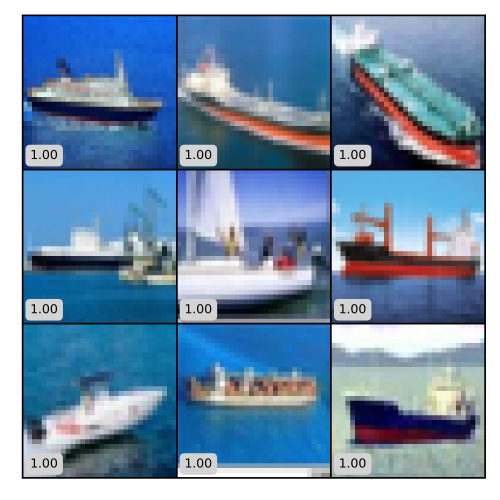}
    \end{subfigure}
    \begin{subfigure}{0.22\textwidth}
        \includegraphics[width=\linewidth, trim={0.2cm, 0.2cm, 0.2cm, 0.2cm}]{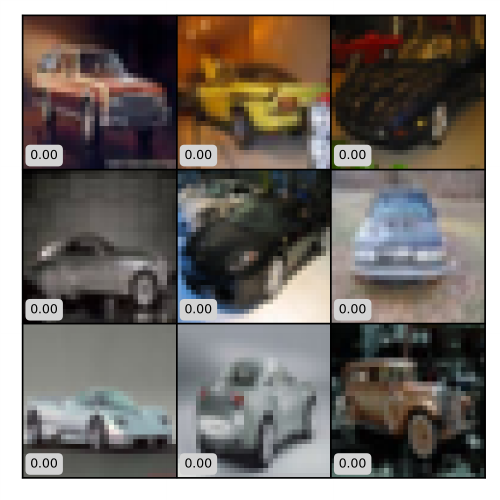}
    \end{subfigure}
    \begin{subfigure}{0.22\textwidth}
        \includegraphics[width=\linewidth, trim={0.2cm, 0.2cm, 0.2cm, 0.2cm}]{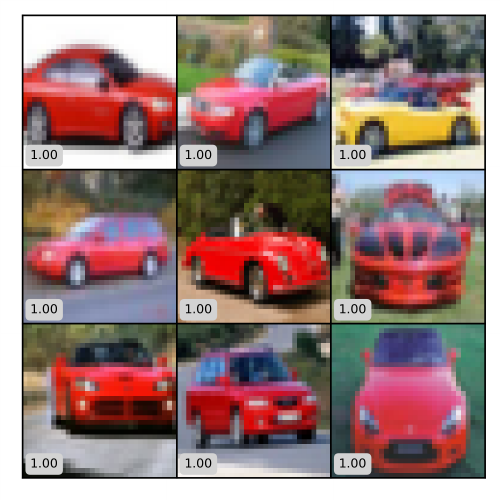}
    \end{subfigure}
    \caption{Example ranking of \probust{} among CIFAR10 classes. Images with high \probust{} are farther away from the decision boundary, and tend to be brighter and have stronger object-background contrast than those with low \probust{}, which are closer to the decision boundary, and thus easily misclassified.}
    \label{fig4:topk-vs-bottomk-main}
\end{figure}

\begin{figure}[h]
    \centering
    \begin{subfigure}{0.35\textwidth}
        \centering
        \includegraphics[width=\linewidth,  trim={1cm, 0.5cm, 0.8cm, 0cm}]{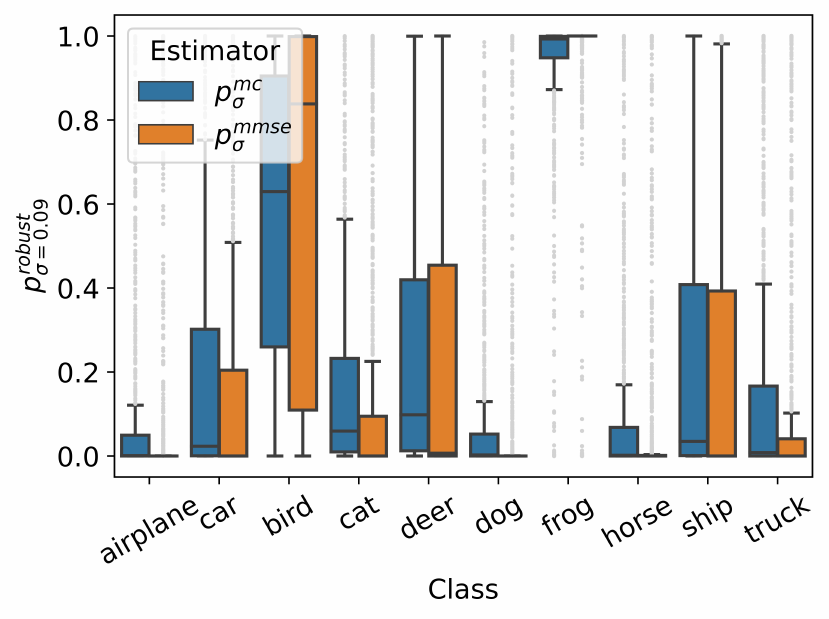}
        \caption{CIFAR10, ResNet18}
        \vspace{0.25cm}
    \end{subfigure}
    \begin{subfigure}{0.35\textwidth}
            \centering
          \includegraphics[width=\linewidth, trim={0.9cm, 0.5cm, 1.1cm, 0cm}]{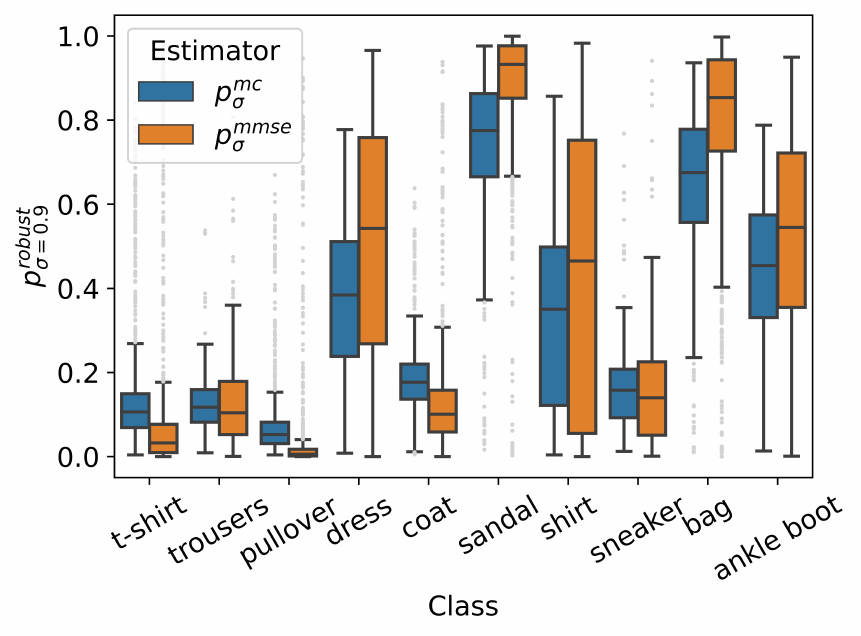}
         \caption{FMNIST, CNN}
    \end{subfigure}
    \caption{Computing robustness bias among classes for the (a) ResNet18 CIFAR10 model, and (b) for the CNN FMNIST model. \probust{} reveals that the model robustness varies significantly across classes, revealing a marked class-wise bias within standard models. The analytical estimator \pmmse{} accurately captures this model bias.}
    \label{fig5:robustness-bias}
\end{figure}

\textbf{Detecting robustness bias among classes: Is the model differently robust for different classes?} We also demonstrate that \probust{} can detect bias in local robustness \cite{nanda2021fairness} by examining its distribution for each class for each model and test set over different $\sigma$'s. Results for the CIFAR10 ResNet18 model are in plotted in Figure~\ref{fig5:robustness-bias}. The results show that different classes have significantly different \probust{} distributions, i.e., the model is significantly more robust for some classes (e.g., frog) than for others (e.g., airplane). Similarly for the FMNIST CNN case in Figure~\ref{fig5:robustness-bias}, we find that the pullover class is much less robust than the sandal class. This observation indicates a disparity in outcomes for these different classes, and underscores the importance of evaluating per-class and per-datapoint robustness metrics before deploying models in the wild.
The results also show that \pmc{} and \pmmse{} have very similar distributions, further indicating that the latter well-approximates the former. \probust{} detects robustness bias across all other models and datasets too: MNIST CNN, and CIFAR100 ResNet18 (Appendix~\ref{app:experiments}). Thus, \probust{} can be applied to detect robustness bias among classes, which is critical when models are deployed in high-stakes, real-world settings.

%% file: 5-conclusion.tex
\section{Conclusion}

In this work, we find that adversarial robustness does not provide a comprehensive picture of model behavior, and to this end, we propose the usage of average-case robustness. While adversarial robustness is more suited for applications in model security, average-case robustness is suited for model understanding and debugging. 
To our knowledge, this work is the first to investigate analytical estimators for average-case robustness in a multi-class setting. The analytical aspect of these estimators helps understand average-case robustness via model decision boundaries, and also connect to ideas such as Randomized Smoothing (via the MMSE estimator) and softmax probabilities.

Future research directions include exploring additional applications of average-case robustness, such as training average-case robust models that minimize the probability of misclassification and debugging-oriented applications such as detecting model memorization, and dataset outliers.

\vspace{0.2cm}
\textbf{Code availability.}
Code is available at \url{https://github.com/AI4LIFE-GROUP/average-case-robustness}.

%% file: 6-appendix.tex
\section{Appendix}

\subsection{Proofs}
\label{app:proofs}

\begin{lemma} \label{proof:lemma}
The local robustness of a multi-class linear model $f(\X) = \mathbf{w}^\top \X + b$ (with $\mathbf{w} \in \R^{d \times C}$ and $b \in \R^C$) at point $\X$ with respect to a target class $t$ is given by the following. Define weights $\U_i = \W_t - \W_i \in \R^d, \forall i \neq t$, where $\W_t, \W_i$ are rows of $\mathbf{w}$ and biases $c_i = {\U_i}^\top\X + (b_t - b_i) \in \R$. Then, 
\begin{align*}
    p^\text{robust}_\sigma(\X) = \cdf \left( \frac{c_i}{\sigma \| \U_i \|_2} \tensor \right)\\
    \mathrm{where}~~\matU = \frac{\U_i}{\| \U_i \|_2} \tensor \in \R^{(C-1) \times d}
\end{align*}
and $\cdf$ is the ($C-1$)-dimensional Normal CDF with zero mean and covariance $\matU \matU^\top$.
\end{lemma}

\begin{proof}
First, we rewrite \probust{} in the following manner, by defining $g_i(\X) = f_t(\X) - f_i(\X) > 0$, which is the ``decision boundary function".

\begin{align*}
    p_\sigma^\text{robust} = P_{\epsilon \sim \mathcal{N}(0,\sigma^2)} \left[ \max_{i} f_i(\X + \epsilon) < f_t(\X + \epsilon) \right] \\= P_{\epsilon \sim \mathcal{N}(0,\sigma^2)} \left[ \bigcup_{i=1; i \neq t}^C g_i(\X + \epsilon) > 0 \right]
\end{align*}

Now, assuming that $f,g$ are linear such that $g_i(\X) = {\U'_i}^\top \X + g(0)$, we have $g_i(\X + \epsilon) = g_i(\X) + {\U_i}^\top \epsilon$, and obtain

\begin{align}
p_\sigma^\text{robust} &= P_{\epsilon \sim \mathcal{N}(0,\sigma^2)}\left[ \bigcup_{i=1; i \neq t}^C {\U_i}^{\top}\epsilon > -g_i(\X) \right] \\
&= P_{z \sim \mathcal{N}(0,I_d)} \left[ \bigcup_{i=1; i \neq t}^C \frac{\U_i}{\| \U_i \|_2}^{\top}z > - \frac{g_i(\X)}{\sigma \| \U_i \|_2} \right] \label{appeqn:key_step}
\end{align}

This step simply involves rescaling and standardizing the Gaussian to be unit normal. We now make the following observations:
\begin{itemize}
    \item For any matrix $\matU \in \R^{C-1 \times d}$ and a d-dimensional Gaussian random variable $z \sim \mathcal{N}(0, I_d) \in \R^d$, we have $\matU^\top z \sim \mathcal{N}(0, \matU \matU^\top)$, i.e., an (C-1) -dimensional Gaussian random variable. 
    \item CDF of a multivariate Gaussian RV is defined as $P_z [\bigcup_i z_i < t_i]$ for some input values $t_i$
\end{itemize}

Using these observations, if we construct $\matU = \frac{\U_i}{\| \U_i \|_2} \tensor \in \R^{(C-1) \times d}$, and obtain

\begin{align*}
p^\text{robust}_{\sigma} &= P_{r \sim \mathcal{N}(0, \matU\matU^\top)} \left[ \bigcup_{i=1; i \neq t}^C r_i < \frac{g_i(\X)}{\sigma \| \U_i \|_2} \right] \\
&= \text{CDF}_{\mathcal{N}(0, UU^{\top})} \left( \frac{g_i(\X)}{\sigma \| \U_i \|_2} \tensor \right)
\end{align*}

where $g_i(\X) = {\U_i}^\top \X + g_i(0) = {(\W_t - \W_i)}^\top\X + (b_t - b_i)$

\end{proof}

\vspace{1cm}

\begin{lemma} (\textbf{Extension to non-Gaussian noise})
    For high-dimensional data ($d \rightarrow \infty$), Lemma \ref{estimator-linear-models} generalizes to any coordinate-wise independent noise distribution that satisfies Lyapunov's condition. 
\end{lemma} 

\begin{proof}
    Applying Lyupanov's central limit theorem, given $\epsilon \sim \mathcal{R}$ is sampled from some distribution $\mathcal{R}$ to equation \ref{appeqn:key_step} in the previous proof, we have we have $\frac{\U}{\sigma \| \U \|_2}^\top \epsilon = \sum_{j=1}^{d} \frac{\U_j}{\sigma\| \U \|_2} \epsilon_j ~~\substack{d\\\longrightarrow} ~~\mathcal{N}(0, 1)$, which holds as long as the sequence $\{\frac{\U_j}{\| \U \|_2} \epsilon_j\}$ are independent random variables and satisfy the Lyapunov condition. In particular, this implies that $\matU^\top z \sim \mathcal{N}(0, \matU \matU^\top)$, and the proof proceeds as similar to the Gaussian case after this step.
\end{proof}

\vspace{1cm}

\begin{lemma} (\textbf{Extension to non-isotropic Gaussian}) Lemma \ref{estimator-linear-models} can be extended to the case of $\epsilon \sim \mathcal{N}(0, \mathcal{C})$ for an arbitrary positive definite covariance matrix $\mathcal{C}$:

\begin{align*}
    p^\text{robust}_\sigma(\X) = \Phi_{\matU \mathcal{C} \matU^\top} \left( \frac{c_i}{\| \U_i \|_2} \tensor \right)
\end{align*}
    
\end{lemma}

\begin{proof}
    We observe that the Gaussian random variable $\frac{\U_i}{\| \U_i \|}^\top \epsilon \vert_{\substack{i=1\\t \neq t}}^C = \matU^\top \epsilon$ has mean zero as $\epsilon$ is mean zero. Computing its covariance matrix, we have $\E_\epsilon \matU^\top \epsilon \epsilon^\top \matU  = \matU^\top \E_\epsilon (\epsilon \epsilon^\top) \matU = \matU^\top \mathcal{C} \matU$. We use this result after equation \ref{appeqn:key_step} in the proof of Lemma \ref{estimator-linear-models}.
\end{proof}

\vspace{1cm}

\begin{thm}
    The \textbf{Taylor estimator} for the local robustness of a classifier $f$ at point $\X$ with respect to target class $t$ is given by linearizing $f$ around $\X$ using a first-order Taylor expansion, with decision boundaries $g_i(\X) = f_t(\X) - f_i(\X)$, $\forall i \neq t$, leading to
    \begin{align*}
        p^\text{taylor}_{\sigma}(\X) = \cdf \left( \frac{g_i(\X)}{\sigma \|\grad g_i(\X)\|_2} \tensor \right) 
    \end{align*}
    with $\matU$ and $\Phi$ defined as in the linear case.
\end{thm}

\begin{proof}
    Using the notations from the previous Lemma \ref{proof:lemma}, we can linearize $g(\X + \epsilon) \approx g(\X) + \grad g(\X)^\top \epsilon$ using a first order Taylor series expansion.
    Thus we use $\U_i = \grad g_i(\X)$ and $c_i = g_i(\X)$, and plug it into the result of Lemma \ref{proof:lemma}.
\end{proof}

\begin{thm} The \textbf{estimation error} of the Taylor estimator for a classifier with a quadratic decision boundary $g_i(\X) = \X^\top A_i \X + \U_i^\top \X + c_i$ for positive-definite $A_i$, is upper bounded by
    \begin{align*}
        | p^{robust}_{\sigma}(\X) - p^{taylor}_{\sigma}(\X) | \leq k \sigma^{C-1} \prod_{\substack{i=1\\i\neq t}}^{C} \frac{\lambda_{\max}^{A_i}}{\| \U_i \|_2}  
    \end{align*}
    for noise $\epsilon \sim \mathcal{N}(0, \sigma^2 / d)$, in the limit of $d \rightarrow \infty$. 
\end{thm} 

\begin{proof}
Without loss of generality, assume that $\X = 0$. For any other $\X_1 \neq 0$, we can simply perform a change of variables of the underlying function to center it at $\X_1$ to yield a different quadratic. We first write an expression for $p^{robust}_\sigma$ for the given quadratic classifier $g_i(\X)$ at $\X = 0$. 

\begin{align*}
    p^{robust}_\sigma(0) &= P_{\epsilon} \left( \bigcup_i g_i(\epsilon) > 0 \right) \\
                          &= P_{\epsilon} \left( \bigcup_i \U_i^\top \epsilon + c > - \epsilon^\top A_i \epsilon \right) 
\end{align*}

Similarly, computing, $p^{taylor}_{\sigma}$ we have $\grad g_i(0) = \U^\top$ and $g_i(0) = c_i$, resulting in

\begin{align*}
    p^{taylor}_{\sigma}(0) &= P_{\epsilon}\left(\bigcup_i g^{taylor}_i(\epsilon) > 0\right) \\
    &= P_{\epsilon}\left(\bigcup_i \U_i^\top \epsilon + c > 0 \right)
\end{align*}

Subtracting the two, we have 

\begin{align*}
    &|p^{robust}_\sigma(0) - p^{taylor}_\sigma(0)| \\
    &= \left| P \left(\bigcup_i 0 >  \U_i^\top \epsilon + c > - \epsilon^\top A_i \epsilon \right) \right| \\
    &= \left| P \left(\bigcup_i 0 >  \frac{\U_i^\top \epsilon + c}{\sigma \| \U_i \|_2} > - \frac{\epsilon^\top A_i \epsilon}{\sigma \| \U_i \|_2} \right) \right| 
\end{align*}

For high-dimensional Gaussian noise $\epsilon \sim \mathcal{N}(0, \sigma^2 / d)$, with $d \rightarrow \infty$, we have that $\| \epsilon \|^2 = \sum_i \epsilon_i^2 \rightarrow \sigma^2$ from the law of large numbers. See \cite{vershynin2018high} for an extended discussion. Thus we have $\epsilon^\top A \epsilon \leq \lambda_{\max}^A \| \epsilon \|^2 = \lambda_{\max}^A \sigma^2$.

Also let $z_i = \frac{\U_i^\top \epsilon + c}{\sigma \| \U_i \|_2}$ be a random variable. We observe that $z_i \vert_i$ is a tensor extension of $z_i$, has a covariance matrix of $\matU \matU^\top$ as before. Let us also define $\mathcal{C}_i = \frac{\lambda_{\max}^{A_i}}{\| \U_i \|_2}$.

\begin{align*}
    |&p^{robust}_\sigma(0) - p^{taylor}_\sigma(0)| \\&= \left|P \left( \bigcup_i 0 > z_i(\epsilon) > - \frac{\epsilon^\top A_i \epsilon}{\sigma \| \U_i \|_2} \right) \right| \\
    &\leq \left| P\left(\bigcup_i 0 > z_i > - \frac{\lambda_{\max}^{A_i}}{\| \U_i \|_2} \sigma\right) \right|~~~(\epsilon^\top A \epsilon < \lambda_{\max}^A \sigma^2) \\
    &= \left| \int ... \int^{0}_{-\mathcal{C}_i \sigma} \text{pdf}(z_i \vert_i)~ \mathrm{d}z_i \vert_i \right| ~~~(\text{Defn of mvn cdf}) \\
    &\leq \max_{z_i \vert_i} \text{pdf}(z_i \vert_i) ~ \prod_i |C_i \sigma |~~~(\text{Upper bound pdf with its max})\\
    &\leq (2\pi)^{-(C-1)/2} \det(\matU \matU^\top)^{-1/2} \prod_{\substack{i=1\\i\neq t}}^C C_i \sigma \\ &=  k \left(\sigma^{C-1} \prod_{\substack{i=1\\i\neq t}}^C \frac{\lambda_{\max}^{A_i}}{\| \U_i \|_2} \right)
\end{align*}

where $k = \max_z pdf(z) = (2 \pi)^{-(C-1)/ 2} \det(\matU \matU^\top)^{-1/2}$, which is the max value of the Gaussian pdf. Note that as the rows of $\matU$ are normalized, $\det(\matU) \leq 1$ and $\det(\matU \matU^\top) = \det(\matU)^2 \leq 1$.

\end{proof}

We note that these bounds are rather pessimistic, as in high-dimensions $\epsilon^\top A_i \epsilon \sim \lambda_{\mn}^{A_i} \leq \lambda_{\max}^{A_i}$, and thus in reality the errors are expected to be much smaller. 

\vspace{1cm}

\begin{thm}
    The \textbf{MMSE estimator} for the local robustness of a classifier $f$ at point $\X$ with respect to target class $t$ is given by an MMSE linearization $f$ around $\X$, for decision boundaries $g_i(\X) = f_t(\X) - f_i(\X)$, $\forall i \neq t$, leading to
    \begin{align*}
        &p^\text{mmse}_{\sigma}(\X) = \cdf \left( \frac{ \Tilde{g}_i(\X)}{\sigma \| \grad \Tilde{g}_i(\X)\|_2} \tensor \right) \\
        &\mathrm{where}~~\Tilde{g}_i(\X) = \frac{1}{N}\sum_{j=1}^{N} g_i(\X + \epsilon) ~,~ \epsilon \sim \mathcal{N}(0, \sigma^2)
    \end{align*}
    with $\matU$ and $\Phi$ defined as in the linear case, and $N$ is the number of perturbations. 

\end{thm}

\begin{proof}
We would like to improve upon the Taylor approximation to $g(\X + \epsilon)$ by using an MMSE local function approximation. Essentially, we'd like the find $\U \in \R^d$ and $c \in \R$ such that 

\begin{align*}
    (\U^*(\X), c^*(\X)) = \arg\min_{\U,c} \E_{\epsilon \sim \mathcal{N}(0, \sigma^2)} (g(x+\epsilon) - \U^{\top} \epsilon - c)^2
\end{align*}

A straightforward solution by finding critical points and equating it to zero gives us the following:

\begin{align*}
    \U^*(\X) &= \E_\epsilon \left[ g(x + \epsilon) \epsilon^{\top} \right] / \sigma^2 \\&= \E_\epsilon \left[ \grad g(\X + \epsilon) \right] ~~~~~ (\text{Stein's Lemma}) \\
    c^*(\X) &= \E_\epsilon g(x + \epsilon)
\end{align*}

Plugging in these values of $U^*, c^*$ into Lemma \ref{proof:lemma}, we have the result.

\end{proof}

\vspace{1cm}

\begin{thm} The \textbf{estimation error} of the MMSE estimator for a classifier with a quadratic decision boundary $g(\X) = \X^\top A \X + \U^\top \X + c$, and positive definite $A$ is upper bounded by
    \begin{align*}
        | p^{robust}_{\sigma}(\X) - p^{mmse}_{\sigma}(\X) | \leq k \sigma^{C-1} \prod_{\substack{i=1\\i\neq t}}^C \frac{|\lambda_{\max}^{A_i} - \lambda_{\mathrm{mean}}^{A_i}|}{\| \U_i \|_2}
    \end{align*}
    for noise $\epsilon \sim \mathcal{N}(0, \sigma^2 / d)$, in the limit of $d \rightarrow \infty$ and $N \rightarrow \infty$.  
\end{thm}

\begin{proof}
We proceed similarly to the proof made for the Taylor estimator, and without loss of generality, assume that $\X = 0$. Computing, $p^{mmse}_{\sigma}$ we have $\E_{\epsilon} \grad g_i(\epsilon) = \U_i^\top$ and $\E_{\epsilon} g_i(\epsilon) = c + \E (\epsilon^\top A_i \epsilon) = c + \E(trace(\epsilon^\top A_i \epsilon)) = c + \E(trace(A_i \epsilon \epsilon^T)) = c + trace(A_i) \sigma^2 / d = c + \sigma^2 \lambda_{\mn}^{A_i}$, resulting in

\begin{align*}
    p^{mmse}_{\sigma}(0) &= P_{\epsilon}\left(\bigcup_i \hat{g}_i(\epsilon) > 0\right) \\
    &= P_{\epsilon}\left(\bigcup_i \U_i^\top \epsilon + c > - \sigma^2 \lambda_{\mn}^{A_i} \right)
\end{align*}

Subtracting the two, we have 

\begin{align*}
    &|p^{robust}_\sigma(0) - p^{mmse}_\sigma(0)| \\
    & \leq \left|P \left(\bigcup_i - \sigma^2 \lambda_\mn^{A_i} >  \U_i^\top \epsilon + c > - \sigma^2 \lambda_{\max}^{A_i} \right) \right| \\
    &= \left|P \left(\bigcup_i - \sigma \frac{\lambda_{\mn}^{A_i}}{\| \U_i \|_2} >  \frac{\U_i^\top \epsilon + c}{\sigma \| \U_i \|_2} > - \sigma \frac{\lambda_{\max}^{A_i}}{\| \U_i \|_2} \right) \right| 
\end{align*}

Similar to the previous proof, let $z_i = \U_i^\top \epsilon + c$ be a random variable, and that $z_i \vert_i$ is a tensor extension of $z_i$ from our previous notation.

\begin{align*}
    |&p^{robust}_\sigma(\X) - p^{mmse}_\sigma(\X)| \\
    &\leq \left| P\left(\bigcup_i  - \lambda_{\mn}^{A_i} \sigma^2 > z_i > - \lambda_{\max}^{A_i} \sigma^2\right) \right| \\
    &= \left| \int ... \int^{-\lambda_{\mn}^{A_i} \sigma^2}_{-\lambda_{\max}^{A_i} \sigma^2} \text{pdf}(z_i \vert_i)~ \mathrm{d}z_i \vert_i \right| \\
    &\leq \max_{z_i \vert_i} \text{pdf}(z_i \vert_i) ~ \sigma^{C-1} \prod_i \frac{|(\lambda_{\max}^{A_i} - \lambda_{\mn}^{A_i}) |}{\| \U_i \|_2}\\
    &= k \sigma^{C-1} \prod_i \frac{|\lambda_{\max}^{A_i} - \lambda_{\mn}^{A_i}|}{\| \U_i \|2}
\end{align*}

where $k = \max_z \text{pdf}(z_i \vert_i) = (2 \pi)^{-(C-1)/ 2} \det(\matU \matU^\top)^{-1/2}$ like in the Taylor case. Note that as the rows of $\matU$ are normalized, $\det(\matU) \leq 1$ and $\det(\matU \matU^\top) = \det(\matU)^2 \leq 1$.

\end{proof}

We note that these bounds are rather pessimistic, as in high-dimensions $\epsilon^\top A_i \epsilon \sim \lambda_{\mn}^{A_i} \leq \lambda_{\max}^{A_i}$, and thus in reality the errors are expected to be much smaller. 

\subsubsection{Approximating the Multivariate Gaussian CDF with mv-sigmoid}\label{app:mv-sigmoid-explain} 

One drawback of the Taylor and MMSE estimators is their use of the \emph{mvn-cdf}, which does not have a closed form solution and can cause the estimators to be slow for settings with a large number of classes $C$. In addition, the \emph{mvn-cdf} makes these estimators non-differentiable, which is inconvenient for applications which require differentiating \probust{}. To alleviate these issues, we approximate the \emph{mvn-cdf} with an analytical closed-form expression. As CDFs are monotonically increasing functions, the approximation should also be monotonically increasing.

To this end, it has been previously shown that the \emph{univariate} Normal CDF $\phi$ is well-approximated by the sigmoid function \cite{hendrycks2016gaussian}. It is also known that when $\matU \matU^\top = I$, \emph{mvn-cdf} is given by $\Phi(\X) = \prod_i\phi(\X_i)$, i.e., it is given by the product of the univariate normal CDFs. Thus, we may choose to approximate $\Phi(\X) = \prod_i \text{sigmoid}(\X)$. However, when the inputs are small, this can be simplified as follows:

\begin{align*}
    &\Phi_{I}(\X) = \prod_i \phi(\X_i) \approx \prod_i \frac{1}{1 + \exp(-\X_i)}\\
    &= \frac{1}{1 + \sum_i \exp(-\X_i) + \sum_{j,k} \exp(-\X_j - \X_k) + ...} \\
    &\approx \frac{1}{1 + \sum_i \exp(-\X_i)} ~~~(\text{for} ~~\X_i \rightarrow \infty~~ \forall i)
\end{align*}


We call the final expression the ``multivariate sigmoid'' (\emph{mv-sigmoid}) which serves as our approximation of \emph{mvn-cdf}, especially at the tails  of the distribution. While we expect estimators using \emph{mv-sigmoid} to approximate ones using \emph{mvn-cdf} only when $\matU \matU^\top = \mathbf{I}$, we find experimentally that the approximation works well even for practical values of the covariance matrix $\matU\matU^\top$. Using this approximation to substitute \emph{mv-sigmoid} for \emph{mvn-cdf} in the \ptaylor{} and \pmmse{} estimators yields the \ptaylormvs{} and \pmmsemvs{} estimators, respectively.


\subsubsection{Relationship between mv-sigmoid, softmax, and the Taylor estimator}
\label{app:softmax-explain}

A common method to estimate the confidence of model predictions is to use the softmax function applied to the logits $f_i(\X)$ of a model. We note that softmax is identical to \emph{mv-sigmoid} when directly applied to the logits of neural networks: 

\begin{align*}
    &\text{softmax}_t\left( f_i(\X) ~\Big\vert_{\substack{i = 1}}^{C} \right) = \frac{\exp(f_t(\X))}{\sum_{i=1}^C \exp(f_i(\X))} = \\& \frac{1}{1 + \sum\limits_{\substack{i=1\\i \neq t}}^{C} \exp(f_i(\X) - f_t(\X))} = \text{mv-sigmoid}\left( g_i(\X) ~\Big\vert_{\substack{i = 1\\i\neq t}}^{C} \right)
\end{align*}

Recall that $g_i(\X) = f_t(\X) - f_i(\X)$ is the decision boundary function. Note that this equivalence only holds for the specific case of logits. Comparing the expressions of softmax applied to logits above and the Taylor estimator, we notice that they are only different in that the Taylor estimator divides by the gradient norm, and uses the \emph{mvn-cdf} function instead of \emph{mv-sigmoid}. Given this similarity to the Taylor estimator, it is reasonable to ask whether softmax applied to logits (henceforth $p^\text{softmax}_{T}$ for softmax with temperature $T$) itself can be a ``good enough'' estimator of $p^\text{robust}_{\sigma}$ in practice. In other words, does $p^\text{softmax}_T$ well-approximate $p^\text{robust}_{\sigma}$ in certain settings?

In general, this cannot hold because softmax does not take in information about $\matU \matU^\top$, nor does it use the gradient information used in all of our estimators, although the temperature parameter $T$ can serve as a substitute for $\sigma$ in our expressions. In Appendix \ref{app:proofs}, we provide a theoretical result for a restricted linear setting where softmax can indeed match the behavior of \ptaylormvs{}, which happens precisely when $\matU \matU^\top = \mathbf{I}$ and all the class-wise gradients are equal. In the next section, we demonstrate empirically that the softmax estimator $p^{\text{softmax}}_T$ is a poor estimator of average robustness in practice.

\paragraph{The softmax estimator} We observe that for linear models with a specific noise perturbation $\sigma$, the common softmax function taken with respect to the output logits can be viewed as an estimator of \probust{}, albeit in a very restricted setting. Specifically,

\begin{lemma}
    For multi-class linear models $f(\X) = \mathbf{w}^\top \X + b$, such that the decision boundary weight norms $\| \U_i \|_2 = k, \forall i \in [1, C], i \neq t$,
    \begin{align*}
        p^\text{softmax}_{T} = p^\text{taylor\_mvs}_{\sigma}~~~~\text{where}~~~~T = \sigma k
    \end{align*}
\label{lemma:softmax}
\end{lemma}

\begin{proof} Consider softmax with respect to the $t^{th}$ output class and define $g_i(\X) = f_t(\X) - f_i(\X)$, with $f$ being the linear model logits. Using this, we first show that softmax is identical to \emph{mv-sigmoid}:

\begin{align*}
        p^\text{softmax}_T(\X) &= \text{softmax}_t(f_1(\X)/T, ..., f_C(\X)/T) \\
        &= \frac{\exp(f_t(\X)/T)}{\sum_i \exp(f_i(\X)/T)} \\ 
        &= \frac{1}{1 + \sum_{i; i\neq t} \exp((f_i(\X) - f_t(\X))/T)} \\ 
        &= ~\text{mv-sigmoid} \left[ g_i(\X)/T \tensor \right]
\end{align*}

Next, by denoting $\U_i = \W_t - \W_i$, each row has equal norm $\| \U_i \|_2 = \| \U_j \|_2, \forall i,j,t \in [1,...C]$ which implies: 

\begin{align*}
        p^\text{taylor\_mvs}_\sigma(\X) &= \text{mv-sigmoid} \left[ \frac{g_i(\X)}{\sigma \| \U_i \|_2} \tensor \right]\\ 
        &= \text{mv-sigmoid} \left[ g_i(\X)/T \tensor \right]~~ (\because \text{$T = \sigma k $})\\ 
        & = p^\text{softmax}_T(\X)
\end{align*}
\end{proof}

Lemma~\ref{lemma:softmax} indicates that the temperature parameter $T$ of softmax roughly corresponds to the $\sigma$ of the added Normal noise with respect to which local robustness is measured. Overall, this shows that under the restricted setting where the local linear model consists of decision boundaries with equal weight norms, the softmax outputs can be viewed as an estimator of the \ptaylormvs{} estimator, which itself is an estimator of \probust{}. However, due to the multiple levels of approximation, we can expect the quality of \psoftmax{}'s approximation of \probust{} to be poor in general settings (outside of the very restricted setting), so much so that in general settings, \probust{} and \psoftmax{} would be unrelated.

\subsection{Datasets}
\label{app:datasets}

The MNIST dataset consists of images of gray-scale handwritten digits spanning 10 classes: digits 0 through 9. The FashionMNIST (FMNIST) dataset consists of gray-scale images of articles of clothing spanning 10 classes: t-shirt, trousers, pullover, dress, coat, sandal, shirt, sneaker, bag, and ankle boot. For MNIST and FMNIST, each image is 28 pixels x 28 pixels. For MNIST and FMNIST, the training set consists of 60,000 images and the test set consists of 10,000 images.

The CIFAR10 dataset consists of color images of common objects and animals spanning 10 classes: airplane, car, bird, cat, deer, dog, frog, horse, ship, and truck. The CIFAR100 dataset consists of color images of common objects and animals spanning 100 classes: apple, bowl, chair, dolphin, lamp, mouse, plain, rose, squirrel, train, etc. For CIFAR10 and CIFAR100, each image is 3 pixels x 32 pixels x 32 pixels. For CIFAR10 and CIFAR100, the training set consists of 50,000 images and the test set consists of 10,000 images.

\subsection{Models}
\label{app:models}

For the MNIST and FMNIST, we train a linear model and a convolutional neural network (CNN) to perform 10-class classification. The linear model consists of one hidden layer with 10 neurons. The CNN consists of four hidden layers: one convolutional layer with 5x5 filters and 10 output channels, one convolutional layer 5x5 filters and 20 output channels, and one linear layer with 50 neurons, and one linear layer 10 neurons. 

For CIFAR10 and CIFAR100, we train a Vision Transformer model to perform 10-class and 100-class classification, respectively, by fine-tuning a Vision Transformer that was pre-trained on ImageNet (\url{https://huggingface.co/google/vit-base-patch16-224-in21k}) on each dataset. For these models, the test set consists of 100 images. We chose this number of datapoints so that \pmc{} would run within a reasonable amount of time.
We also train a ResNet18 model to perform 10-class and 100-class classification, respectively. The model architecture is described in \citep{he2016deep}. For CIFAR10 and CIFAR100, we also train the ResNet18 models using varying levels of gradient norm regularization to obtain models with varying levels of robustness. The larger the weight of gradient norm regularization ($\lambda$), the more robust the model.

All models were trained using stochastic gradient descent. Hyperparameters were selected to achieve decent model performance. The emphasis is on analyzing the estimators’ estimates of local robustness of each model, not on high model performance. Thus, we do not focus on tuning model hyperparameters. All models were trained for 200 epochs. The test set accuracy for each model is shown in Table~\ref{table:app-model-acc}.

\begin{table*}[ht!]
    \centering
    \begin{tabular}{c|c|c|c}
    Dataset      & Model  & $\lambda$  & Test set accuracy \\
    \midrule
    MNIST        & Linear  & 0 & 92\%                         \\
    MNIST        & CNN     & 0 & 99\%                         \\
    \midrule
    FashionMNIST & Linear  & 0 & 84\%                         \\
    FashionMNIST & CNN     & 0 & 91\%                         \\
    \midrule
    CIFAR10      & Vision Transformer & 0 & 99\%                         \\
    CIFAR10      & ResNet18 & 0 & 94\%                         \\
    CIFAR10      & ResNet18 & 0.0001 & 93\%                         \\
    CIFAR10      & ResNet18 & 0.001 & 90\%                         \\
    CIFAR10      & ResNet18 & 0.01 & 85\%                         \\
    \midrule
    CIFAR100     & Vision Transformer & 0 & 91\%                        \\
    CIFAR100     & ResNet18 & 0 & 76\%                        \\
    CIFAR100     & ResNet18 & 0.0001 & 74\%                         \\
    CIFAR100     & ResNet18 & 0.001 & 69\%                         \\
    CIFAR100     & ResNet18 & 0.01 & 60\%                         
    \end{tabular}
    \vspace*{3mm}
    \caption{Test set accuracy of models.}
    \label{table:app-model-acc}
\end{table*}

\subsection*{Experiments} 
Due to file size constraints, Section A.4 can be found in the Supplementary material.

\clearpage

\subsection{Experiments}
\label{app:experiments}
In this section, we provide the following additional experimental results:

\begin{enumerate}
    \item Figure \ref{app:convergence} shows results on the convergence of \pmc{}. \pmc{} takes a large number of samples to converge and is computationally inefficient.
    \item Figure \ref{app:convergence_mmse} shows results on the convergence of \pmmse{}. \pmmse{} takes only a few samples to converge and is more computationally inefficient than \pmc{}.
    \item Figure \ref{app:distribution_probust} shows the distribution of \probust{} as a function of $\sigma$. Consistent with theory in Section \ref{sec:methods}, (1) as noise increases, \probust{} decreases, and (2) \pmmse{} accurately estimates \pmc{}.
    \item Table \ref{app:runtimes} presents estimator runtimes. Our analytical estimators are more efficient than the naïve estimator (\pmc{}).
    \item Figure \ref{app:accuracy_probust} shows the accuracy of the analytical robustness estimators as a function of $\sigma$. \pmmse{} and \pmmsemvs{} are the best estimators of \probust{}, followed closely by \ptaylormvs{} and \ptaylor{}, trailed by \psoftmax{}.
    \item Figure \ref{app:accuracy_robust} shows the accuracy of the analytical estimators for robust models. For more robust models, the estimators compute \probust{} more accurately over a larger $\sigma$.
    \item Figures \ref{fig2:mvsig-mvncdf} and \ref{app:mvsigmoid} shows that \emph{mv-sigmoid} well-approximates \emph{mvn-cdf} over $\sigma$.
    \item Figure \ref{fig4:probust-and-psoftmax} shows that \psoftmax{} is not a good approximator of \probust{}.
    \item Figure \ref{app:robustness_bias} shows the distribution of \probust{} among classes (measured by \pmmse{}), revealing that models display robustness bias among classes. 
    \item Figures \ref{fig-supp:topk-vs-bottomk} and \ref{fig6:topk-vs-bottomk} show the application of \pmmse{} and \psoftmax{} to identification of robust and non-robust points. \probust{} better identifies robust and non-robust points than \psoftmax{}. 
    \item Figures \ref{app:noisy_mnist}, \ref{app:noisy_fmnist}, \ref{app:noisy_cifar10}, and \ref{app:noisy_cifar100} show examples of noisy images with the level of noise analyzed in our paper. Overall, the noise levels seem visually significant.
 \end{enumerate}

\begin{figure*}[htbp!]
    \centering
    \begin{flushleft}
        \hspace{-0.1cm}\rotatebox{90}{\hspace{-5.9cm}Relative error \hspace{1.2cm}Absolute error}
        \hspace{1.4cm}MNIST CNN
        \hspace{1.8cm} FMNIST CNN
        \hspace{1.4cm} CIFAR10 ResNet18
        \hspace{1cm} CIFAR100 ResNet18
    \end{flushleft}
         
    \begin{subfigure}{0.23\textwidth}
        \includegraphics[width=\linewidth]{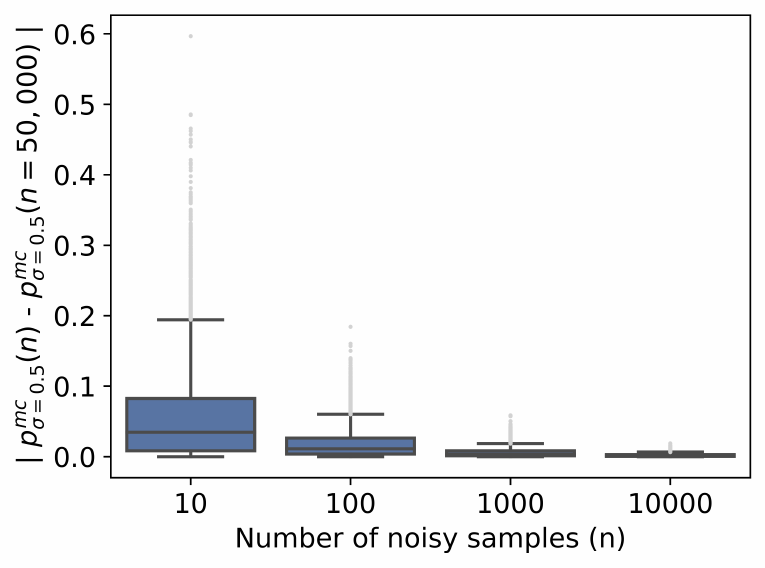}
    \end{subfigure}
    \begin{subfigure}{0.23\textwidth}
        \includegraphics[width=\linewidth]{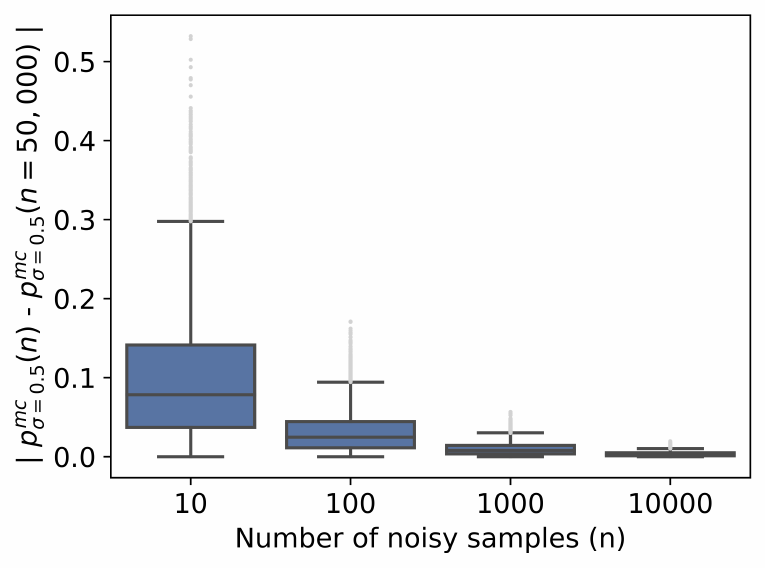}
    \end{subfigure}
    \begin{subfigure}{0.23\textwidth}
        \includegraphics[width=\linewidth]{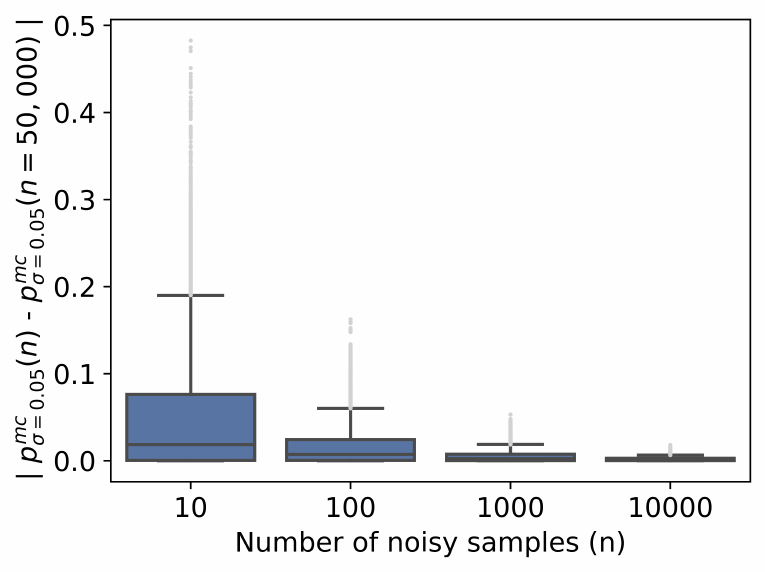}
    \end{subfigure}
    \begin{subfigure}{0.23\textwidth}
        \includegraphics[width=\linewidth]{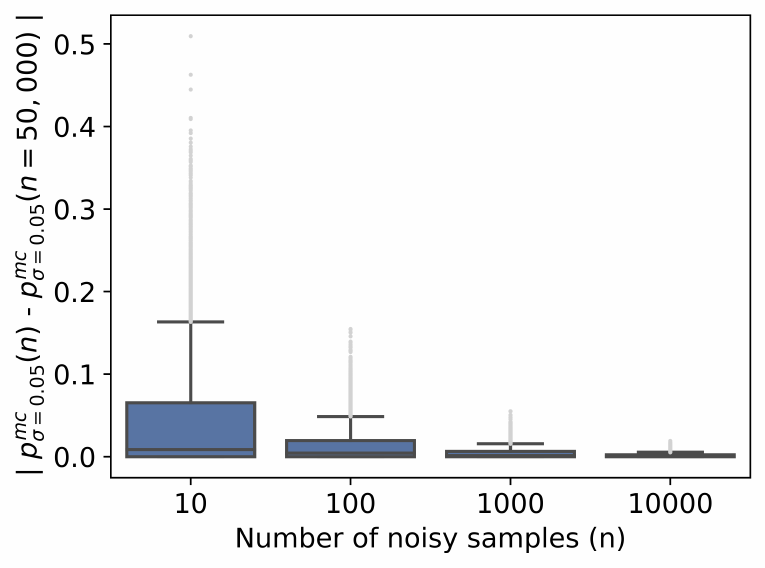}
    \end{subfigure}
    
    \begin{subfigure}{0.23\textwidth}
        \includegraphics[width=\linewidth]{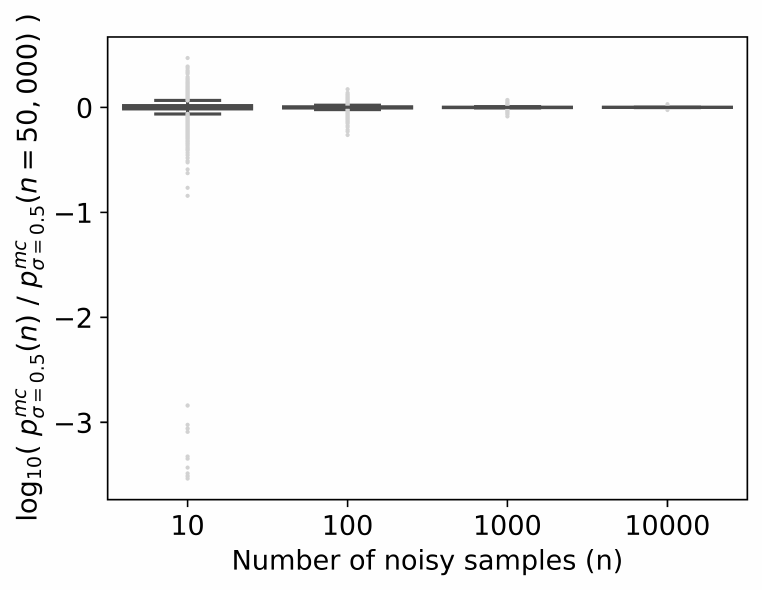}
    \end{subfigure}
    \begin{subfigure}{0.23\textwidth}
        \includegraphics[width=\linewidth]{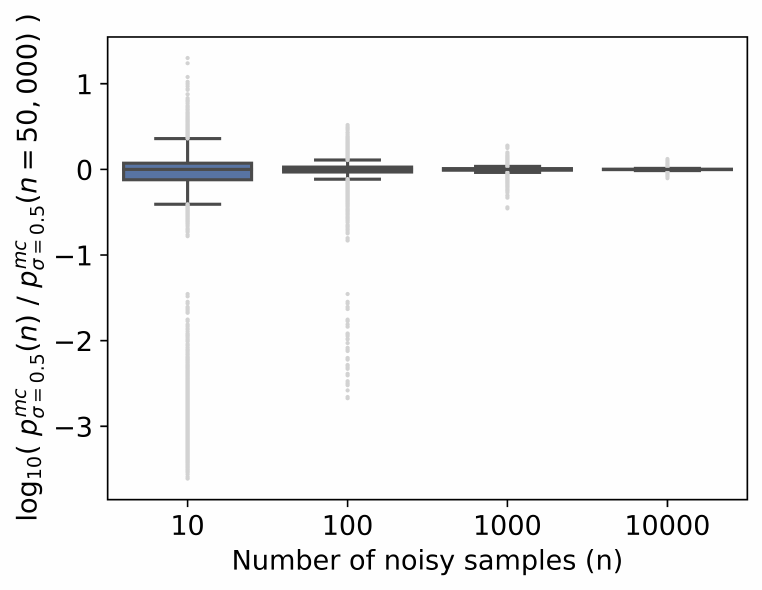}
    \end{subfigure}
    \begin{subfigure}{0.23\textwidth}
        \includegraphics[width=\linewidth]{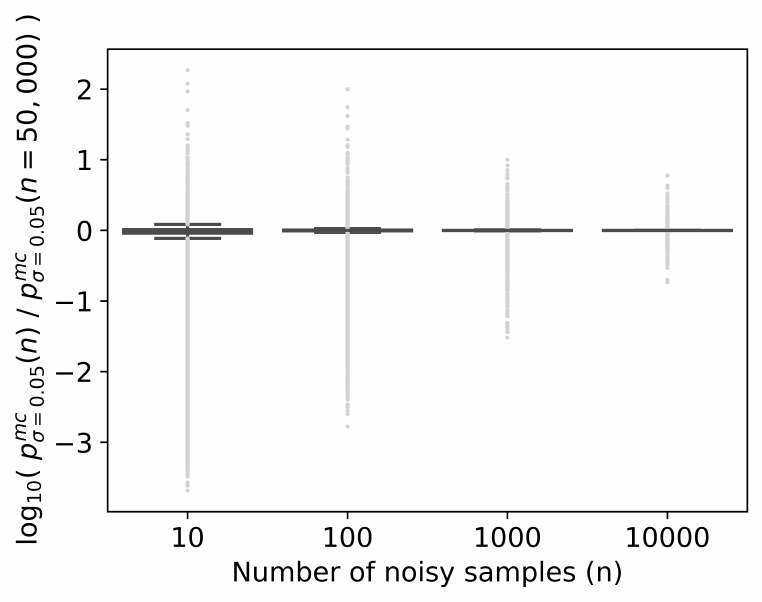}
    \end{subfigure}
    \begin{subfigure}{0.23\textwidth}
        \includegraphics[width=\linewidth]{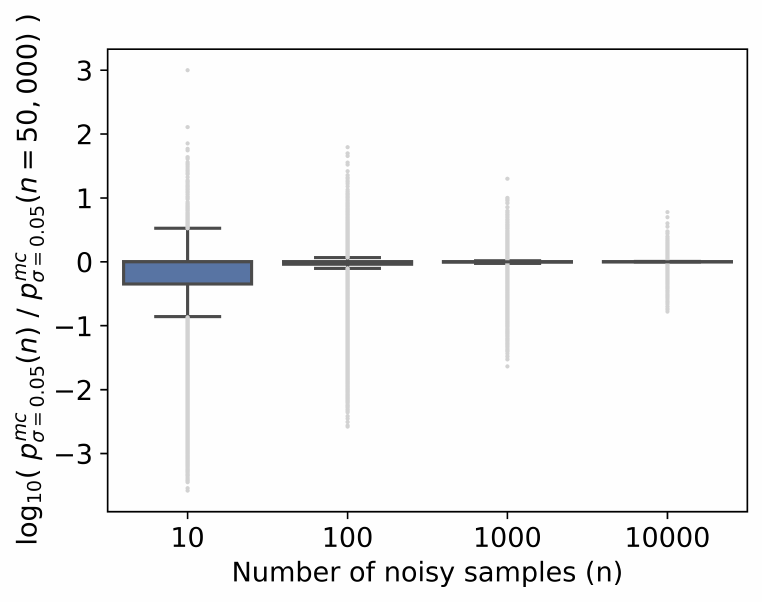}
    \end{subfigure}
    \caption{Convergence of \pmc{}. In practice, \pmc{} takes around $n=10,000$ samples to converge and is computationally inefficient.}\label{app:convergence}
\end{figure*}

\begin{figure*}[htbp!]
    \centering
    \begin{flushleft}
        \hspace{-0.1cm}\rotatebox{90}{\hspace{-5.8cm}Relative error \hspace{1.1cm}Absolute error}
        \hspace{1.4cm} MNIST CNN
        \hspace{1.9cm} FMNIST CNN
        \hspace{1.2cm} CIFAR10 ResNet18
        \hspace{0.9cm} CIFAR100 ResNet18
    \end{flushleft}
         
    \begin{subfigure}{0.23\textwidth}
        \includegraphics[width=\linewidth]{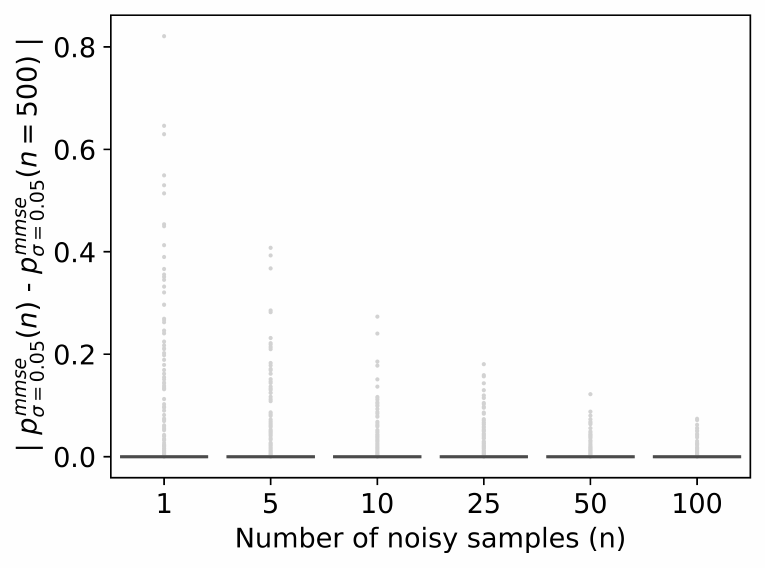}
    \end{subfigure}
    \begin{subfigure}{0.23\textwidth}
        \includegraphics[width=\linewidth]{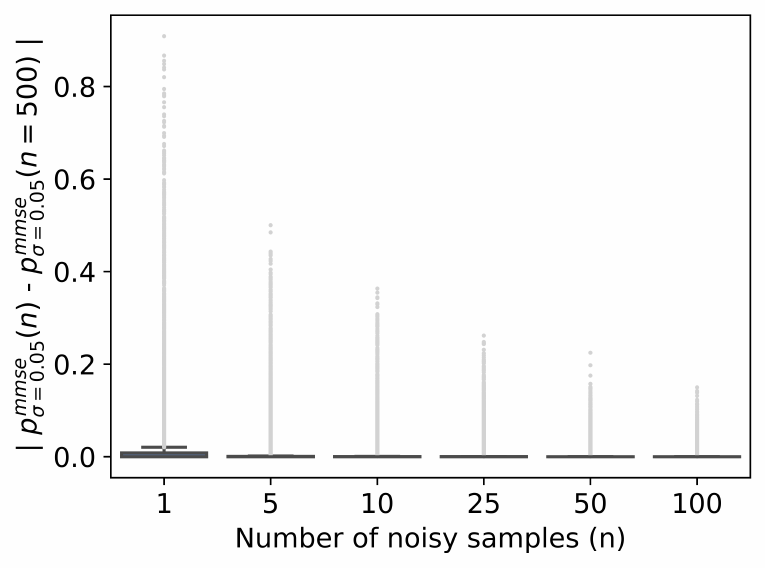}
    \end{subfigure}
    \begin{subfigure}{0.23\textwidth}
        \includegraphics[width=\linewidth]{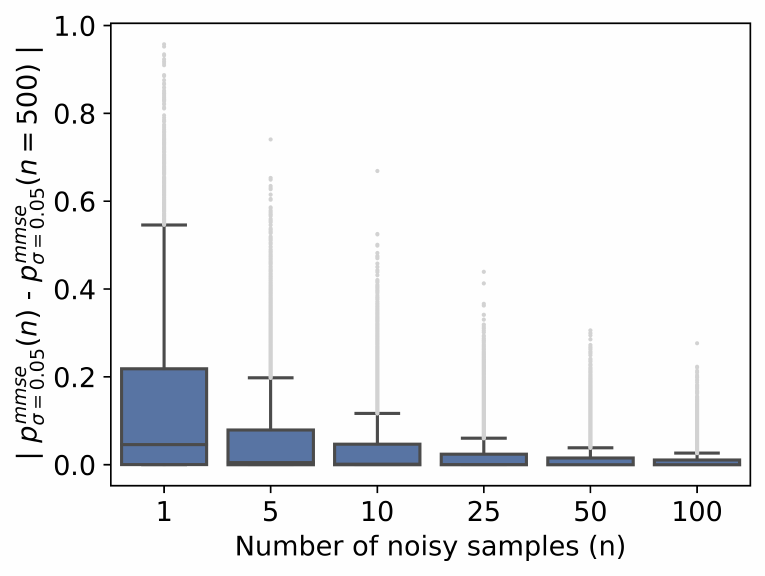}
    \end{subfigure}
    \begin{subfigure}{0.23\textwidth}
        \includegraphics[width=\linewidth]{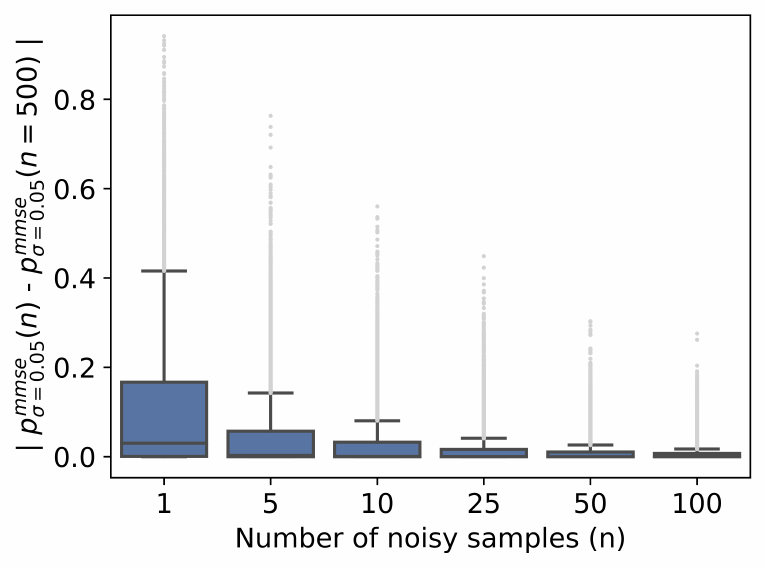}
    \end{subfigure}
    
    \begin{subfigure}{0.23\textwidth}
        \includegraphics[width=\linewidth]{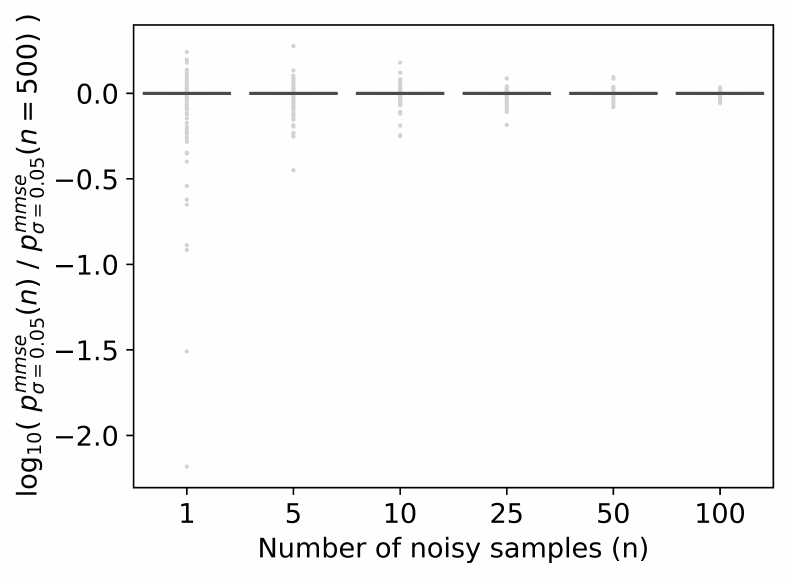}
    \end{subfigure}
    \begin{subfigure}{0.23\textwidth}
        \includegraphics[width=\linewidth]{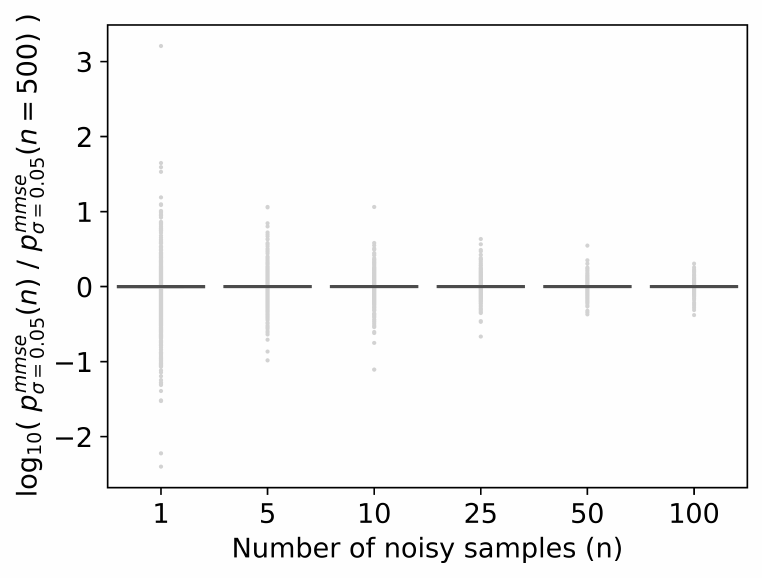}
    \end{subfigure}
    \begin{subfigure}{0.23\textwidth}
        \includegraphics[width=\linewidth]{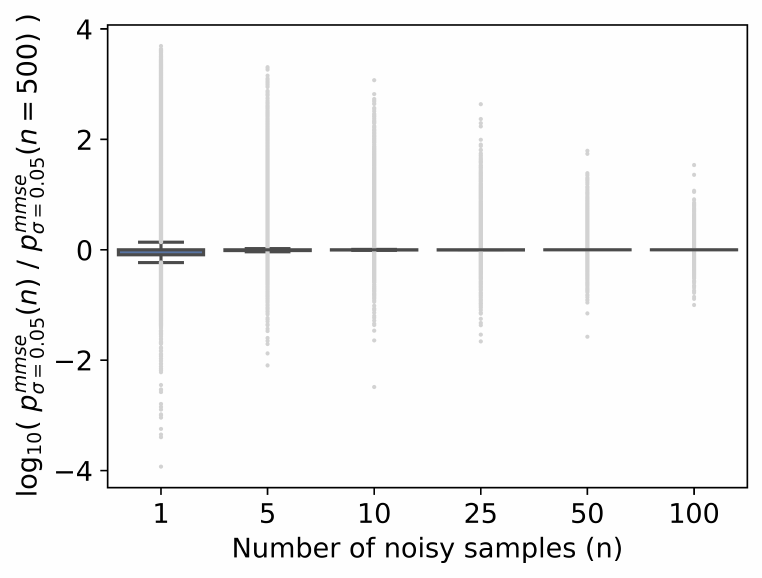}
    \end{subfigure}
    \begin{subfigure}{0.23\textwidth}
        \includegraphics[width=\linewidth]{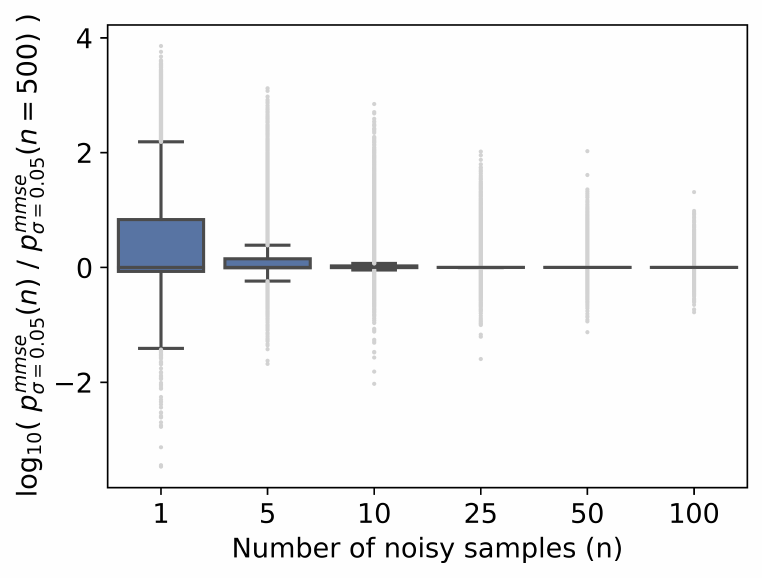}
    \end{subfigure}
    \caption{Convergence of \pmmse{}. In practice, \pmmse{} takes around $n=5-10$ samples to converge and is more computationally efficient than \pmc{}.} \label{app:convergence_mmse}
\end{figure*}



\begin{figure*}[htbp!]
    \centering
    \begin{subfigure}{0.3\textwidth}
        \includegraphics[width=\linewidth]{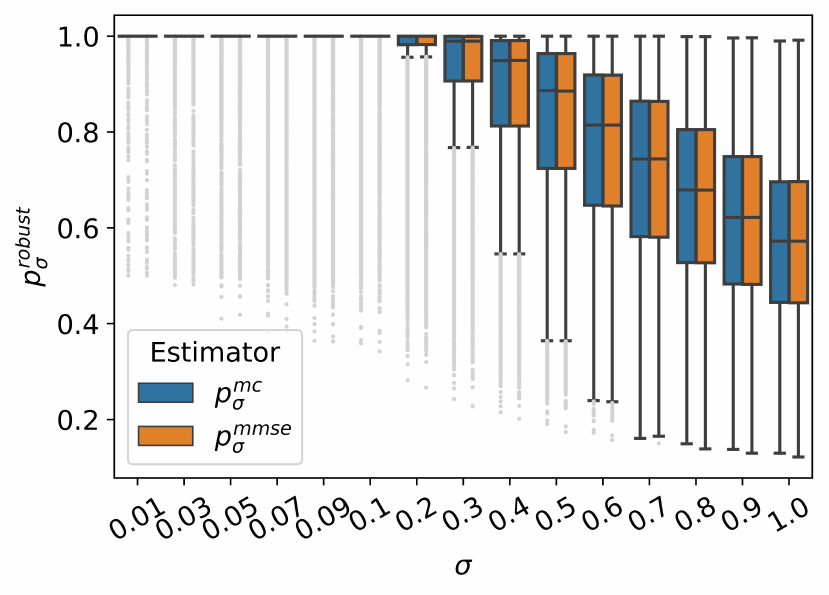}
        \caption{MNIST, Linear}
    \end{subfigure}
    \begin{subfigure}{0.3\textwidth}
        \includegraphics[width=\linewidth]{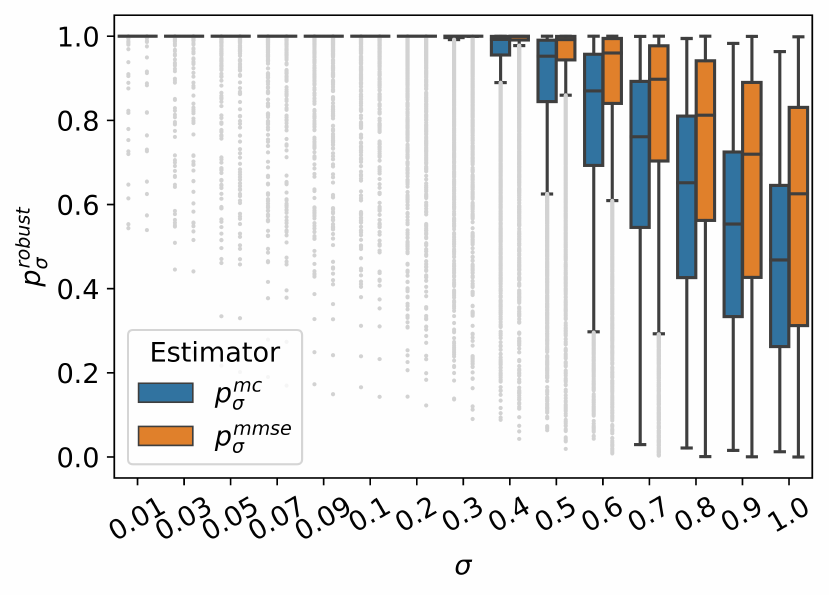}
        \caption{MNIST, CNN}
    \end{subfigure}
    \begin{subfigure}{0.3\textwidth}
        \includegraphics[width=\linewidth]{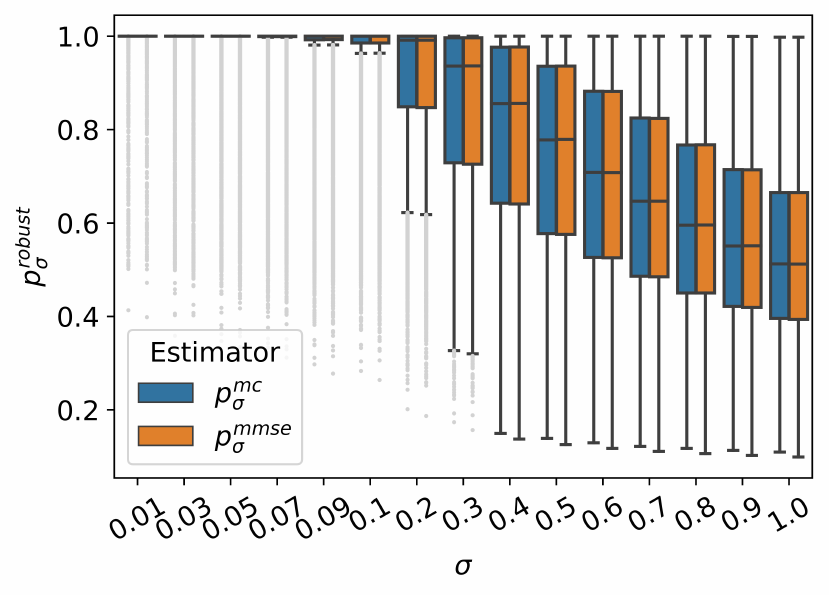}
        \caption{FMNIST, Linear}
    \end{subfigure}
    
    \begin{subfigure}{0.3\textwidth}
        \includegraphics[width=\linewidth]{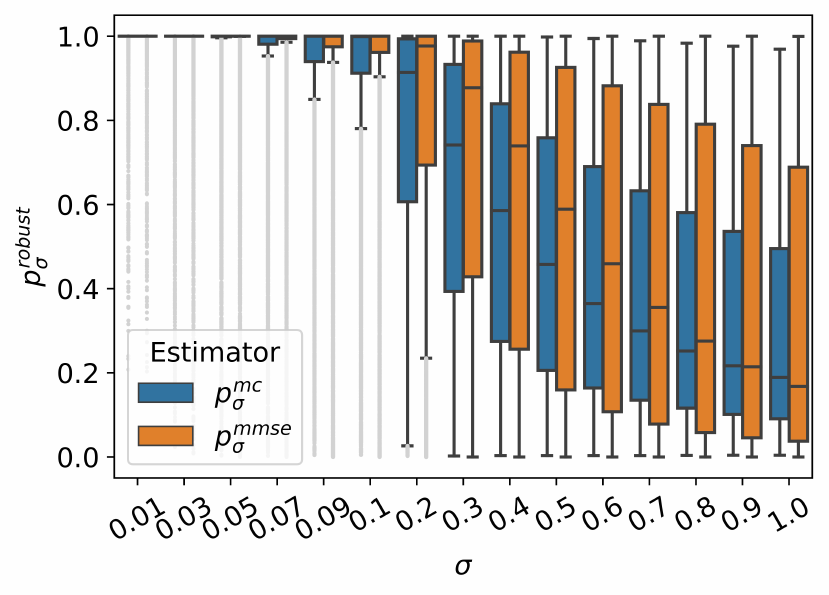}
        \caption{FMNIST, CNN}
    \end{subfigure}
    \begin{subfigure}{0.3\textwidth}
        \includegraphics[width=\linewidth]{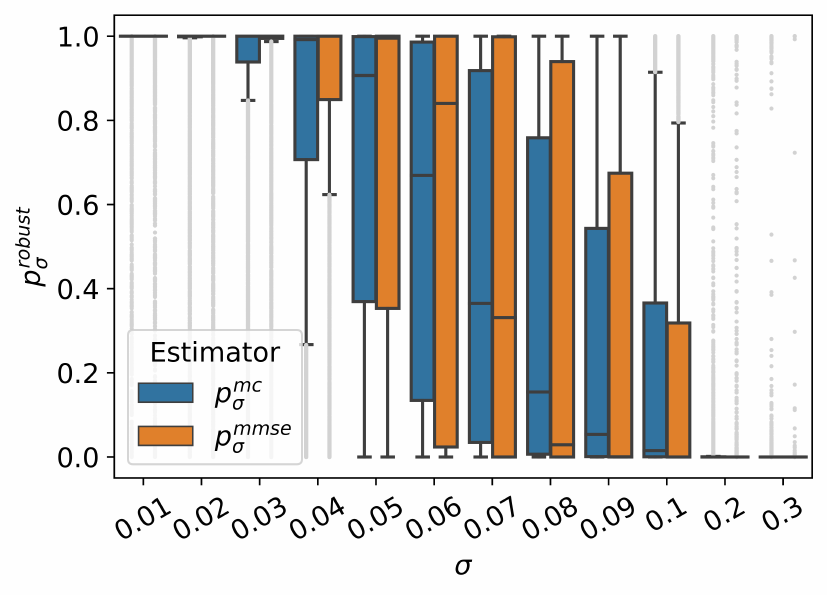}
        \caption{CIFAR10, ResNet18}
    \end{subfigure}
    \begin{subfigure}{0.3\textwidth}
        \includegraphics[width=\linewidth]{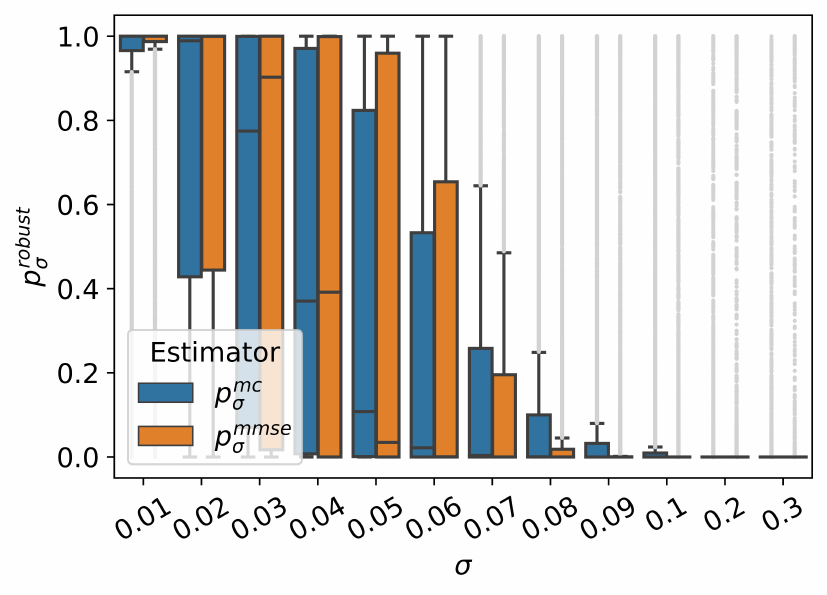}
        \caption{CIFAR100, ResNet18}
    \end{subfigure}

    \begin{subfigure}{0.3\textwidth}
        \includegraphics[width=\linewidth]{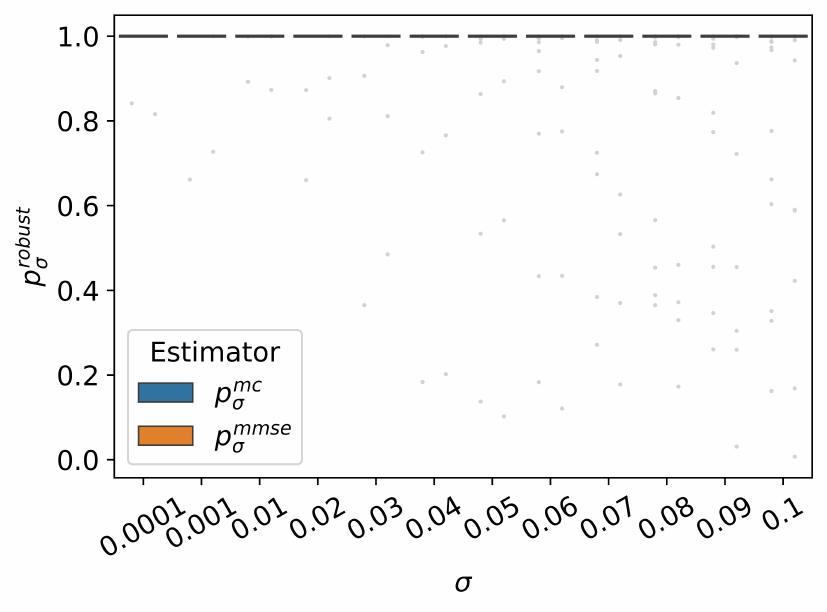}
        \caption{CIFAR10, Vision Transformer}
    \end{subfigure}
    \begin{subfigure}{0.3\textwidth}
        \includegraphics[width=\linewidth]{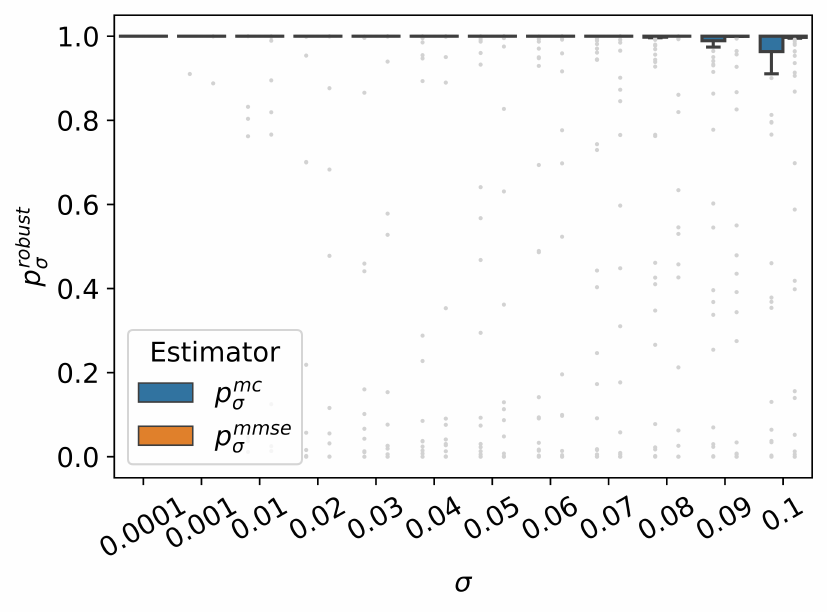}
        \caption{CIFAR100, Vision Transformer}
    \end{subfigure}
    \caption{Distribution of \probust{} over $\sigma$. As noise increases, \probust{} decreases. In addition, \pmmse{} accurately estimates \pmc{}.}\label{app:distribution_probust}
\end{figure*}

\begin{figure*}[htbp!]
    \centering
    \begin{subfigure}{0.45\textwidth}
        \includegraphics[width=\linewidth]{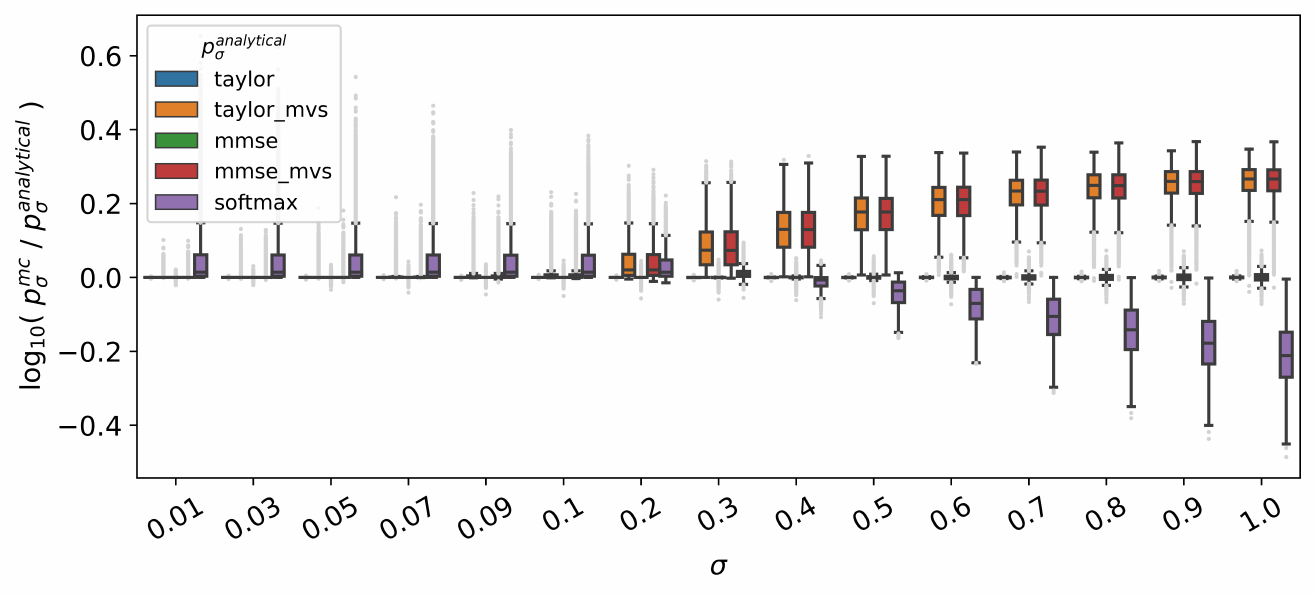}
        \caption{MNIST, Linear}
    \end{subfigure}
    \begin{subfigure}{0.45\textwidth}
        \includegraphics[width=\linewidth]{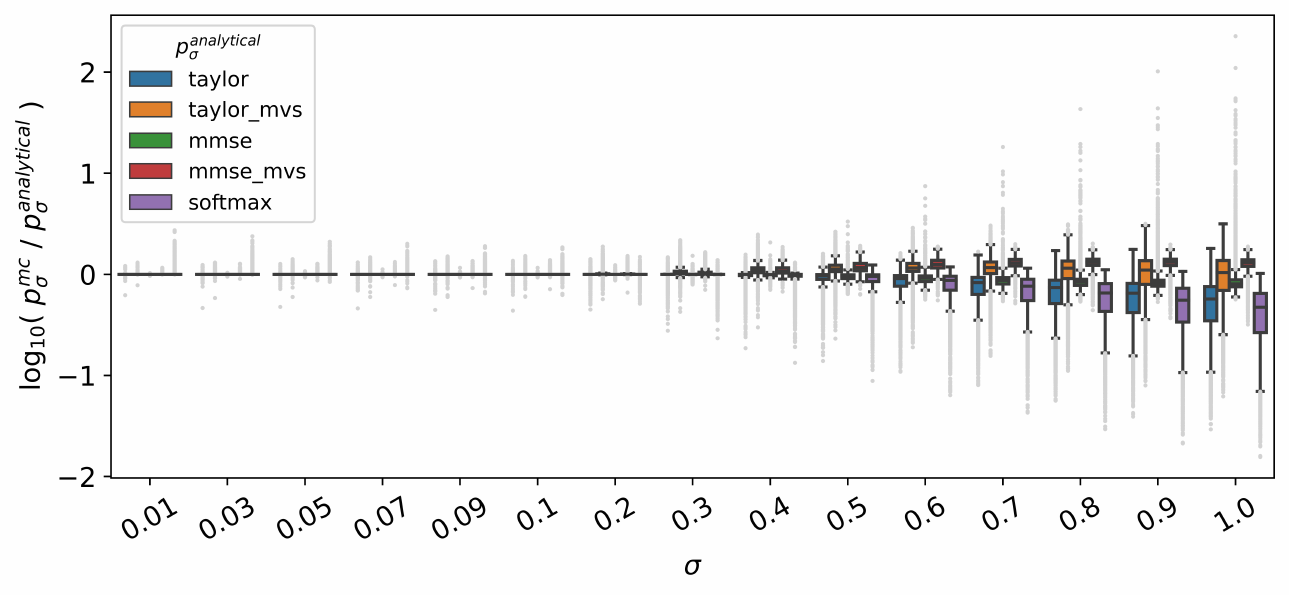}
        \caption{MNIST, CNN}
    \end{subfigure}

    \begin{subfigure}{0.45\textwidth}
        \includegraphics[width=\linewidth]{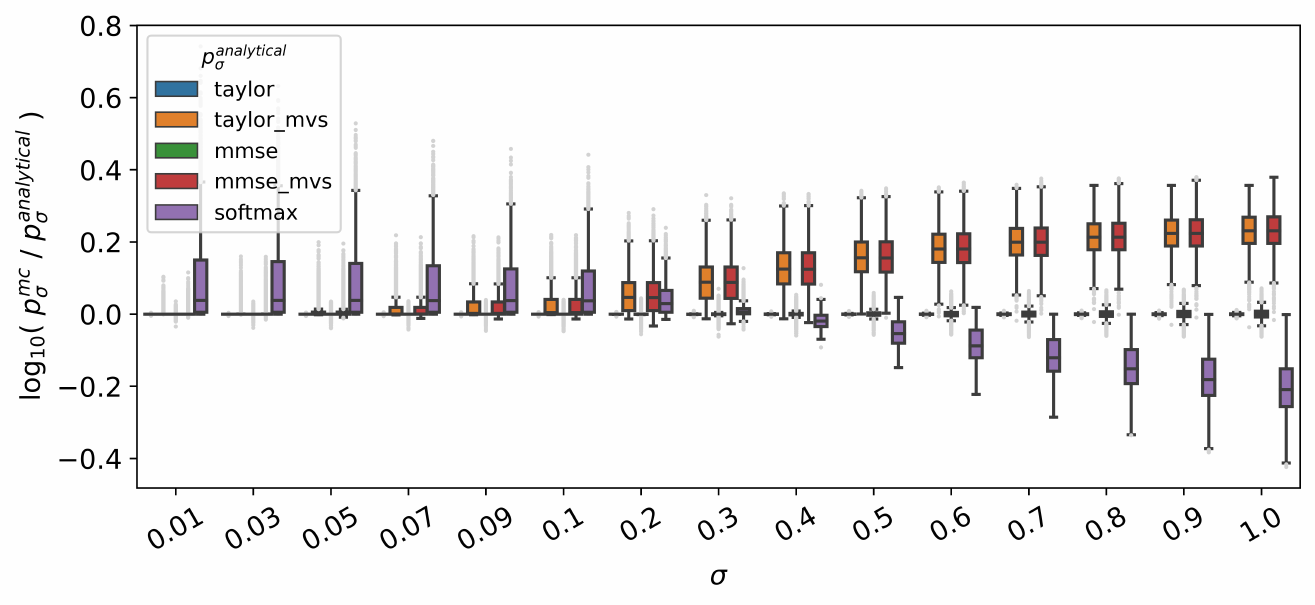}
        \caption{FMNIST, Linear}
    \end{subfigure}
    \begin{subfigure}{0.45\textwidth}
        \includegraphics[width=\linewidth]{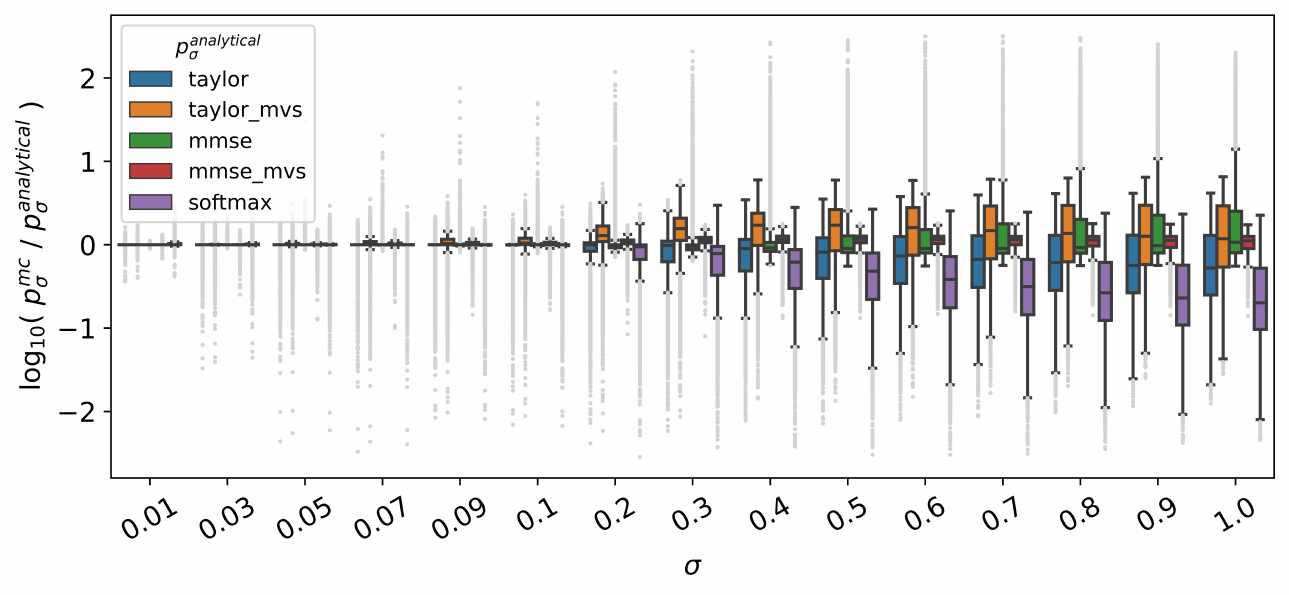}
        \caption{FMNIST, CNN}
    \end{subfigure}
    
    \begin{subfigure}{0.45\textwidth}
        \includegraphics[width=\linewidth]{figures/appendix/d_accuracy_of_estimators/cifar10_resnet18.pdf}
        \caption{CIFAR10, ResNet18}
    \end{subfigure}
    \begin{subfigure}{0.45\textwidth}
        \includegraphics[width=\linewidth]{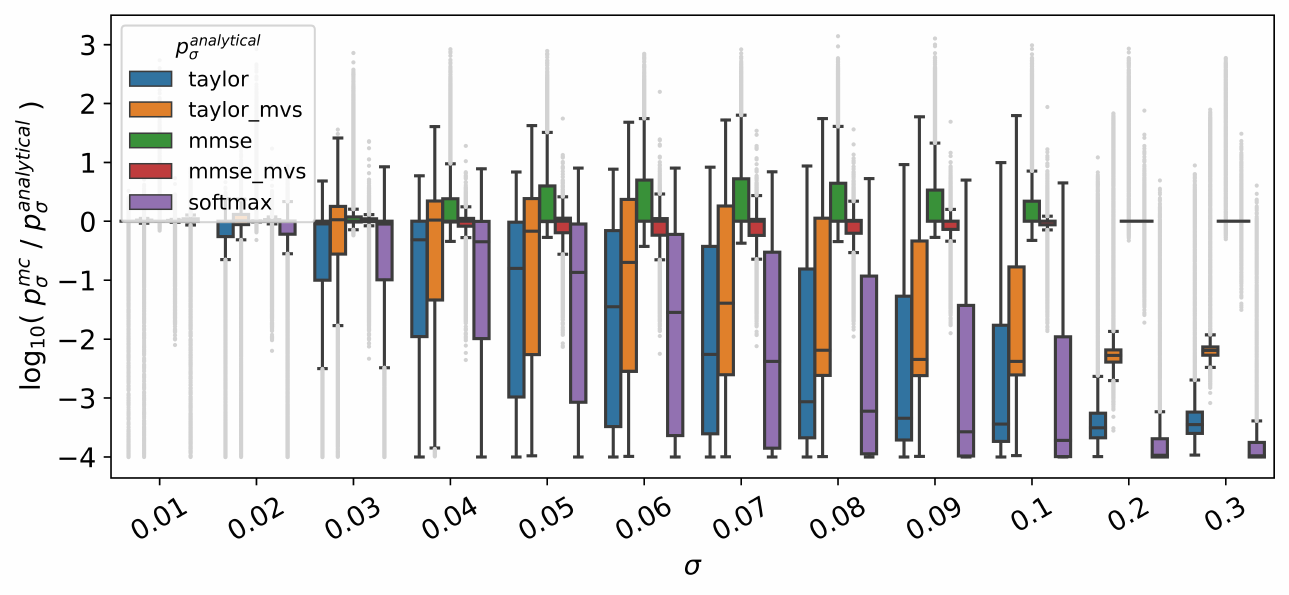}
        \caption{CIFAR100, ResNet18}
    \end{subfigure}

    \begin{subfigure}{0.45\textwidth}
        \includegraphics[width=\linewidth]{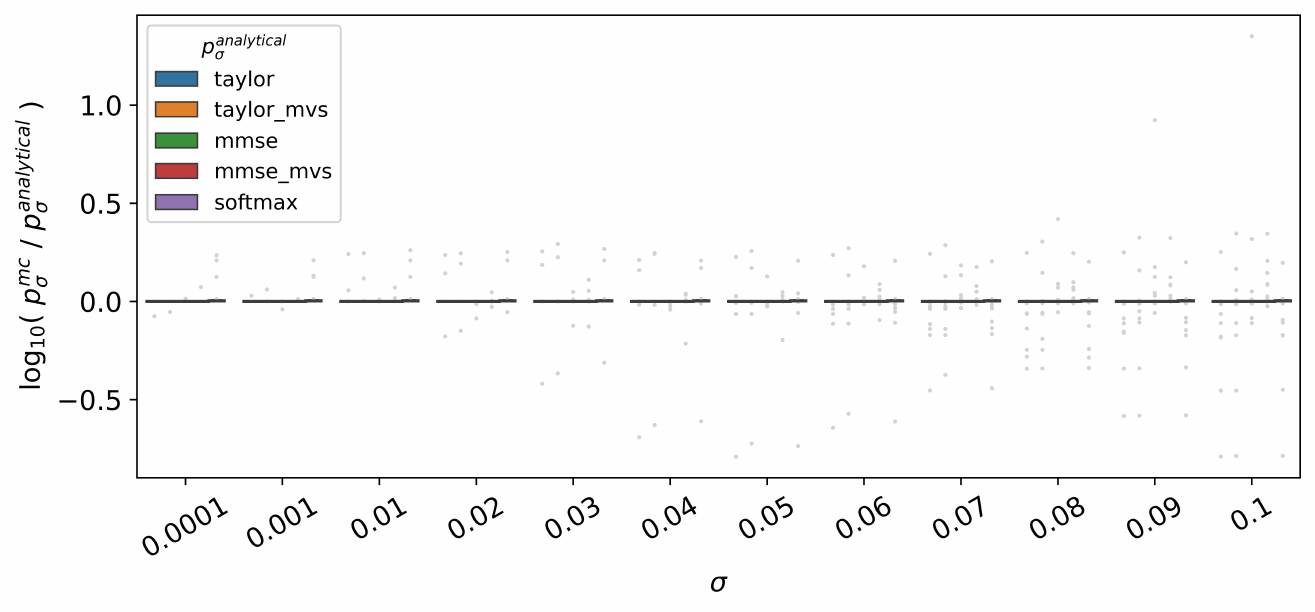}
        \caption{CIFAR10, Vision Transformer}
    \end{subfigure}
    \begin{subfigure}{0.45\textwidth}
        \includegraphics[width=\linewidth]{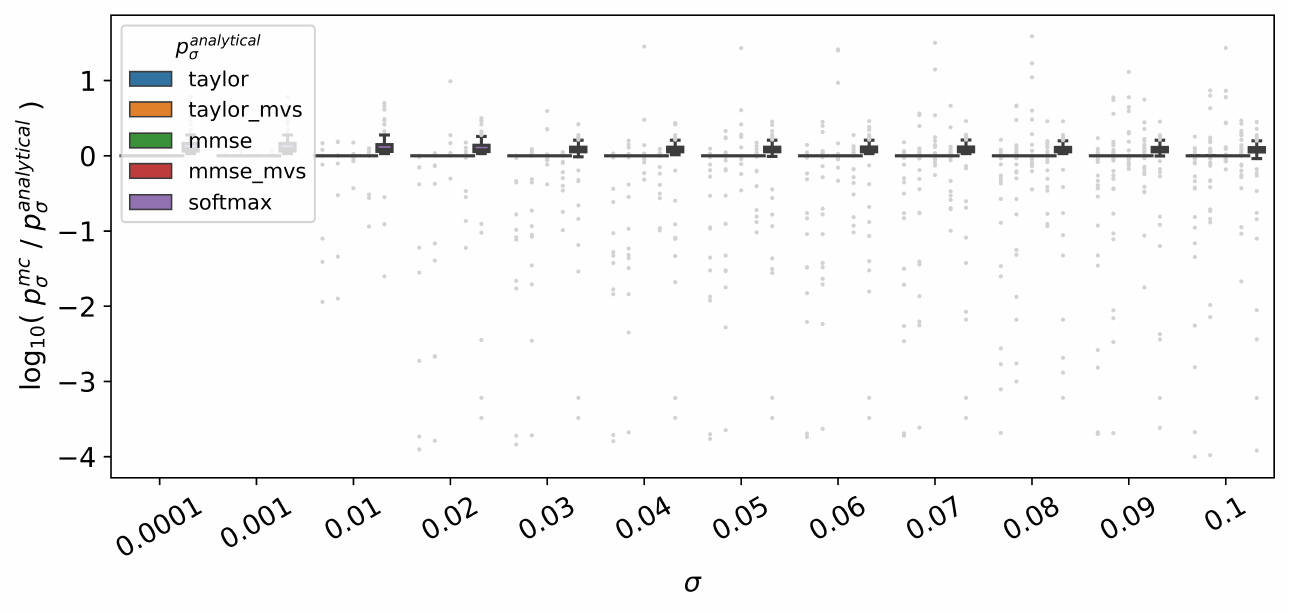}
        \caption{CIFAR100, Vision Transformer}
    \end{subfigure}
    \caption{Accuracy of \probust{} estimators over $\sigma$. The smaller the noise neighborhood $\sigma$, the more accurately the estimators compute \probust{}. \pmmse{} and \pmmsemvs{} are the best estimators of \probust{}, followed closely by \ptaylormvs{} and \ptaylor{}, trailed by \psoftmax{}.} \label{app:accuracy_probust}
\end{figure*}

\begin{figure*}[h]
    \centering
    \begin{subfigure}{0.3\textwidth}
        \includegraphics[width=\linewidth]{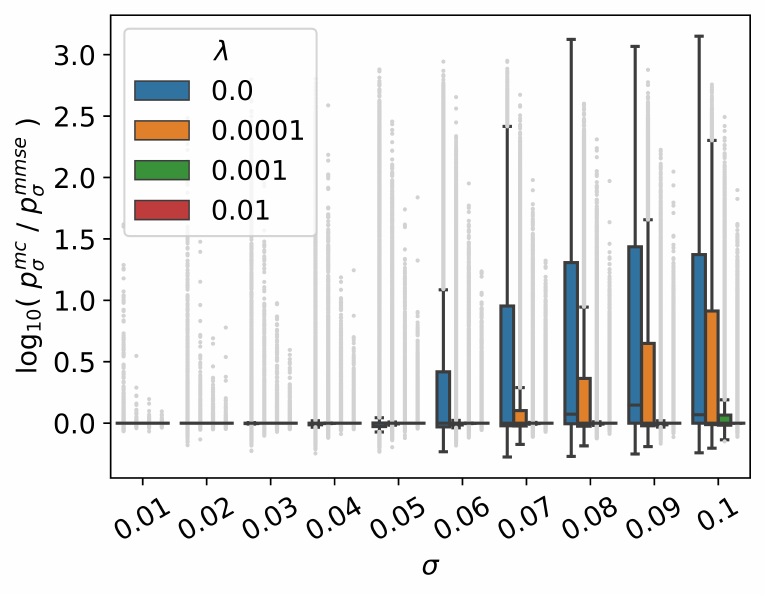}
        \caption{CIFAR10, ResNet18}
    \end{subfigure}
    \begin{subfigure}{0.3\textwidth}
        \includegraphics[width=\linewidth]{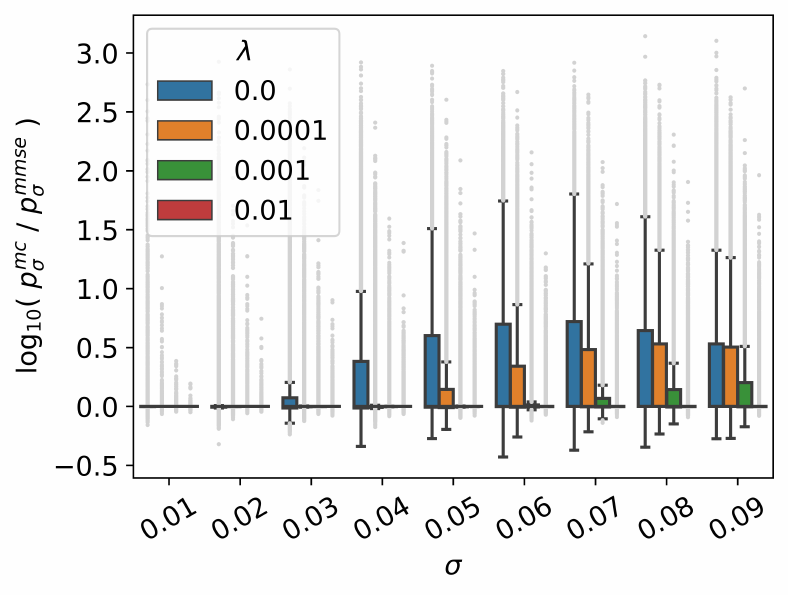}
        \caption{CIFAR100, ResNet18}
    \end{subfigure}
    \caption{Accuracy of \probust{} estimators over $\sigma$ for robust models. For more robust models, the estimators compute \probust{} more accurately over a larger $\sigma$.} \label{app:accuracy_robust}
\end{figure*}

\begin{figure}
  \centering
  \includegraphics[width=0.4\linewidth]{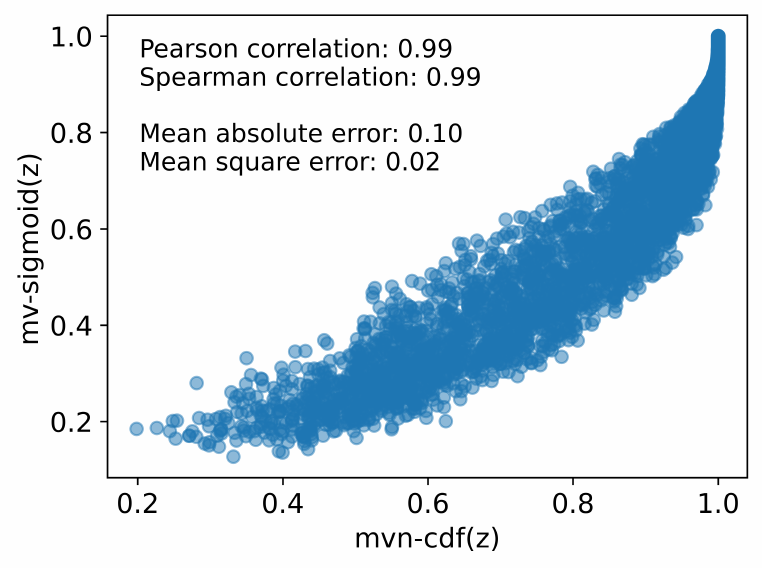}
  \caption{Correlation of \emph{mvn-cdf(z)} and \emph{mv-sigmoid(z)} for the CIFAR10 ResNet18 model. The formulation of $z$ is described in Section~\ref{sec:exp_correctness}. In practice, \emph{mv-sigmoid} approximates \emph{mvn-cdf} well.}
  \label{fig2:mvsig-mvncdf}
\end{figure}

\begin{figure*}[h]
    \centering
    \includegraphics[width=0.4\linewidth]{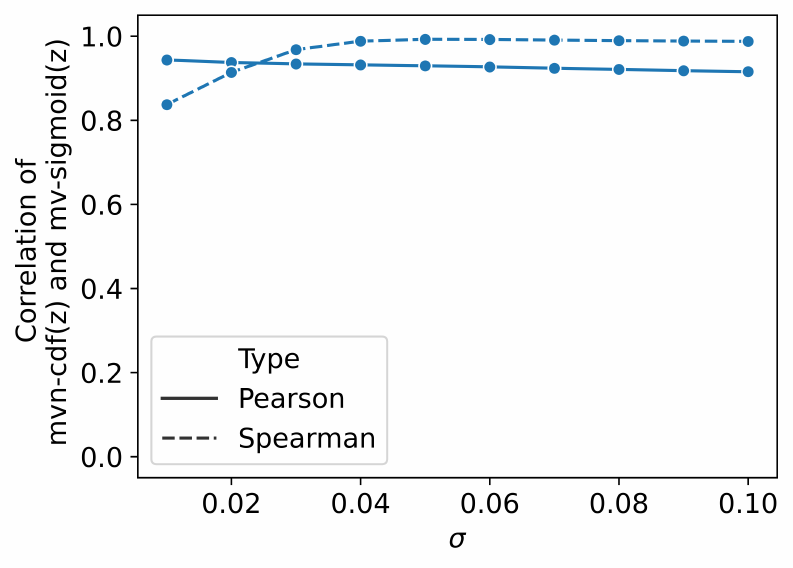}
    \caption{mv-sigmoid's approximation of mvn-cdf over $\sigma$. mv-sigmoid well-approximates mvn-cdf over $\sigma$.} \label{app:mvsigmoid}
\end{figure*}

\begin{figure*}[ht!]
    \centerline{
    \begin{subfigure}[h]{0.32\textwidth}
        \centering
        \includegraphics[width=\textwidth]{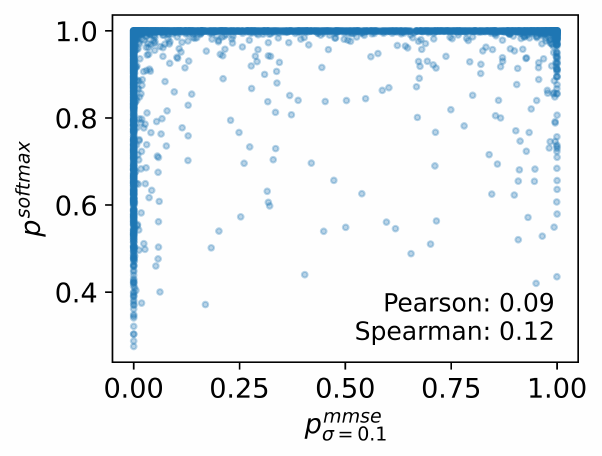}
        \captionsetup{justification=centering}
        \caption{CIFAR10 \\ Non-robust model ($\lambda=0$)}
        \label{fig4a:ps-nonrob-model}
    \end{subfigure}
    \begin{subfigure}[ht]{0.32\textwidth}
        \centering
        \includegraphics[width=\textwidth]{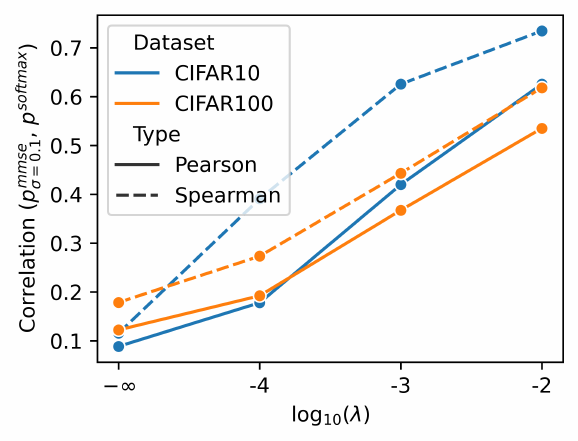}
        \captionsetup{justification=centering}
        \caption{CIFAR10 and CIFAR100 \\ Varying model robustness}
        \label{fig4b:ps-rob-models-lineplot}
    \end{subfigure}
    \begin{subfigure}[h]{0.32\textwidth}
        \centering
        \includegraphics[width=\textwidth]{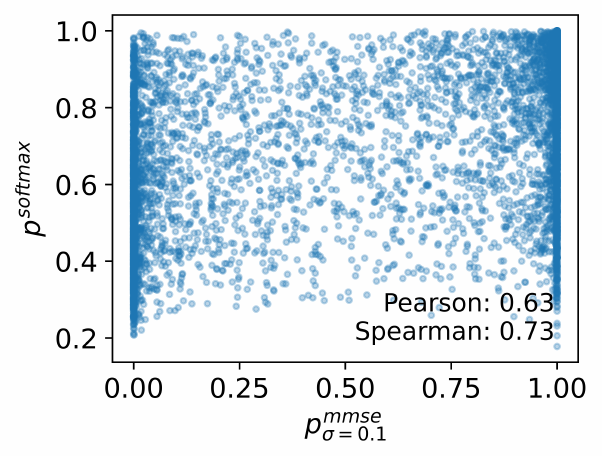}
        \captionsetup{justification=centering}
        \caption{CIFAR10 \\ Robust model ($\lambda=0.01$)}
        \label{fig4c:ps-rob-model}
    \end{subfigure}
    }
    \caption{Relationship between \probust{} and \psoftmax{} for CIFAR10 and CIFAR100 ResNet18 models. (a) For a non-robust model, \probust{} and \psoftmax{} are not strongly correlated. (b) As model robustness increases, the two quantities become more correlated. (c) However, even for robust models, the relationship between \probust{} and \psoftmax{} is mild. Together, these results indicate that, consistent with the theory in Section~\ref{sec:methods}, \psoftmax{} is not a good estimator for \probust{} in general settings.}
    \label{fig4:probust-and-psoftmax}
\end{figure*}

\begin{figure*}[h]
    \centering
         
    \begin{subfigure}{0.3\textwidth}
        \includegraphics[width=\linewidth]{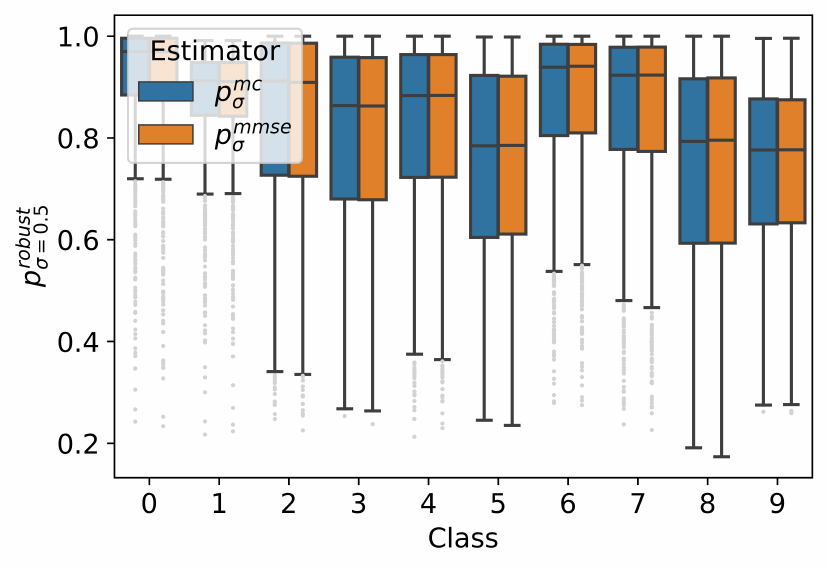}
        \caption{MNIST, Linear}
    \end{subfigure}
    \begin{subfigure}{0.3\textwidth}
        \includegraphics[width=\linewidth]{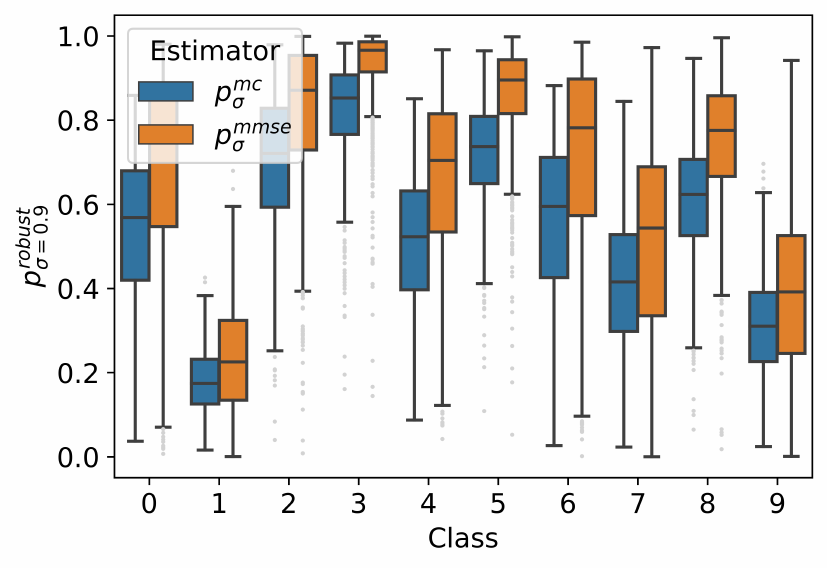}
        \caption{MNIST, CNN}
    \end{subfigure}
    \begin{subfigure}{0.3\textwidth}
        \includegraphics[width=\linewidth]{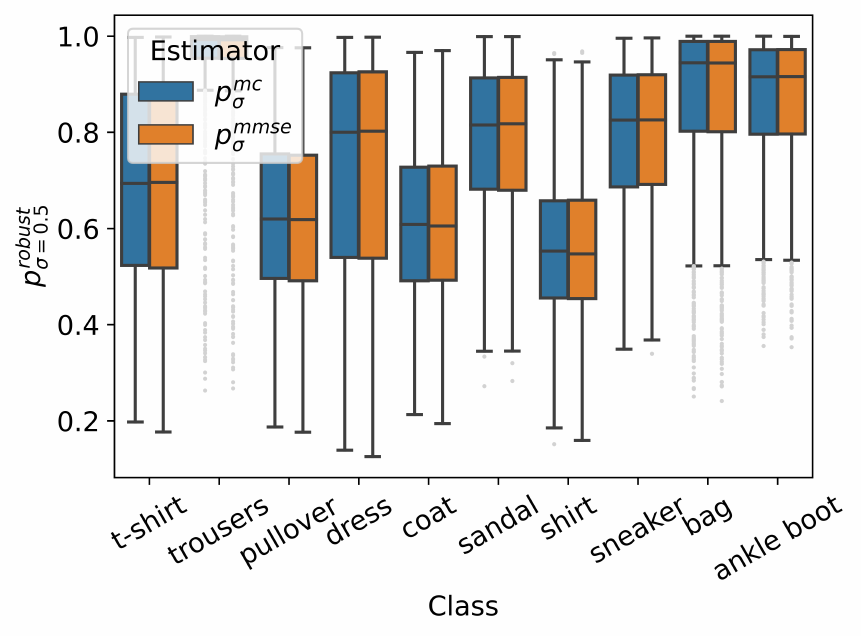}
        \caption{FMNIST, Linear}
    \end{subfigure}
    
    \begin{subfigure}{0.3\textwidth}
        \includegraphics[width=\linewidth]{figures/appendix/g_robustness_bias/fmnist_cnn_sigma0.9.pdf}
        \caption{FMNIST, CNN}
    \end{subfigure}
    \begin{subfigure}{0.3\textwidth}
        \includegraphics[width=\linewidth]{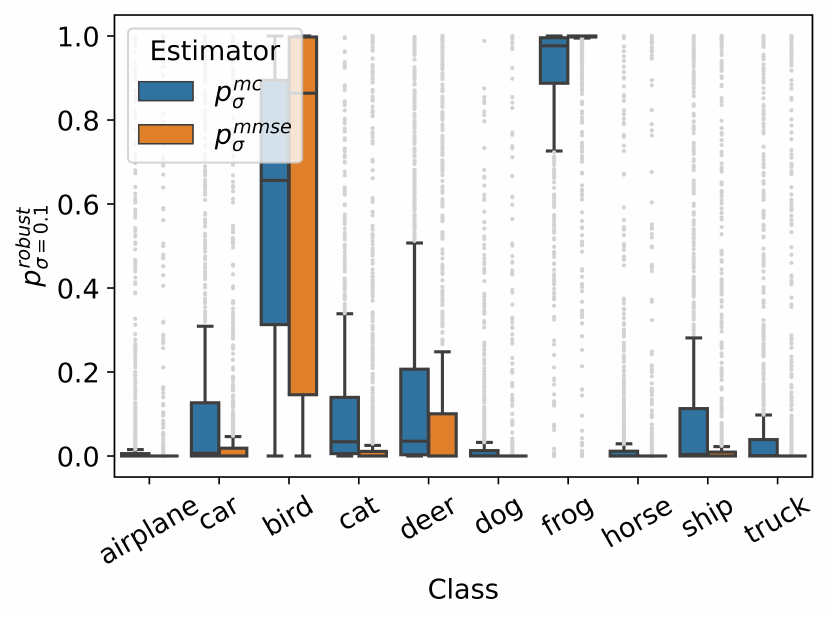}
        \caption{CIFAR10, ResNet18}
    \end{subfigure}
    \begin{subfigure}{0.3\textwidth}
        \includegraphics[width=\linewidth]{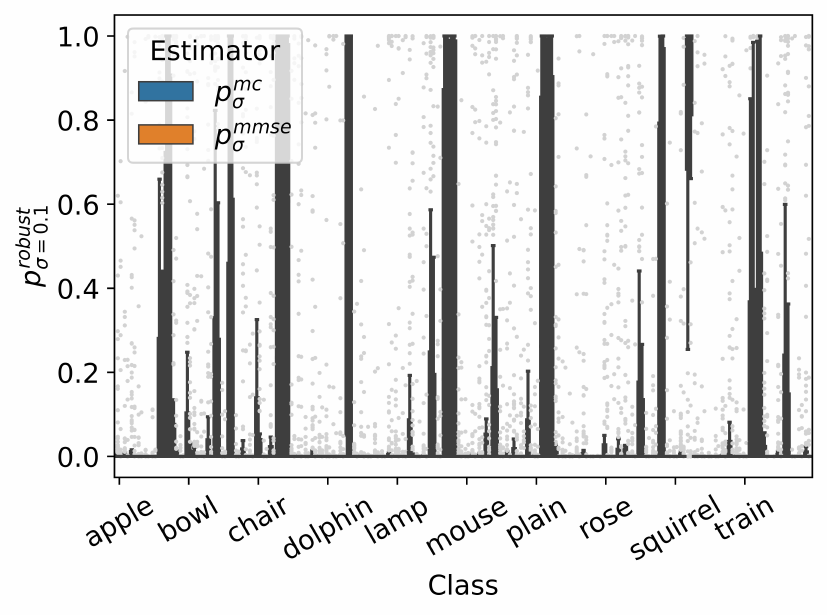}
        \caption{CIFAR100, ResNet18}
    \end{subfigure}
    \caption{Local robustness bias among classes. \probust{} reveals that the model is less locally robust for some classes than for others. The analytical estimator \pmmse{} properly captures this model bias.} \label{app:robustness_bias}
\end{figure*}

\begin{table*}[h]
\centering
\begin{tabular}{l|l|l|l|l|l}
    \multicolumn{2}{c}{}   & \multicolumn{2}{|c|}{CPU: Intel x86\_64}   & \multicolumn{2}{|c}{GPU: Tesla V100-PCIE-32GB} \\
    \midrule
    Estimator   & \# samples ($n$)   & Serial   & Batched   & Serial   & Batched \\
    \midrule
    \pmc{}   & \begin{tabular}[c]{@{}l@{}}$n=100$\\ $n=1000$\\ $n=10000$\end{tabular}               
             & \begin{tabular}[c]{@{}l@{}}0:00:59\\ 0:09:50\\ \textit{1:41:11}\end{tabular}                                               
             & \begin{tabular}[c]{@{}l@{}}0:00:42\\ 0:07:22\\ \textit{1:14:38}\end{tabular}                                                
             & \begin{tabular}[c]{@{}l@{}}0:00:12\\ 0:02:00\\ \textit{0:19:56}\end{tabular}                                                
             & \begin{tabular}[c]{@{}l@{}}0:00:01\\ 0:00:04\\ \textit{0:00:35}\end{tabular} \\
    \midrule
    \ptaylor{}   & N/A
                 & 0:00:08                                                                                                                      
                 & 0:00:07                                                                                                                      
                 & 0:00:02                                                                                                                      
                 & $<$ 0:00:01 \\
    \midrule
    \ptaylormvs{}   & N/A
                    & 0:00:08                                                                                                                   
                    & 0:00:07                                                                                                                  
                    & 0:00:01                                                                                                                   
                    & $<$ 0:00:01 \\
    \midrule
    \pmmse{}   & \begin{tabular}[c]{@{}l@{}}$n=1$\\ $n=5$\\ $n=10$\\ $n=25$\\ $n=50$\\ $n=100$\end{tabular} 
               & \begin{tabular}[c]{@{}l@{}}0:00:08\\ \textit{0:00:41}\\ 0:01:21\\ 0:03:21\\ 0:06:47\\ 0:13:57\end{tabular} 
               & \begin{tabular}[c]{@{}l@{}}0:00:10\\ \textit{0:00:31}\\ 0:01:02\\ 0:02:44\\ 0:05:38\\ 0:11:31\end{tabular} 
               & \begin{tabular}[c]{@{}l@{}}0:00:02\\ \textit{0:00:06}\\ 0:00:11\\ 0:00:26\\ 0:00:51\\ 0:01:42\end{tabular} 
               & \begin{tabular}[c]{@{}l@{}}0:00:02\\ \textit{0:00:02}\\ 0:00:02\\ 0:00:03\\ 0:00:04\\ 0:00:06\end{tabular} \\
    \midrule
    \pmmsemvs{}   & \begin{tabular}[c]{@{}l@{}}$n=1$\\ $n=5$\\ $n=10$\\ $n=25$\\ $n=50$\\ $n=100$\end{tabular} 
                  & \begin{tabular}[c]{@{}l@{}}0:00:08\\ \textit{0:00:41}\\ 0:01:21\\ 0:03:24\\ 0:06:47\\ 0:13:28\end{tabular} 
                  & \begin{tabular}[c]{@{}l@{}}0:00:08\\ \textit{0:00:32}\\ 0:01:00\\ 0:02:37\\ 0:05:35\\ 0:11:32\end{tabular} 
                  & \begin{tabular}[c]{@{}l@{}}0:00:01\\ \textit{0:00:05}\\ 0:00:10\\ 0:00:25\\ 0:00:51\\ 0:01:42\end{tabular} 
                  & \begin{tabular}[c]{@{}l@{}}0:00:01\\ \textit{0:00:01}\\ 0:00:02\\ 0:00:02\\ 0:00:03\\ 0:00:06\end{tabular} \\
    \midrule
    \psoftmax{}   & N/A                                                                             
                  & 0:00:01                                                                                                                              
                  & $<$ 0:00:01                                                                                                                              
                  & $<$ 0:00:01                                                                                                                              
                  & $<$ 0:00:01                                                                                                                             
\end{tabular}
\caption{Runtimes of each \probust{} estimator. Each estimator computes \probustwsigma{0.1} for the CIFAR10 ResNet18 model for 50 data points. For estimators that use sampling, the row with the minimum number of samples necessary for convergence is italicized. Runtimes are in the format of hour:minute:second. The analytical estimators (\ptaylor{}, \ptaylormvs{}, \pmmse{}, and \pmmsemvs{}) are more efficient than the naïve estimator (\pmc{}).} \label{app:runtimes}
\end{table*}

\subsubsection{\probust{} identifies images that are robust to and images
that are vulnerable to random noise}

For each dataset, we train a simple CNN to distinguish between images with high and low \pmmse{}. We train the same CNN to also distinguish between images with high and low \psoftmax{}. The CNN consists of two convolutional layers and two fully-connected feedforward layers with a total of 21,878 parameters. For a given dataset, for each class, we take the images with the top-25 and bottom-25 \pmmse{} values. This yields 500 images for CIFAR10 (10 classes x 50 images per class) and 5,000 images for CIFAR100 (100 classes x 50 images per class). We also perform the same steps using \psoftmax{}, yielding another 500 images for CIFAR10 and another 5,000 images for CIFAR100. For each dataset, the train/test split is 90\%/10\% of points. 

Then, we compare the performance of the two models. For CIFAR10, the test set accuracy for the \pmmse{} CNN is 0.92 while that for the \psoftmax{} CNN is 0.58. For CIFAR100, the test set accuracy for the \pmmse{} CNN is 0.74 while that for the \psoftmax{} CNN is 0.55. The higher the test set accuracy of a CNN, the better the CNN distinguishes between images. Thus, the results indicate that \probust{} better identifies images that are robust to and vulnerable to random noise than \psoftmax{}.

We also provide additional visualizations of images with the highest and lowest \probust{} and images with the highest and lowest \psoftmax{}.

\subsubsection{Softmax probability is not a good proxy for average-case robustness}

To examine the relationship between \probust{} and \psoftmax{}, we calculate \pmmse{} and \psoftmax{} for CIFAR10 and CIFAR100 models of varying levels of robustness, and measure the correlation of their values and ranks using Pearson and Spearman correlations. Results are in Appendix~\ref{app:experiments} (Figure~\ref{fig4:probust-and-psoftmax}). For a non-robust model, \probust{} and \psoftmax{} are not strongly correlated (Figure~\ref{fig4a:ps-nonrob-model}). As model robustness increases, the two quantities become more correlated (Figures~\ref{fig4b:ps-rob-models-lineplot} and~\ref{fig4c:ps-rob-model}). However, even for robust models, the relationship between the two quantities is mild (Figure~\ref{fig4c:ps-rob-model}). That \probust{} and \psoftmax{} are not strongly correlated is consistent with the theory in Section~\ref{sec:methods}: in general settings, \psoftmax{} is not a good estimator for \probust{}.

\newpage
While working on the paper, we hypothesized that \probust{} (e.g., \pmmse{}) might be correlated with model accuracy. However, we did not find this in practice. Instead, what we find is that \probust{} succeeds in identifying canonical data points of a class, and does so much better than \psoftmax{}. We first assess this finding through visual inspection, finding that images with higher \probust{} tend to be more canonical and clear images, and that this distinction is less apparent for \psoftmax{} (Figures~\ref{fig4:topk-vs-bottomk-main} and \ref{fig-supp:topk-vs-bottomk}). We then use a model to classify these images as an additional, more objective assessment of this pattern (as discussed in Section~\ref{subsec:case-studies}).

\begin{figure*}[htbp!]
    \vspace{1cm}
    \centering
    \begin{flushleft}
        \hspace{-0.1cm}\rotatebox{90}{\hspace{-6.5cm}Car \hspace{3cm}Boat}
        \hspace{1.3cm}Lowest \pmmsewsigma{0.1}
        \hspace{1.6cm} Highest \pmmsewsigma{0.1}
        \hspace{1.5cm} Lowest \psoftmax{}
        \hspace{1.5cm} Highest \psoftmax{}
    \end{flushleft}
         
    \begin{subfigure}{0.23\textwidth}
        \includegraphics[width=\linewidth, trim={0.2cm, 0.2cm, 0.2cm, 0.2cm}]{figures/appendix/h_topk_bottomk/cifar10_resnet18_p_mmse_sigma0.1_class8_bottomk.pdf}
    \end{subfigure}
    \begin{subfigure}{0.23\textwidth}
        \includegraphics[width=\linewidth, trim={0.2cm, 0.2cm, 0.2cm, 0.2cm}]{figures/appendix/h_topk_bottomk/cifar10_resnet18_p_mmse_sigma0.1_class8_topk.pdf}
    \end{subfigure}
    \begin{subfigure}{0.23\textwidth}
        \includegraphics[width=\linewidth, trim={0.2cm, 0.2cm, 0.2cm, 0.2cm}]{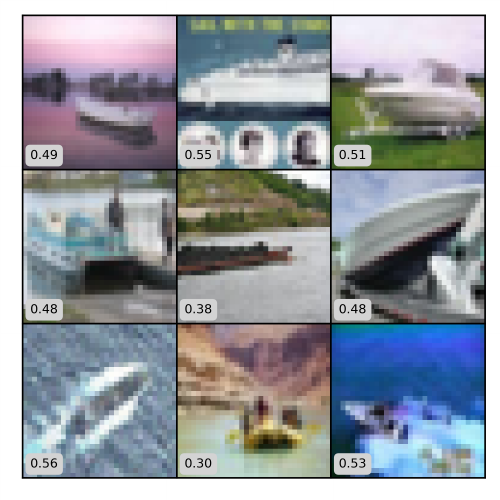}
    \end{subfigure}
    \begin{subfigure}{0.23\textwidth}
        \includegraphics[width=\linewidth, trim={0.2cm, 0.2cm, 0.2cm, 0.2cm}]{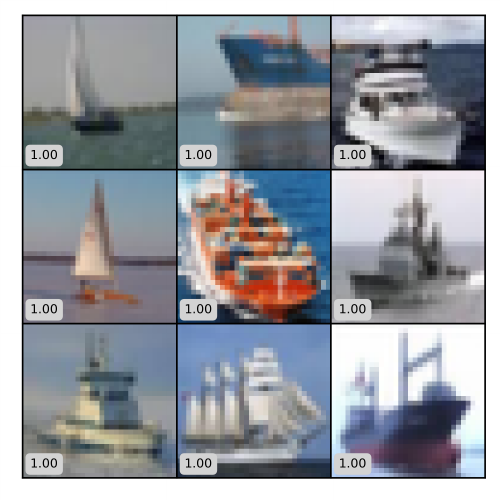}
    \end{subfigure}
    
    \begin{subfigure}{0.23\textwidth}
        \includegraphics[width=\linewidth, trim={0.2cm, 0.2cm, 0.2cm, 0.2cm}]{figures/appendix/h_topk_bottomk/cifar10_resnet18_p_mmse_sigma0.1_class1_bottomk.pdf}
    \end{subfigure}
    \begin{subfigure}{0.23\textwidth}
        \includegraphics[width=\linewidth, trim={0.2cm, 0.2cm, 0.2cm, 0.2cm}]{figures/appendix/h_topk_bottomk/cifar10_resnet18_p_mmse_sigma0.1_class1_topk.pdf}
    \end{subfigure}
    \begin{subfigure}{0.23\textwidth}
        \includegraphics[width=\linewidth, trim={0.2cm, 0.2cm, 0.2cm, 0.2cm}]{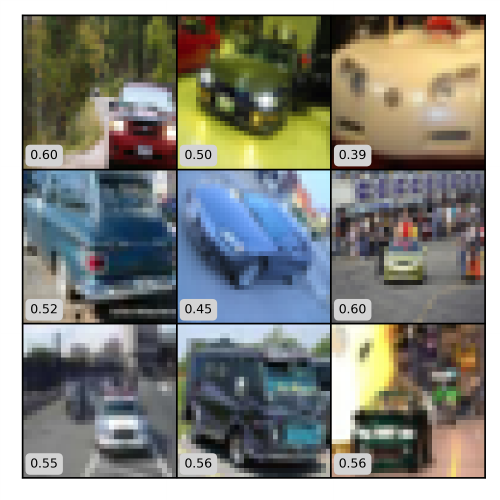}
    \end{subfigure}
    \begin{subfigure}{0.23\textwidth}
        \includegraphics[width=\linewidth, trim={0.2cm, 0.2cm, 0.2cm, 0.2cm}]{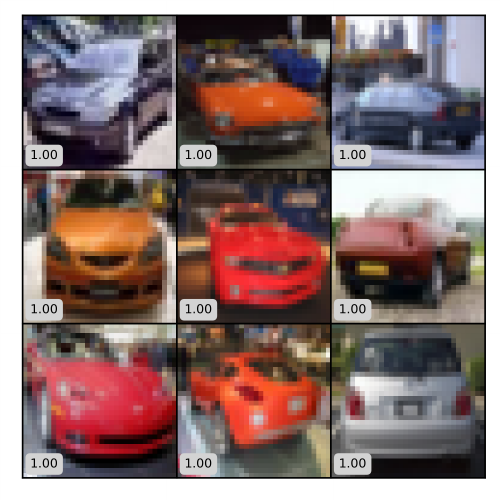}
    \end{subfigure}
    \caption{Additional images with the lowest and highest \probust{} and \psoftmax{} values among CIFAR10 classes. Images with high \probust{} tend to be brighter and have stronger object-background contrast (making them more robust to random noise) than those with low \probust{}. The difference between images with high and low \psoftmax{} is less clear. Thus, \probust{} better captures the model's local robustness with respect to an input than \psoftmax{}.}
    \label{fig-supp:topk-vs-bottomk}
\end{figure*}

\begin{figure*}[htbp!]
    \vspace{1cm}
    \centering
    \begin{flushleft}
        \hspace{-0.1cm}\rotatebox{90}{\hspace{-6.7cm}Cloud \hspace{2.8cm}Bicycle}
        \hspace{1.2cm}Lowest \pmmsewsigma{0.05}
        \hspace{1.5cm} Highest \pmmsewsigma{0.05}
        \hspace{1.5 cm} Lowest \psoftmax{}
        \hspace{1.5cm} Highest \psoftmax{}
    \end{flushleft}
         
    \begin{subfigure}{0.23\textwidth}
        \includegraphics[width=\linewidth, trim={0.2cm, 0.2cm, 0.2cm, 0.2cm}]{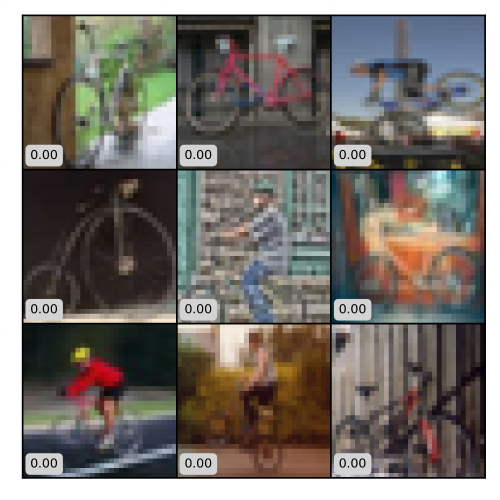}
    \end{subfigure}
    \begin{subfigure}{0.23\textwidth}
        \includegraphics[width=\linewidth, trim={0.2cm, 0.2cm, 0.2cm, 0.2cm}]{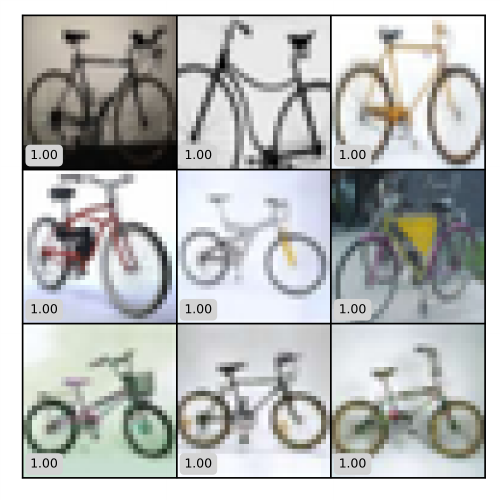}
    \end{subfigure}
    \begin{subfigure}{0.23\textwidth}
        \includegraphics[width=\linewidth, trim={0.2cm, 0.2cm, 0.2cm, 0.2cm}]{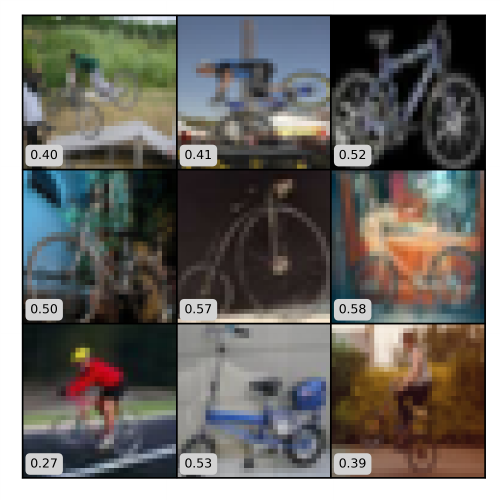}
    \end{subfigure}
    \begin{subfigure}{0.23\textwidth}
        \includegraphics[width=\linewidth, trim={0.2cm, 0.2cm, 0.2cm, 0.2cm}]{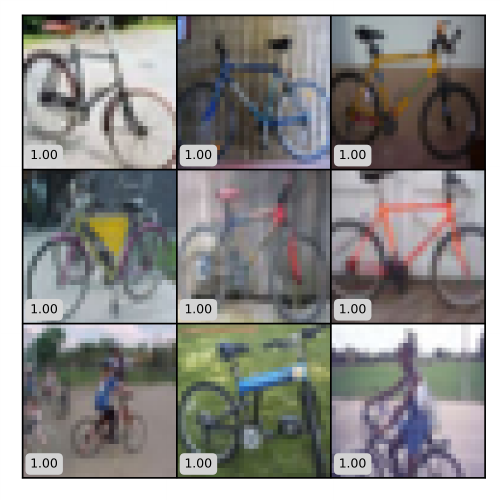}
    \end{subfigure}
    
    \begin{subfigure}{0.23\textwidth}
        \includegraphics[width=\linewidth, trim={0.2cm, 0.2cm, 0.2cm, 0.2cm}]{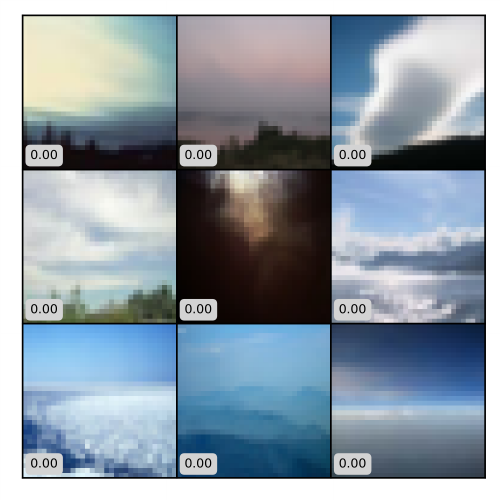}
    \end{subfigure}
    \begin{subfigure}{0.23\textwidth}
        \includegraphics[width=\linewidth, trim={0.2cm, 0.2cm, 0.2cm, 0.2cm}]{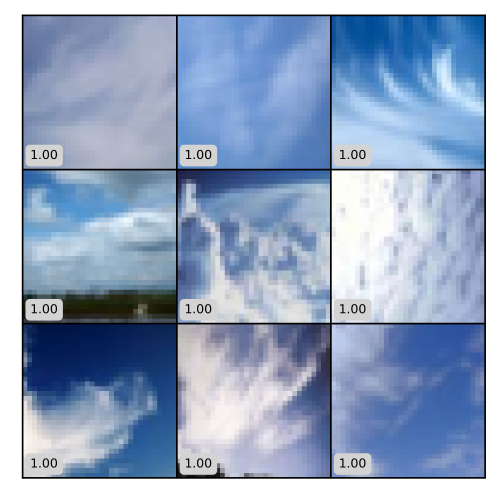}
    \end{subfigure}
    \begin{subfigure}{0.23\textwidth}
        \includegraphics[width=\linewidth, trim={0.2cm, 0.2cm, 0.2cm, 0.2cm}]{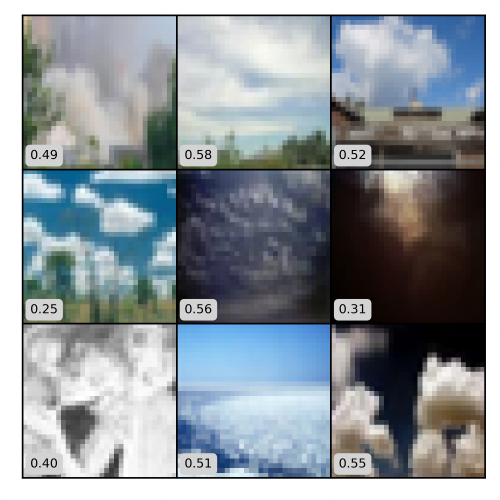}
    \end{subfigure}
    \begin{subfigure}{0.23\textwidth}
        \includegraphics[width=\linewidth, trim={0.2cm, 0.2cm, 0.2cm, 0.2cm}]{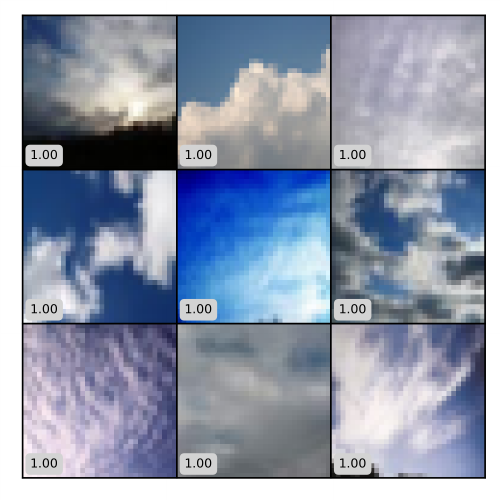}
    \end{subfigure}
    \caption{Images with the lowest and highest \probust{} and \psoftmax{} values among CIFAR100 classes. Images with high \probust{} tend to be brighter and have stronger object-background contrast (making them more robust to random noise) than those with low \probust{}. The difference between images with high and low \psoftmax{} is less clear. Thus, \probust{} better captures the model's local robustness with respect to an input than \psoftmax{}.}
    \label{fig6:topk-vs-bottomk}
\end{figure*}

\begin{figure*}[htpb!]
    \vspace{1cm}
    \centering
    \begin{flushleft}
        \hspace{-0.1cm}\rotatebox{90}{\hspace{-7.7cm}Digit 2 \hspace{1.6cm}Digit 1 \hspace{1.3cm} Digit 0}
        \hspace{1.3cm}Original 
        \hspace{1.1cm} $\sigma=0.2$ 
        \hspace{1.2cm} $\sigma=0.4$
        \hspace{1.2cm} $\sigma=0.6$
        \hspace{1.2cm} $\sigma=0.8$
        \hspace{1.2cm} $\sigma=1.0$
    \end{flushleft}

    \begin{subfigure}{0.9\textwidth}
        \includegraphics[width=\linewidth]{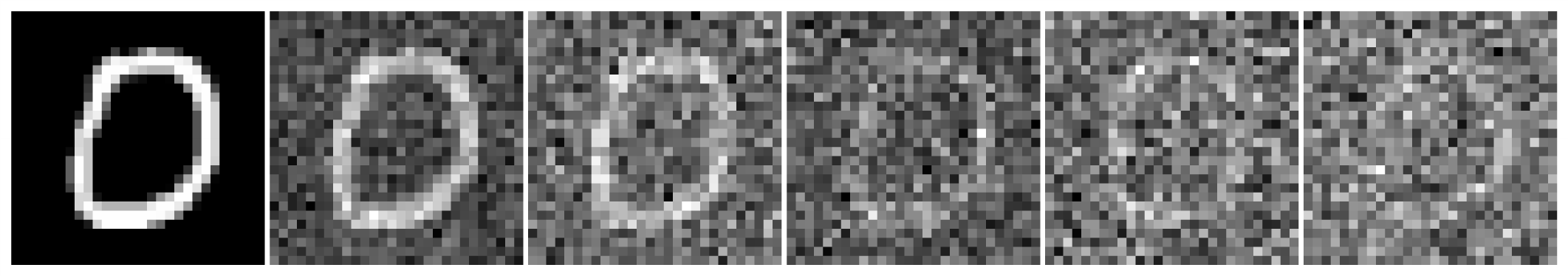}
    \end{subfigure}
    
    \begin{subfigure}{0.9\textwidth}
        \includegraphics[width=\linewidth]{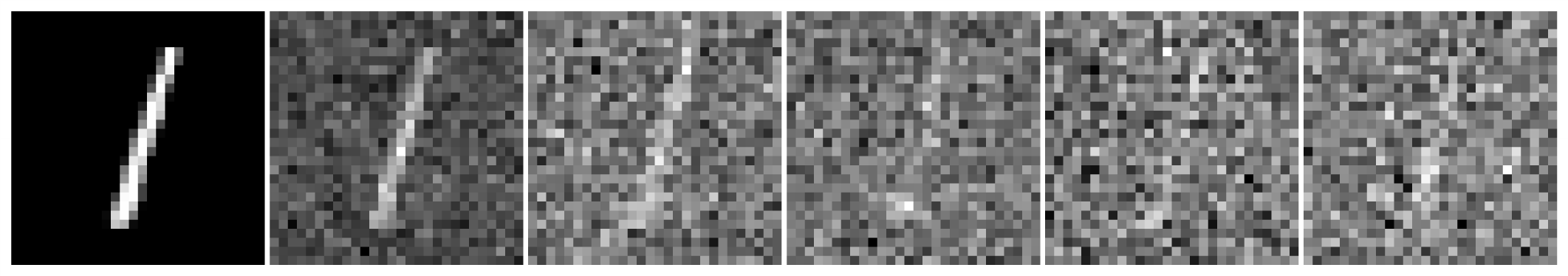}
    \end{subfigure}
    
    \begin{subfigure}{0.9\textwidth}
        \includegraphics[width=\linewidth]{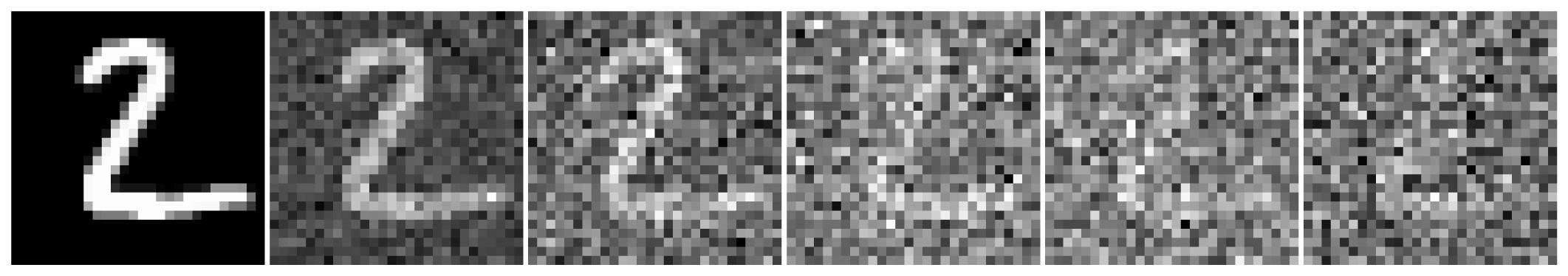}
    \end{subfigure}
    \caption{Examples of noisy images for MNIST.}
    \label{app:noisy_mnist}
\end{figure*}

\begin{figure*}[htbp!]
    \vspace{1cm}
    \centering
    \begin{flushleft}
        \hspace{-0.1cm}\rotatebox{90}{\hspace{-8cm}Ankle boot \hspace{1.1cm}Trousers \hspace{1.5cm} Shirt}
        \hspace{1.3cm}Original 
        \hspace{1.1cm} $\sigma=0.2$ 
        \hspace{1.2cm} $\sigma=0.4$
        \hspace{1.2cm} $\sigma=0.6$
        \hspace{1.2cm} $\sigma=0.8$
        \hspace{1.2cm} $\sigma=1.0$
    \end{flushleft}

    \begin{subfigure}{0.9\textwidth}
        \includegraphics[width=\linewidth]{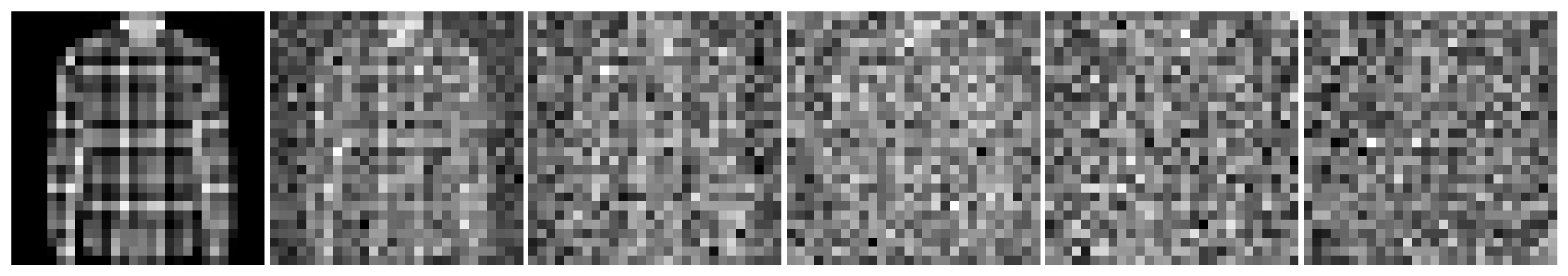}
    \end{subfigure}
    
    \begin{subfigure}{0.9\textwidth}
        \includegraphics[width=\linewidth]{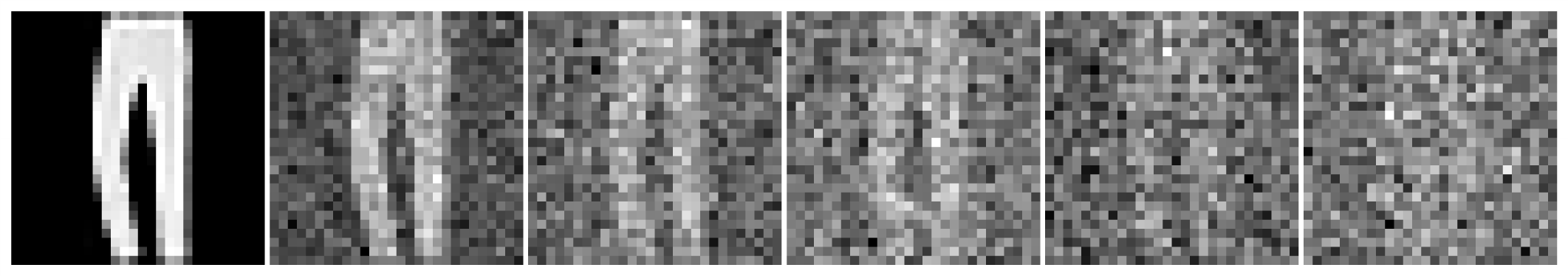}
    \end{subfigure}
    
    \begin{subfigure}{0.9\textwidth}
        \includegraphics[width=\linewidth]{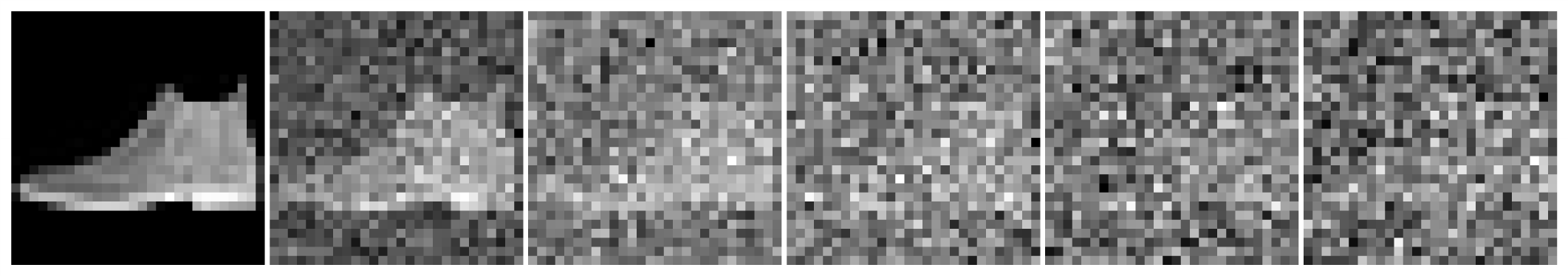}
    \end{subfigure}
    \caption{Examples of noisy images for FMNIST.}
    \label{app:noisy_fmnist}
\end{figure*}

\begin{figure*}[htbp!]
    \vspace{1cm}
    \centering
    \begin{flushleft}
        \hspace{-0.1cm}\rotatebox{90}{\hspace{-7.5cm}Ship \hspace{1.7cm}Airplane \hspace{1.5cm} Deer}
        \hspace{1.3cm}Original 
        \hspace{1cm} $\sigma=0.02$ 
        \hspace{1cm} $\sigma=0.04$
        \hspace{1cm} $\sigma=0.06$
        \hspace{1cm} $\sigma=0.08$
        \hspace{1.1cm} $\sigma=0.1$
    \end{flushleft}

    \begin{subfigure}{0.9\textwidth}
        \includegraphics[width=\linewidth]{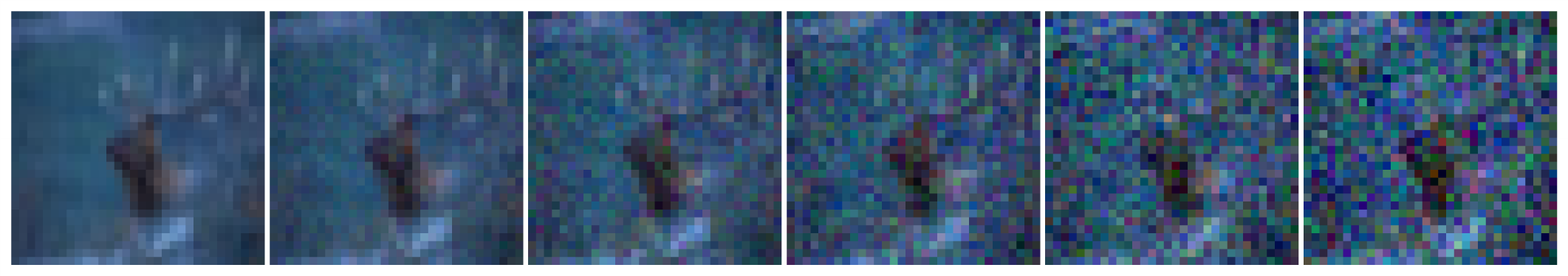}
    \end{subfigure}
    
    \begin{subfigure}{0.9\textwidth}
        \includegraphics[width=\linewidth]{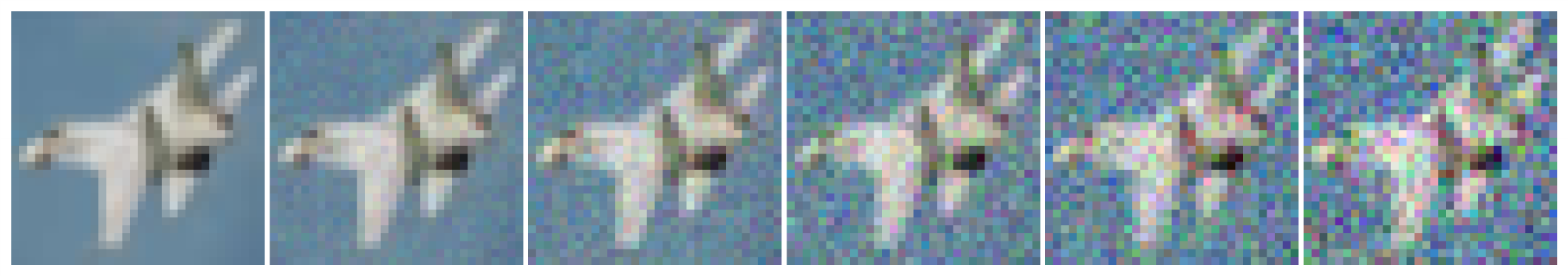}
    \end{subfigure}
    
    \begin{subfigure}{0.9\textwidth}
        \includegraphics[width=\linewidth]{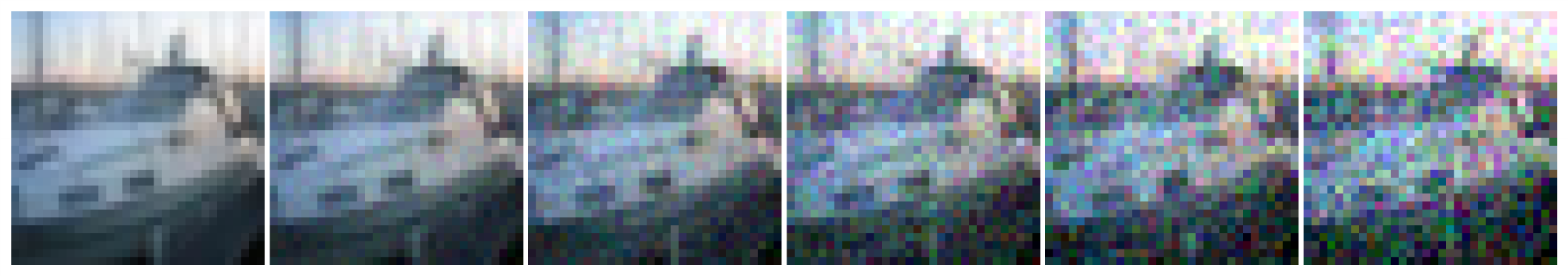}
    \end{subfigure}
    \caption{Examples of noisy images for CIFAR10.}
    \label{app:noisy_cifar10}
\end{figure*}

\begin{figure*}[htbp!]
    \vspace{1cm}
    \centering
    \begin{flushleft}
        \hspace{-0.1cm}\rotatebox{90}{\hspace{-7.5cm}Lion \hspace{1.8cm}Cloud \hspace{1.7cm} Sea}
        \hspace{1.3cm}Original 
        \hspace{1cm} $\sigma=0.02$ 
        \hspace{1cm} $\sigma=0.04$
        \hspace{1cm} $\sigma=0.06$
        \hspace{1cm} $\sigma=0.08$
        \hspace{1.1cm} $\sigma=0.1$
    \end{flushleft}

    \begin{subfigure}{0.9\textwidth}
        \includegraphics[width=\linewidth]{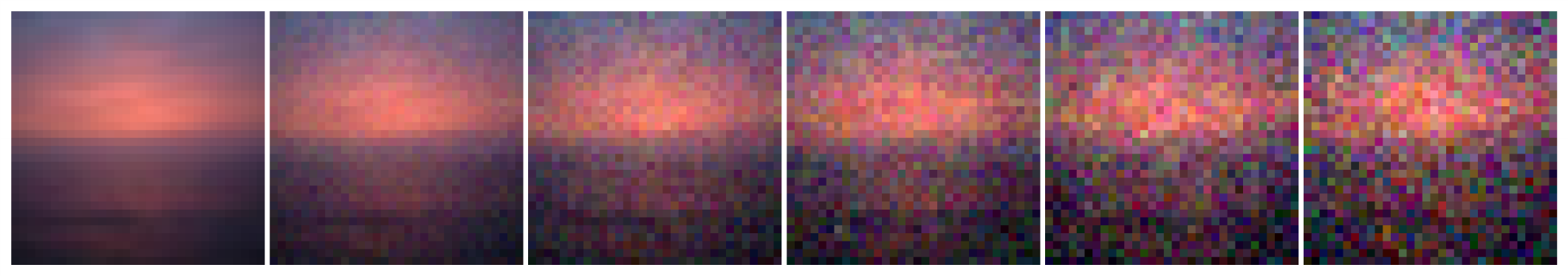}
    \end{subfigure}
    
    \begin{subfigure}{0.9\textwidth}
        \includegraphics[width=\linewidth]{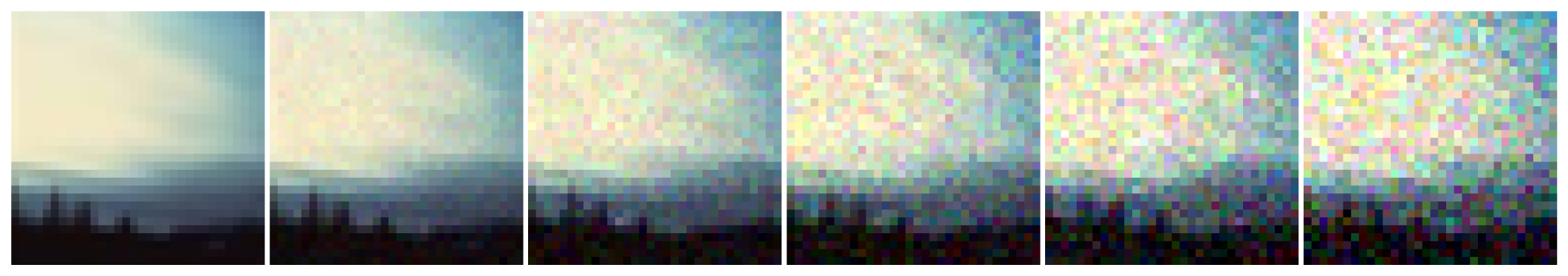}
    \end{subfigure}
    
    \begin{subfigure}{0.9\textwidth}
        \includegraphics[width=\linewidth]{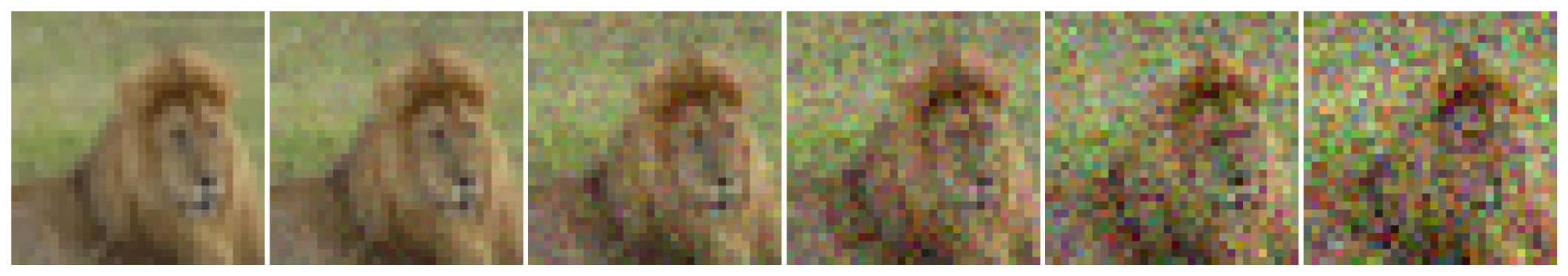}
    \end{subfigure}
    \caption{Examples of noisy images for CIFAR100.}
    \label{app:noisy_cifar100}
\end{figure*}


%% file: uai2024-template.bbl
\begin{thebibliography}{35}
\providecommand{\natexlab}[1]{#1}
\providecommand{\url}[1]{\texttt{#1}}
\expandafter\ifx\csname urlstyle\endcsname\relax
  \providecommand{\doi}[1]{doi: #1}\else
  \providecommand{\doi}{doi: \begingroup \urlstyle{rm}\Url}\fi

\bibitem[Sci()]{SciPy}
Sci{P}y multivariate normal {CDF}.
\newblock \url{https://docs.scipy.org/doc/scipy/reference/generated/scipy.stats.multivariate_normal.html}.

\bibitem[flo()]{flops}
What’s the backward-forward {FLOP} ratio for neural networks?
\newblock \url{https://www.lesswrong.com/posts/fnjKpBoWJXcSDwhZk/what-s-the-backward-forward-flop-ratio-for-neural-networks}.
\newblock Accessed: 2024-04-04.

\bibitem[Agarwal et~al.(2021)Agarwal, Jabbari, Agarwal, Upadhyay, Wu, and Lakkaraju]{agarwal2021towards}
Sushant Agarwal, Shahin Jabbari, Chirag Agarwal, Sohini Upadhyay, Steven Wu, and Himabindu Lakkaraju.
\newblock Towards the unification and robustness of perturbation and gradient based explanations.
\newblock \emph{International Conference on Machine Learning}, 2021.

\bibitem[Botev(2017)]{botev2017normal}
Zdravko~I Botev.
\newblock The normal law under linear restrictions: {S}imulation and estimation via minimax tilting.
\newblock \emph{Journal of the Royal Statistical Society. Series B (Statistical Methodology)}, 2017.

\bibitem[Carlini and Wagner(2017)]{carlini2017towards}
Nicholas Carlini and David Wagner.
\newblock Towards evaluating the robustness of neural networks.
\newblock \emph{IEEE Symposium on Security and Privacy}, 2017.

\bibitem[Carlini et~al.(2022)Carlini, Tramer, Dvijotham, Rice, Sun, and Kolter]{carlini2022certified}
Nicholas Carlini, Florian Tramer, Krishnamurthy Dvijotham, Leslie Rice, Mingjie Sun, and Zico Kolter.
\newblock ({C}ertified!!) adversarial robustness for free!
\newblock \emph{International Conference on Learning Representations}, 2022.

\bibitem[Cohen et~al.(2019)Cohen, Rosenfeld, and Kolter]{cohen2019certified}
Jeremy Cohen, Elan Rosenfeld, and Zico Kolter.
\newblock Certified adversarial robustness via randomized smoothing.
\newblock \emph{International Conference on Machine Learning}, 2019.

\bibitem[Deng(2012)]{deng2012mnist}
Li~Deng.
\newblock The {MNIST} database of handwritten digit images for machine learning research.
\newblock \emph{IEEE Signal Processing Magazine}, 2012.

\bibitem[Fazlyab et~al.(2019)Fazlyab, Morari, and Pappas]{fazlyab2019probabilistic}
Mahyar Fazlyab, Manfred Morari, and George~J Pappas.
\newblock Probabilistic verification and reachability analysis of neural networks via semidefinite programming.
\newblock In \emph{2019 IEEE 58th Conference on Decision and Control (CDC)}, pages 2726--2731. IEEE, 2019.

\bibitem[Franceschi et~al.(2018)Franceschi, Fawzi, and Fawzi]{franceschi2018robustness}
Jean-Yves Franceschi, Alhussein Fawzi, and Omar Fawzi.
\newblock Robustness of classifiers to uniform lp and gaussian noise.
\newblock In \emph{International Conference on Artificial Intelligence and Statistics}, pages 1280--1288. PMLR, 2018.

\bibitem[Goodfellow et~al.(2015)Goodfellow, Shlens, and Szegedy]{goodfellow2014explaining}
Ian~J Goodfellow, Jonathon Shlens, and Christian Szegedy.
\newblock Explaining and harnessing adversarial examples.
\newblock \emph{International Conference on Learning Representations}, 2015.

\bibitem[Han et~al.(2022)Han, Srinivas, and Lakkaraju]{han2022explanation}
Tessa Han, Suraj Srinivas, and Himabindu Lakkaraju.
\newblock Which explanation should {I} choose? {A} function approximation perspective to characterizing post hoc explanations.
\newblock \emph{Advances in Neural Information Processing Systems}, 2022.

\bibitem[He et~al.(2016)He, Zhang, Ren, and Sun]{he2016deep}
Kaiming He, Xiangyu Zhang, Shaoqing Ren, and Jian Sun.
\newblock Deep residual learning for image recognition.
\newblock \emph{IEEE Conference on Computer Vision and Pattern Recognition}, 2016.

\bibitem[Hendrycks and Gimpel(2016)]{hendrycks2016gaussian}
Dan Hendrycks and Kevin Gimpel.
\newblock Gaussian error linear units ({GELU}s).
\newblock \emph{arXiv preprint arXiv:1606.08415}, 2016.

\bibitem[Krizhevsky et~al.(2009)Krizhevsky, Hinton, et~al.]{krizhevsky2009learning}
Alex Krizhevsky, Geoffrey Hinton, et~al.
\newblock Learning multiple layers of features from tiny images.
\newblock \emph{University of Toronto}, 2009.

\bibitem[Kumar et~al.(2020)Kumar, Levine, Feizi, and Goldstein]{kumar2020certifying}
Aounon Kumar, Alexander Levine, Soheil Feizi, and Tom Goldstein.
\newblock Certifying confidence via randomized smoothing.
\newblock \emph{Advances in Neural Information Processing Systems}, 33:\penalty0 5165--5177, 2020.

\bibitem[Mangal et~al.(2019)Mangal, Nori, and Orso]{mangal2019robustness}
Ravi Mangal, Aditya~V Nori, and Alessandro Orso.
\newblock Robustness of neural networks: A probabilistic and practical approach.
\newblock In \emph{2019 IEEE/ACM 41st International Conference on Software Engineering: New Ideas and Emerging Results (ICSE-NIER)}, pages 93--96. IEEE, 2019.

\bibitem[Moayeri et~al.(2022)Moayeri, Banihashem, and Feizi]{moayeri2022explicit}
Mazda Moayeri, Kiarash Banihashem, and Soheil Feizi.
\newblock {E}xplicit tradeoffs between adversarial and natural distributional robustness.
\newblock \emph{Advances in Neural Information Processing Systems}, 2022.

\bibitem[Moosavi-Dezfooli et~al.(2016)Moosavi-Dezfooli, Fawzi, and Frossard]{moosavi2016deepfool}
Seyed-Mohsen Moosavi-Dezfooli, Alhussein Fawzi, and Pascal Frossard.
\newblock Deep{F}ool: {A} simple and accurate method to fool deep neural networks.
\newblock \emph{IEEE Conference on Computer Vision and Pattern Recognition}, 2016.

\bibitem[Moosavi-Dezfooli et~al.(2019)Moosavi-Dezfooli, Fawzi, Uesato, and Frossard]{moosavi2019robustness}
Seyed-Mohsen Moosavi-Dezfooli, Alhussein Fawzi, Jonathan Uesato, and Pascal Frossard.
\newblock Robustness via curvature regularization, and vice versa.
\newblock In \emph{Proceedings of the IEEE/CVF Conference on Computer Vision and Pattern Recognition}, pages 9078--9086, 2019.

\bibitem[Nanda et~al.(2021)Nanda, Dooley, Singla, Feizi, and Dickerson]{nanda2021fairness}
Vedant Nanda, Samuel Dooley, Sahil Singla, Soheil Feizi, and John~P Dickerson.
\newblock Fairness through robustness: {I}nvestigating robustness disparity in deep learning.
\newblock \emph{ACM Conference on Fairness, Accountability, and Transparency}, 2021.

\bibitem[Ovadia et~al.(2019)Ovadia, Fertig, Ren, Nado, Sculley, Nowozin, Dillon, Lakshminarayanan, and Snoek]{ovadia2019can}
Yaniv Ovadia, Emily Fertig, Jie Ren, Zachary Nado, David Sculley, Sebastian Nowozin, Joshua Dillon, Balaji Lakshminarayanan, and Jasper Snoek.
\newblock {C}an you trust your model's uncertainty? {E}valuating predictive uncertainty under dataset shift.
\newblock \emph{Advances in Neural Information Processing Systems}, 2019.

\bibitem[Patrick(1995)]{patrick1995probability}
Billingsley Patrick.
\newblock Probability and measure.
\newblock \emph{A Wiley-Interscience Publication, John Wiley}, 118:\penalty0 119, 1995.

\bibitem[Pawelczyk et~al.(2023)Pawelczyk, Datta, van-den Heuvel, Kasneci, and Lakkaraju]{pawelczyk2022probabilistically}
Martin Pawelczyk, Teresa Datta, Johannes van-den Heuvel, Gjergji Kasneci, and Himabindu Lakkaraju.
\newblock Probabilistically robust recourse: {N}avigating the trade-offs between costs and robustness in algorithmic recourse.
\newblock \emph{International Conference on Learning Representations}, 2023.

\bibitem[Rice et~al.(2021)Rice, Bair, Zhang, and Kolter]{rice2021robustness}
Leslie Rice, Anna Bair, Huan Zhang, and J~Zico Kolter.
\newblock Robustness between the worst and average case.
\newblock \emph{Advances in Neural Information Processing Systems}, 34:\penalty0 27840--27851, 2021.

\bibitem[Robey et~al.(2022)Robey, Chamon, Pappas, and Hassani]{robey2022probabilistically}
Alexander Robey, Luiz Chamon, George~J Pappas, and Hamed Hassani.
\newblock Probabilistically robust learning: Balancing average and worst-case performance.
\newblock In \emph{International Conference on Machine Learning}, pages 18667--18686. PMLR, 2022.

\bibitem[Smilkov et~al.(2017)Smilkov, Thorat, Kim, Vi{\'e}gas, and Wattenberg]{smilkov2017smoothgrad}
Daniel Smilkov, Nikhil Thorat, Been Kim, Fernanda Vi{\'e}gas, and Martin Wattenberg.
\newblock Smooth{G}rad: removing noise by adding noise.
\newblock \emph{arXiv preprint arXiv:1706.03825}, 2017.

\bibitem[Srinivas and Fleuret(2018)]{srinivas2018knowledge}
Suraj Srinivas and Fran{\c{c}}ois Fleuret.
\newblock Knowledge transfer with jacobian matching.
\newblock \emph{International Conference on Machine Learning}, 2018.

\bibitem[Srinivas et~al.(2022)Srinivas, Matoba, Lakkaraju, and Fleuret]{srinivas2022efficient}
Suraj Srinivas, Kyle Matoba, Himabindu Lakkaraju, and Fran{\c{c}}ois Fleuret.
\newblock Efficient training of low-curvature neural networks.
\newblock \emph{Advances in Neural Information Processing Systems}, 35:\penalty0 25951--25964, 2022.

\bibitem[Srinivas et~al.(2024)Srinivas, Bordt, and Lakkaraju]{srinivas2024models}
Suraj Srinivas, Sebastian Bordt, and Himabindu Lakkaraju.
\newblock {W}hich models have perceptually-aligned gradients? {A}n explanation via off-manifold robustness.
\newblock \emph{Advances in Neural Information Processing Systems}, 2024.

\bibitem[Taori et~al.(2020)Taori, Dave, Shankar, Carlini, Recht, and Schmidt]{taori2020measuring}
Rohan Taori, Achal Dave, Vaishaal Shankar, Nicholas Carlini, Benjamin Recht, and Ludwig Schmidt.
\newblock Measuring robustness to natural distribution shifts in image classification.
\newblock \emph{Advances in Neural Information Processing Systems}, 2020.

\bibitem[Thulasidasan et~al.(2021)Thulasidasan, Thapa, Dhaubhadel, Chennupati, Bhattacharya, and Bilmes]{thulasidasan2021effective}
Sunil Thulasidasan, Sushil Thapa, Sayera Dhaubhadel, Gopinath Chennupati, Tanmoy Bhattacharya, and Jeff Bilmes.
\newblock {A}n effective baseline for robustness to distributional shift.
\newblock \emph{IEEE International Conference on Machine Learning and Applications (ICMLA)}, 2021.

\bibitem[Vershynin(2018)]{vershynin2018high}
Roman Vershynin.
\newblock \emph{High-dimensional probability: An introduction with applications in data science}.
\newblock Cambridge University Press, 2018.

\bibitem[Weng et~al.(2019)Weng, Chen, Nguyen, Squillante, Boopathy, Oseledets, and Daniel]{weng2019proven}
Lily Weng, Pin-Yu Chen, Lam Nguyen, Mark Squillante, Akhilan Boopathy, Ivan Oseledets, and Luca Daniel.
\newblock Proven: Verifying robustness of neural networks with a probabilistic approach.
\newblock In \emph{International Conference on Machine Learning}, pages 6727--6736. PMLR, 2019.

\bibitem[Xiao et~al.(2017)Xiao, Rasul, and Vollgraf]{xiao2017fashion}
Han Xiao, Kashif Rasul, and Roland Vollgraf.
\newblock Fashion-{MNIST}: {A} novel image dataset for benchmarking machine learning algorithms.
\newblock \emph{arXiv preprint arXiv:1708.07747}, 2017.

\end{thebibliography}
